\documentclass[sigconf,authorversion,nonacm]{aamas}

\usepackage{balance} 

\usepackage{hyperref}

\usepackage{etoolbox}

\usepackage{savesym,multirow,url,xspace,graphicx,hyperref,enumitem,color}
\usepackage[titletoc]{appendix}
\usepackage{subcaption}
\usepackage{dsfont}
\usepackage{multicol}

\usepackage{tikz}
\usetikzlibrary{automata, positioning, arrows}
\tikzset{
	->, 
	every state/.style={thick, fill=gray!10}, 
	initial text=$ $, 
}
\usetikzlibrary{arrows.meta}
\usepackage{subcaption}
\usepackage{pgfplots}
\pgfplotsset{width=10cm,compat=1.9}
\usepackage{diagbox}
\usepackage{caption}

\makeatletter
\newif\if@restonecol
\makeatother
\usepackage{amsmath,amsthm}
\usepackage{mathtools}
\setitemize{noitemsep,topsep=1pt,parsep=1pt,partopsep=1pt, leftmargin=12pt}
\newtheorem{lemma}{Lemma}
\newtheorem{theorem}{Theorem}
\newtheorem*{theorem*}{Theorem}

\newtheorem{proposition}{Proposition}

\newtheorem{remark}{Remark}

\makeatletter

\newcommand{\Rmnum}[1]{\expandafter\@slowromancap\romannumeral #1@}
\makeatother
\usepackage{caption}

\usepackage{enumitem}
\usepackage{wasysym}

\DeclareMathOperator*{\argmax}{arg\,max}

\newcommand{\cA}{1}
\newcommand{\cM}{{\mathcal{M}}}

\newcommand{\cP}{{\mathcal{P}}}
\newcommand{\cL}{2}

\usepackage{bm}
\usepackage{bbm}

\newcommand{\norm}[1]{\left\lVert#1\right\rVert}

\newcommand{\expct}[1]{\mathbb{E}\left[#1\right]}

\newcommand{\ind}[1]{\mathds{1}\left[#1\right]}
\newcommand{\pos}[1]{\left[#1\right]^+}

\newcommand{\eps}{\epsilon}

\newcommand{\pola}{\pi_{\cA}}
\newcommand{\polap}{\pi_{\theta}}
\newcommand{\Pola}{\Pi^{\cA}}

\newcommand{\poll}{\pi_{\cL}}
\newcommand{\pollp}{\pi_{\phi}}

\newcommand{\pollpopt}{\pi_{\phi^*}}
\newcommand{\Poll}{\Pi^{\cL}}
\newcommand{\Polld}{\Pi^{\cL}_{\text{det}}}

\newcommand{\targetpi}{\ensuremath{{\pi^{\dagger}_{\cL}}}}
\newcommand{\initpi}{\ensuremath{{\pi_{\cA}^0}}}

\newcommand{\targetpiproxy}{\ensuremath{{\tilde \pi^{\dagger}_{\cL}}}}
\newcommand{\ntargetpiproxy}{\ensuremath{{\bar \pi^{\dagger}_{\cL}}}}

\newcommand{\dpi}{\ensuremath{ \delta }}

\newcommand{\score}{{\rho}}

\newcommand{\occupancy}{{\psi}}

\newcommand{\occstate}{{\mu}}

\newcommand{\optpil}{\pi^{*|\pi_1}_{\cL}}
\newcommand{\Optl}{\textsc{{Opt}}_{\cL}^{\eps} }

\newcommand{\Cost}{\textsc{Cost}}

\newcommand{\dist}{\textup{dist}}
\newcommand{\cost}{\textup{cost}}

\newcommand{\conspolsearch}{\textsc{Conservative Policy Search for Implicit Attacks (General)}}

\newcommand{\conspolsearcherg}{\textsc{Conservative Policy Search for Implicit Attacks and Ergodic Environments}}

\newcommand{\alterupdates}{\textsc{Alternating Policy Updates for Implicit Attacks}}

\usepackage{algorithm}
\usepackage{algorithmic}

\acmSubmissionID{1132}

\title[Implicit Poisoning Attacks in Two-Agent Reinforcement Learning]{Implicit Poisoning Attacks in Two-Agent Reinforcement Learning:\\Adversarial Policies for Training-Time Attacks}

\author{Mohammad Mohammadi}\authornote{Equal contributions}
\affiliation{
    \institution{MPI-SWS}
  \city{Saarbr\"ucken}
  \country{Germany}
  }
\email{mmohamma@mpi-sws.org}

\author{Jonathan N\"other}\authornotemark[1]
\affiliation{
  \institution{Saarland University}
  \city{Saarbr\"ucken}
  \country{Germany}
  }
\email{s8jonoet@stud.uni-saarland.de}

\author{Debmalya Mandal}
\affiliation{
  \institution{MPI-SWS}
  \city{Saarbr\"ucken}
  \country{Germany}
  }
\email{dmandal@mpi-sws.org}

\author{Adish Singla}
\affiliation{
  \institution{MPI-SWS}
  \city{Saarbr\"ucken}
  \country{Germany}
  }
\email{adishs@mpi-sws.org}

\author{Goran Radanovic}
\affiliation{
  \institution{MPI-SWS}
  \city{Saarbr\"ucken}
  \country{Germany}
  }
\email{gradanovic@mpi-sws.org}

\begin{abstract}
In targeted poisoning attacks, an attacker manipulates an agent-environment interaction to force the agent into adopting a policy of interest, called {\em target} policy. Prior work has primarily focused on attacks that modify standard MDP primitives, such as rewards or transitions. In this paper, we study targeted poisoning attacks in a two-agent setting where an attacker 
{\em implicitly} poisons the {\em effective} environment of one of the agents by modifying the policy of its peer. 
We develop an optimization framework for designing optimal attacks, where the cost of the attack measures how much the solution deviates from the assumed default policy of the peer agent. We further study the computational properties of this optimization framework. 
Focusing on a tabular setting, we show that in contrast to poisoning attacks based on MDP primitives (transitions and (unbounded) rewards), which are always feasible, it is NP-hard to determine the feasibility of implicit poisoning attacks. 
We provide characterization results that establish sufficient conditions for the feasibility of the attack problem, as well as an upper and a lower bound on the optimal cost of the attack.
We propose two algorithmic approaches for finding an optimal adversarial policy: a model-based approach with tabular policies and a model-free approach with parametric/neural policies. 
We showcase the efficacy of the proposed algorithms through experiments.
\end{abstract}

\newcommand{\BibTeX}{\rm B\kern-.05em{\sc i\kern-.025em b}\kern-.08em\TeX}

\begin{document}

\maketitle

\newtoggle{longversion}
\settoggle{longversion}{true}


\section{Introduction}

Recent works on adversarial attacks in reinforcement learning (RL) have demonstrated the susceptibility of RL algorithms to various forms of adversarial attacks \cite{huang2017adversarial,lin2017tactics,sun2020stealthy,sun2020vulnerability,kiourti2020trojdrl}, 
including poisoning attacks which manipulate a learning agent in its training phase, altering the end result of the learning process, i.e., the agent's policy \cite{ma2019policy,rakhsha2020policy,zhang2020adaptive,sun2020vulnerability,rakhsha2021policy,liu2021provably,ijcai2022-471}. 
In order to understand and evaluate the stealthiness of such attacks, it is important to assess the underlying assumptions that are made in the respective attack models. 
Often, attack models are based on altering the environment feedback of a learning agent. For example, in environment poisoning attacks, the attacker can manipulate the agent's rewards or transitions \cite{ma2019policy,rakhsha2020policy}---these manipulations could correspond to changing the parameters of the model that the agent is using during its training process. However, directly manipulating the environment feedback of a learning agent may not always be practical. For example, rewards are often internalized or are goal specific, in which case one cannot directly poison the agent's rewards. Similarly, direct manipulations of transitions and observations may not be practical in environments with complex dynamics, given the constraints on what can be manipulated by such attacks. 

In order to tackle these practical challenges, \citet{gleave2019adversarial} introduce a novel class of attack models for a competitive two-agent RL setting with a zero-sum game structure. In particular, they consider an attacker that controls one of the agents. By learning an adversarial policy for that agent, the adversary can force the other, victim agent, to significantly degrade its performance. \citet{gleave2019adversarial} focus on test-time attacks that learn adversarial policies for an already trained victim. This idea has been further explored by \citet{guo2021adversarial} in the context of more general two-player games, which are not necessarily zero-sum, and by \citet{wang2021backdoorl} in the context of backdoor attacks, where the action of the victim's opponent can trigger a backdoor hidden in the victim's policy. 

In this paper, we focus on {\em targeted} poisoning attacks that aim to force a learning agent into adopting a certain policy of interest, called {\em target policy}. In contrast to prior work on targeted poisoning attacks \cite{ma2019policy,rakhsha2020policy,rakhsha2021policy}, which primarily considered single-agent RL and attack models that directly manipulate MDP primitives (e.g., rewards or transitions), we consider a two-agent RL setting and an attack model that is akin to the one studied in \cite{gleave2019adversarial,guo2021adversarial,wang2021backdoorl}, but tailored to environment poisoning attacks. More specifically, in our setting, the attacker {\em implicitly} poisons the {\em effective} environment of a victim agent by controlling its peer at training-time. 
To ensure the stealthiness of the attack, the attacker aims to minimally alter the default behavior of the peer, which we model through the cost of the attack. 

\begin{figure*}[ht]
	\centering
	\begin{subfigure}{0.48\textwidth}
    	\centering
		\resizebox{0.98\linewidth}{!}{\includegraphics{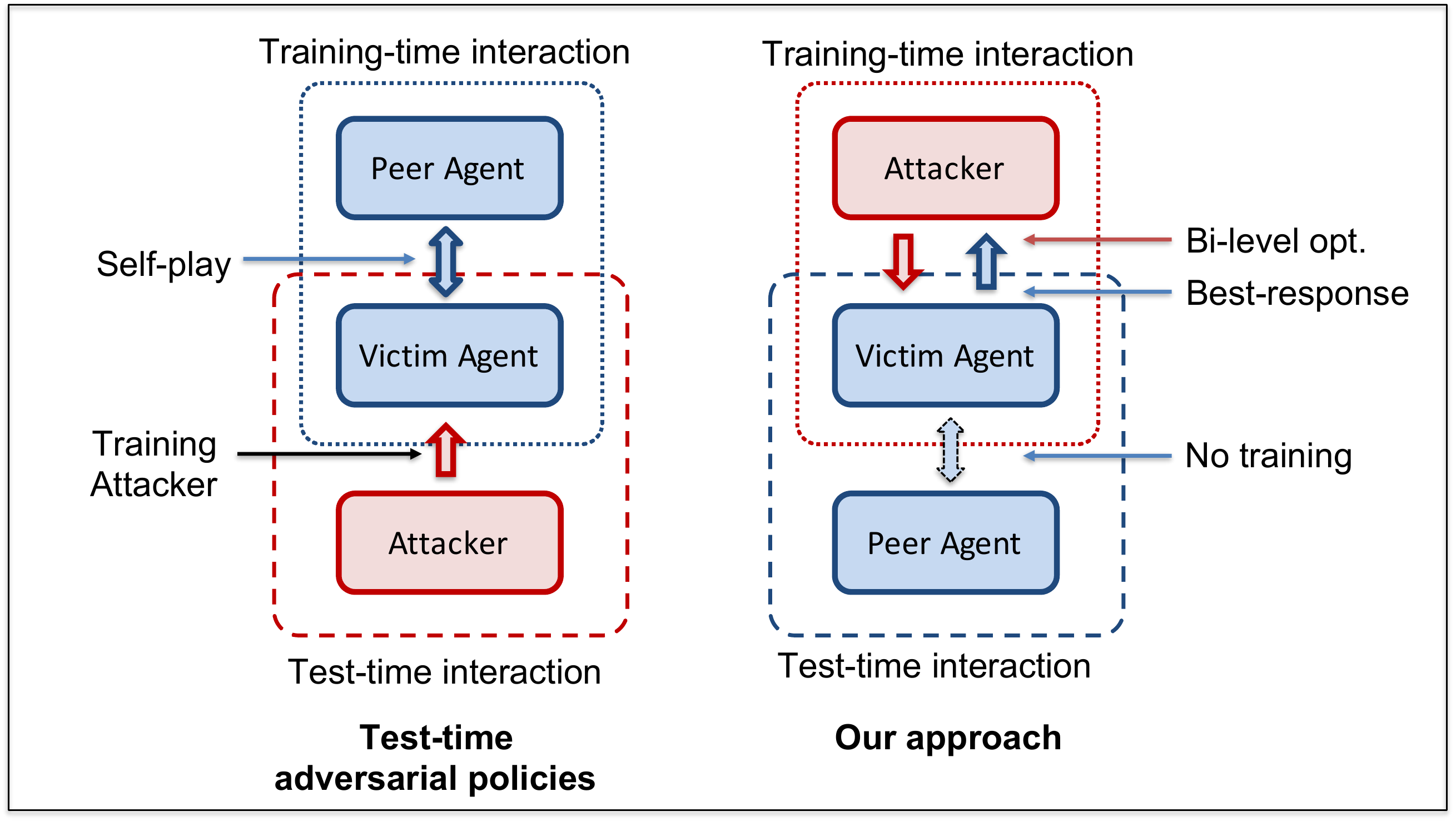}}
		\caption{Comparison to test-time adversarial policies}\label{fig.intro.tap}
	\end{subfigure}
	\begin{subfigure}{0.48\textwidth}
    	\centering
		\resizebox{0.98\linewidth}{!}{\includegraphics{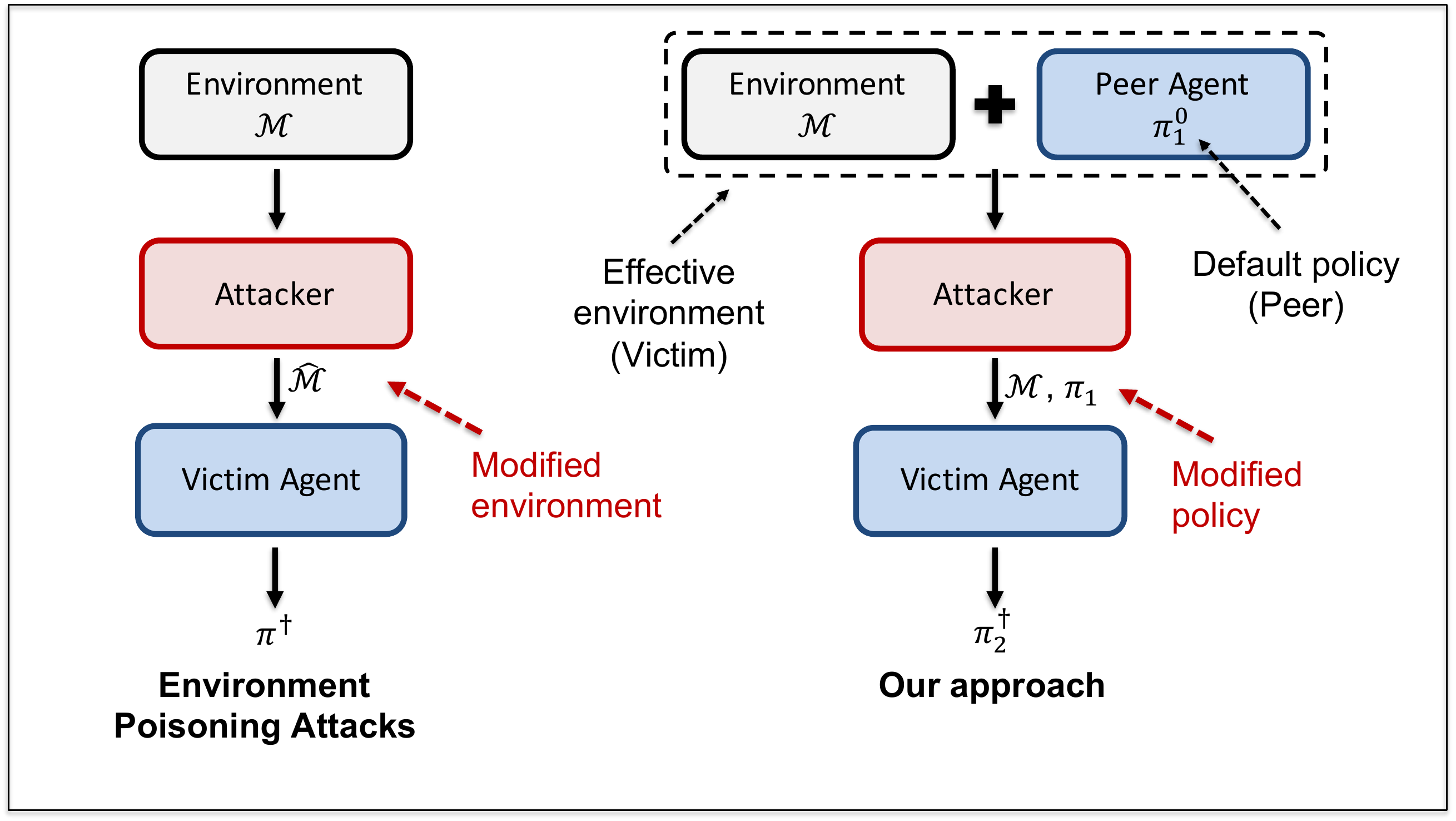}}
		\caption{Comparison to environment poisoning attacks}\label{fig.intro.tap.vs.epa}
	\end{subfigure}
    \caption{The figure compares our approach to the closely related prior work. Fig.  \ref{fig.intro.tap} illustrates the differences between our setting and test-time adversarial policies (from \cite{gleave2019adversarial}). Test-time adversarial policies are attacking a fixed victim (trained in self-play), whereas our approach attacks a victim during training phase. We consider an optimization framework based on bi-level optimization that minimizes the attack cost while ensuring that the victim's best response to our attack is to adopt target policy $\targetpi$. As shown in Fig. \ref{fig.intro.tap.vs.epa}, this optimization approach is similar to environment poisoning attacks (from \cite{ma2019policy,rakhsha2020policy,rakhsha2021policy}). 
    However, it differs from environment poisoning attacks in that the attack only modifies the default policy of the victim's peer, $\initpi$, but not the underlying environment $\mathcal M$. I.e, our approach implicitly poisons the effective environment of the victim.}
	\label{fig.intro}
\end{figure*}

Fig. \ref{fig.intro} illustrates the main aspects of the setting considered in this work. In this setting, the default policy of the victim's peer is fixed and the victim is trained to best respond to a corrupted version of the peer's  policy.\footnote{While the setting is similar to Stackelberg models where the peer agent commits its policy~(e.g., \cite{letchford2012computing}),
a critical difference is that the adversary can modify this policy. See also the related work section.} This is different from test-time adversarial policies where the victim is fixed and the adversary controls the victim's opponent/peer at test-time.
Our setting corresponds to a practical scenario in which an attacker controls the peer agent at training-time and executes a training-time adversarial policy instead of the peer's default policy. If the victim is trained offline, the attacker can corrupt the offline data instead, e.g., by executing training-time adversarial policies when the offline data is collected or by directly poisoning the data.
Note that direct access to the training process of the victim agent is not required to train an adversarial policy. It suffices that the victim agent approximately best respond to the adversarial policy. This is effectively the same assumption that prior work on environment poisoning attacks in offline RL adopts, where the attacker first manipulates the underlying environment, after which the victim agent optimizes its policy in the poisoned environment~\cite{ma2019policy,rakhsha2020policy,rakhsha2021policy}. 
\iftoggle{longversion}{
 
Fig. \ref{fig.intro} also shows how our setting compares to  environment poisoning target attacks. As explained in the figure, there are conceptual differences between the corresponding attack models. To see this more concretely, consider a simple single-shot two-agent setting where the peer can take actions $\{a, b \}$ and the victim can take actions $\{a, b, \text{no-op}\}$. In this setting, the victim receives reward of $1$ if its action matches the action the peer, $r \in (0, 1)$ if it takes $\text{no-op}$, and $0$ otherwise. Under the peer's default policy, which takes action $a$, the optimal policy of the victim is to take $a$. Now, suppose that an attacker controls the peer and wants to force action $\text{no-op}$ on the victim, so that $\text{no-op}$ is strictly optimal for the victim by some margin $\eps$. It is easy to show that this attack is feasible only if $r \ge 0.5 + \eps$ and only when the adversarial policy is non-deterministic. In contrast, under the default policy of the peer, reward poisoning attacks can force the target policy by simply increasing $r$ to $1 + \eps$. Nevertheless, training-time adversarial policies are arguably a more practical attack modality, and hence understanding their effectiveness and limitations is important. 
}
{
Fig. \ref{fig.intro} also shows how our setting compares to  environment poisoning target attacks. As explained in the figure, there are conceptual differences between the corresponding attack models.
}

To the best of our knowledge, this is the first work that studies adversarial policies for training-time attacks. Our contributions are summarized below:
\begin{itemize}
\item We introduce a novel optimization framework for studying an implicit form of targeted poisoning attacks in a two-agent RL setting, where an attacker manipulates the victim agent by controlling its peer at training-time. 
\item We then analyze computational aspects of this optimization problem. 
We show that it is NP-hard to decide whether the optimization problem is feasible, i.e., whether it is possible to force a target policy. This is in contrast to general environment poisoning attacks that manipulate both rewards and transitions~\cite{rakhsha2021policy}, which are always feasible.\footnote{
Attacks that poison only rewards are always feasible.
Attacks that poison only transitions may not be always feasible since transitions cannot be arbitrarily changed~\cite{rakhsha2020policy}. 
Similarly, when rewards are bounded, \citet{ijcai2022-471} show that reward poisoning attacks may be infeasible. 
The computational complexity of such attacks have not been formally analyzed.
}  
\item We further analyze the cost of the optimal attack, providing a lower and an upper bound on the cost of the attack, as well as a sufficient condition for the feasibility of the attack problem. To obtain the lower bound, we follow the theoretical analysis in recent works on environment poisoning attacks \cite{ma2019policy,rakhsha2020policy,rakhsha2021policy} that establish similar lower bounds for the single-agent setting, and we adapt it to the two-agent setting of interest. The theoretical analysis that yields the upper bound does not follow from prior work, since the corresponding proof techniques cannot be directly applied to our setting. 
\item We propose two algorithmic approaches for finding an efficient adversarial policy. The first one is a model-based  approach with tabular policies which outputs a feasible solution to the attack problem, if one is found. It is based on a conservative policy search algorithm that performs efficient policy updates that account for the cost of the attack, while aiming to minimize the margin by which the constraints of the attack problem are not satisfied. The second one is a model-free approach with parametric/neural policies, which is based on a nonconvex-nonconcave minimax optimization.   
\item Finally, we conduct extensive experiments to demonstrate the efficacy of our algorithmic approaches. \iftoggle{longversion}{Our experimental test-bed is based on or inspired by environments from prior work on poisoning attacks and single-agent RL (which we modify and make two-agent) and standard multi-agent RL environments (which we modify to fit our problem setting). The experimental results showcase the utility of our algorithmic approaches in finding cost-efficient adversarial policies. }{}  
\end{itemize}
\iftoggle{longversion}{These results complement those on environment poisoning attacks, demonstrating the effectiveness of training time-adversarial policies.}{}

\subsection{Related Work}

{\bf Adversarial Attacks in ML.}  
Adversarial attacks on machine learning models have been extensively studied by prior work. We recognize two main attack approaches on machine learning: test-time attacks~\cite{biggio2013evasion,szegedy2014intriguing, nguyen2015deep, moosavi2016deepfool, papernot2017practical}, which do not alter a learning model, but rather they fool the model by manipulating its input, and  training-time or data poisoning attacks~\cite{biggio2012poisoning,xiao2012adversarial,mei2015using,xiao2015feature,li2016data}, which manipulate a learning model by, e.g., altering its training points. We also mention backdoor attacks~\cite{liu2017neural,gu2017badnets,chen2017targeted}, that hide a trigger in a learned model, which can then be activated at test-time. 

{\bf Adversarial Attacks in RL.}  Needless to say, such attack strategies have also been studied in RL.  \cite{huang2017adversarial,kos2017delving,lin2017tactics,behzadan2017whatever,sun2020stealthy} consider efficient test-time attacks on agents' observations. In contrast to this line of work, we consider training-time attacks which are not based on state/observation perturbations. \cite{yang2019design,kiourti2020trojdrl,wang2021stop} consider backdoor attacks on RL policies. These works are different in that backdoor triggers affect the victim's observations; our attack model influences the victim's transitions and rewards. Poisoning attacks in single-agent RL have been studied under different poisoning aims: attacking rewards \cite{ma2019policy,rakhsha2020policy,ijcai2022-471}, attacking transitions \cite{rakhsha2020policy}, attacking both rewards and transition \cite{rakhsha2021policy}, attacking actions \cite{liu2021provably}, or attacking a generic observation-action-reward tuple \cite{sun2020vulnerability}. Reward poisoning attacks have also been studied in multi-agent RL~\cite{wu2022reward}. In contrast to such poisoning attacks, our attack model does not directly poison any of the mentioned poisoning aims. It instead indirectly influences the victim's rewards and transitions. 
This work, therefore, complements prior work on poisoning attacks in RL and adversarial policies, as already explained. 

{\bf Other Related Work.} 
We also mention closely related work on robustness to adversarial attacks and settings that have similar formalisms.
Much of the works on robustness to these attacks study robustness to test-time attacks~\cite{pattanaik2017robust,zhang2020robust,zhang2021robust,wu2021crop} and, closer to this paper, poisoning attacks~\cite{lykouris2021corruption,zhang2021robust-icml,zhang2022corruption,wu2021copa,kumar2021policy,banihashem2021defense}. Out of these, \cite{banihashem2021defense} has the formal setting that is the most similar to ours, focusing on defenses against targeted reward poisoning attacks.  
Our setting is also related to stochastic Stackelberg games and similar frameworks~\cite{vorobeychik2012computing,letchford2012computing,dimitrakakis2017multi} in that we have an attacker who acts as a leader that aims to minimize its cost, while accounting for a rational follower (victim) that optimizes its return. 
However, in our framework, the cost of the attack is not modeled via a reward function, while the attack goal of forcing a target policy is a hard constraint. Hence, the computational intractability results for Stackelberg stochastic games do not directly apply to our setting.
Nonetheless, the reduction that we use to show our NP-hardness result is inspired by the proofs of the hardness results in \cite{letchford2012computing}. 
Finally, we also mention the line of work on policy teaching~\cite{zhang2008value,zhang2009policy,banihashem2022admissible}, whose formal settings are quite similar to those of targeted reward poisoning attacks~\cite{ma2019policy,rakhsha2021policy}.

\section{Implicit Poisoning Attacks}\label{sec.setting}

In this section, we formalize the attack problem of interest: adversarial policies for training-time attacks. 

\subsection{Multi-Agent Environment}

\textbf{Environment model. }
We study a reinforcement learning setting formalized by a two-agent Markov Decision Process $\cM = (\{ \cA, \cL \}, S, A, P, R_{\cL}, \gamma, \sigma)$, where $\cA$ is the index of an agent controlled by an attacker, $\cL$ is the index of a learning agent (victim) under attack, $S$ is the state space, $A = A_{\cA} \times A_{\cL}$ is the joint action space with $A_{\cA}$ and $A_{\cL}$ defining the action spaces of agents $\cA$ and $\cL$ respectively,  $P: S \times A \times S \rightarrow [0, 1]$ is the transition model, $R_{\cL}: S \times A \rightarrow \mathds R$ is the reward function of the learner,  $\gamma$ is the discount factor, and $\sigma$ is the initial state distribution. We denote the probability of transitioning to state $s'$ from $s$ by $P(s, a_{\cA}, a_{\cL}, s')$ and the reward obtained in state $s$ by $R_{\cL}(s, a_{\cA}, a_{\cL})$, where $a_{\cA}$ and $a_{\cL}$ are the actions of agent $\cA$ and agent $\cL$ taken in state $s$. 
In our formal treatment of the problem, we will primarily focus on finite state and action spaces $S$ and $A$.

\textbf{Policies.}
The policy of agent $\cA$ is denoted by $\pola$ and we assume that it comes from the set of stochastic stationary policies $\Pola$. 
That is, policy $\pola$ is mapping $\pola: S \rightarrow \mathcal P(A_{\cA})$, where $\mathcal P(A_\cA)$ is the probability simplex over $A_{\cA}$. 
Analogously, the policy of agent $\cL$ is denoted by $\poll$. A stochastic stationary policy $\poll \in \Poll$ 
is a mapping $\poll: S \rightarrow \mathcal P(A_{\cL})$. The set of all deterministic policies in $\Poll$ is denoted by $\Polld = \{ \poll \in \Poll \text{ s.t. } \poll(s, a_{\cL}) \in \{0, 1\} \}$. 

\textbf{Score \& Occupancy Measures.}
We further consider standard quantities. The (normalized) expected discounted return 
 of agent $\cL$ under policies $\pola$ and $\poll$ is defined as 
$$\score_{\cL}= (1 - \gamma) \cdot \expct{\sum_{t = 1}^{\infty} \gamma^{t-1} \cdot R_{\cL}(s_t, a_{\cA, t}, a_{\cL, t}) | \sigma, \pola, \poll},$$
where the expectation is taken over trajectory $(s_1, a_{\cA, 1}, a_{\cL, 1}, ...)$ obtained by executing policy $\pi$ starting in a state sampled from $\sigma$. The return $\score_{\cL}^{\pola, \poll}$ is equal to 
\begin{align}\label{eq.return_occupance}
    \score_{\cL}^{\pola, \poll} = 
    \sum_{s, a_{\cA}, a_{\cL}} \occupancy^{\pola, \poll}(s, a_{\cA}, a_{\cL}) \cdot R_{\cL}(s, a_{\cA}, a_{\cL}),
\end{align}
where $\occupancy^{\pola, \poll}(s, a_{\cA}, a_{\cL}) = \occstate^{\pola, \poll}(s) \cdot \pola(s, a_{\cA}) \cdot \poll(s, a_{\cL})$ is the state-action occupancy measure, and  $\occstate^{\pola, \poll}$ 
is the state occupancy measure, i.e., 
$\occstate^{\pola, \poll}(s) = (1 - \gamma) \cdot \expct{\sum_{t = 1}^{\infty} \gamma^{t-1} \cdot \ind{s_t = s}|\sigma, \pola, \poll}$.
Note that we do not assume that the underlying MDP is ergodic, i.e., we allow that $\occstate^{\pola, \poll}(s) = 0$ for some states $s$. Finally, we also define value function $V^{\pola, \poll}: S \rightarrow \mathds R$ as $$V^{\pola, \poll}(s) = \expct{\sum_{t = 1}^{\infty} \gamma^{t-1} \cdot R_{\cL}(s_t, a_{\cA, t}, a_{\cL, t}) | s_1 = s, \pola, \poll}.$$

\begin{remark}\label{rm.notation}
To simplify the notation, we often abbreviate summations, e.g., the summation over $a_{\cA}$ and $a_{\cL}$ can be replaced by $R_{\cL}(s, \pola, \poll)$. Furthermore, since in our formal treatment of the problem we focus on a tabular setting with finite state and action spaces, we in part utilize vector notation when convenient. For example, $R_{\cL}$ can be thought of as a vector with $|S|\cdot |A|$ components.
\end{remark}

\subsection{Problem Statement}

We focus on an attack model that manipulates a {\em default} policy of the victim's peer, $\initpi$, to force a target policy $\targetpi$.  
Following prior work on targeted policy attacks \cite{ma2019policy,rakhsha2020policy,rakhsha2021policy}, we first consider an optimization problem which models the attack goal as a hard constraint with {\em deterministic} $\targetpi$ and a victim agent that adopts an approximately optimal deterministic policy\footnote{Similar learner models have been considered in prior work that analyzes a dual to optimal reward poisoning attacks~\cite{banihashem2022admissible}.}:
\begin{align*}
    \label{prob.instance_1}\tag{P1}
    &\min_{\pola} \quad \Cost(\pola, \initpi) & \mbox{ s.t. } 
    \quad \Optl(\pola) \subseteq \Pi_2^{\dagger}(\pola).
\end{align*} 
Here, $\Pi_2^{\dagger}(\pola)$ is a set of policies $\poll$ that are equal to $\targetpi$ on visited states, i.e.,  $\poll(s, a_{\cL}) = \targetpi(s, a_{\cL})$ when $\occstate^{\pola, \targetpi}(s) > 0$. Furthermore, $\Optl(\pola)$ is the set of approximately optimal deterministic policies $\poll$ given $\pola$, i.e.,  $\Optl(\pola) = \{ \poll \in \Polld \text{ s.t. } \score^{\pola, \poll} > \score^{\pola, \optpil} - \eps  \}$, where $\optpil \in \argmax_{\poll} \score_{\cL}^{\pola, \poll}$, while $\eps \ge 0$ is a parameter that controls the sub-optimality of the learner.
As standard in this line of work, in our characterization results of \eqref{prob.instance_1}, we focus on a norm-based attack cost function:
\begin{align*}
\Cost(\pola, \initpi) = \left ( \sum_{s}\left (\sum_{a_\cA} | \pola (s, a_\cA) - \initpi(s, a_\cA) | \right )^{\frac{1}{p}} \right )^p,
\end{align*}
where $p \ge 1$. In the next sections, we formally analyze \eqref{prob.instance_1} and propose an algorithm for solving it. We also consider an optimization problem that relaxes the attack goal, but is more amenable to optimization with deep RL and allows {\em stochastic} target policies $\targetpi$:
\begin{align}
\label{prob.instance_2}\tag{P2}
    \min_{\theta} \max_{\phi} \mathcal L_I(\theta, \initpi) - \lambda \cdot \left[ \score_{\cL}^{\polap, \targetpi} -  \score_{\cL}^{\polap, \pollp} \right].
\end{align}
Here, $\polap$ and $\pollp$ are parametric policies that respectively correspond to $\pola$ and $\poll$, and $\mathcal L_I(\theta, \initpi)$ is an imitation learning loss function. The imitation learning loss corresponds to the cost of the attack: we instantiate it with standard cross-entropy imitation objective for deterministic $\initpi$ and Kullback–Leibler divergence for stochastic $\initpi$.
We further motivate \eqref{prob.instance_2} in the next sections.

\section{Characterization Results}\label{sec.characterization}

In this section, we provide a theoretical treatment of the optimization problem \eqref{prob.instance_1} akin to those from prior work on poisoning attacks in RL ~\cite{ma2019policy,rakhsha2020policy,rakhsha2021policy}. We start by analyzing the complexity of the optimization problem, followed by the analysis that provides bounds on the optimal value of \eqref{prob.instance_1}. The full proofs of  our results from this section are provided in the appendix. 

\subsection{Computational Complexity}

To study the properties of the optimization problem \eqref{prob.instance_1}, let us more explicitly write its constraint using the following set of inequalities:
\begin{align*}
    \score_2^{\pola, \targetpi} \ge \score_2^{\pola, \poll} + \eps, \quad \forall \poll \in \Polld \backslash \Pi_2^{\dagger}(\pola).
\end{align*}
At the first glance, the optimization problem \eqref{prob.instance_1} appears to be computationally challenging: 
the number of inequality constraints
is exponential. 
On the other hand, Lemma 1 from \cite{rakhsha2021policy} suggests that it suffices to consider {\em neighbor} policies of the target policy to determine whether a solution is feasible---a neighbor policy $\pi\{s,a\}$ of policy $\pi$ is equal to $\pi$ in all states except in $s$, where it is defined as $\pi\{s,a\}(s, a) = 1.0$. However, given the differences between the setting of \citet{rakhsha2021policy} and the setting of this paper, in particular, because the latter considers a two-agent and possibly non-ergodic MDP environment, this result does not directly apply.
In the appendix, we prove a couple of results akin to Lemma 1 from \cite{rakhsha2021policy}, but for the setting of interest. These results allow us to reduce the number of constraints one ought to account for when testing the feasibility of solution $\pola$. For example, if the MDP environment is ergodic, $\pola$ is a feasible solution iff 
\begin{align}\label{eq.sufficient_condition.simple}
    \score_2^{\pola, \targetpi} \ge \score_2^{\pola, \targetpi \{s,a_\cL \}} +  \eps \quad \forall s, a \text{ s.t. } \targetpi(s, a_\cL ) = 0.
\end{align}

While such results are useful as they reduce the number of constraints one ought to account for when testing the feasibility of solution $\pola$, they do not necessarily imply that the optimization problem is easy to solve. The difficulty lies in the quadratic form of the constraints in Eq. \eqref{eq.sufficient_condition.simple}. Namely, as can be seen from Eq. \eqref{eq.return_occupance}, they depend on $\pola$ through policy $\pola$ itself but also through the state occupancy measure $\occstate^{\pola, \cdot }$. 
Our next result verifies this intuition. 

\begin{theorem}\label{thm.computational_complexity}
It is NP-hard to decide if the optimization problem \eqref{prob.instance_1} is feasible, i.e., whether there exists a solution $\pola$ s.t. the constraints of the optimization problem are satisfied.
\end{theorem}

The proof of the claim can be found in the appendix, and is based on a polynomial time reduction of the Boolean 3-SAT problem to our optimization framework. Hence, the tractability of the optimization problem \eqref{prob.instance_1} would imply that NP=P. 
To conclude, despite the similarities between our implicit poisoning attack model and the general environment poisoning attacks from \cite{rakhsha2021policy}, which are always feasible, determining the feasibly of implicit poisoning attacks is computationally challenging. 

\subsection{Bounds on the Optimal Value}\label{sec.bounds_section}

\textbf{Lower Bound.} Next, we aim to bound the value of the optimal solution. We first focus on a lower bound on the cost of optimal solution. In particular, we follow the recent line of work on poisoning attacks in RL~\cite{ma2019policy,rakhsha2020policy,rakhsha2021policy}, and adapt their proof techniques to our problem setting in order to establish a lower bound on the cost of the optimal attack. 
To state the main theorem, we define a state-action dependent quantity $\bar \chi_{\eps'}(s, a)$ similar to the one from~\cite{rakhsha2021policy}, but adapted to the setting of the paper. In particular, we define\footnote{Note that $[x]^+ = \max(0, x)$.} 
$\overline \chi_{\eps'}(s, a_{\cL}) = \pos{\frac{\score_{\cL }^{\initpi, \targetpi\{s, a_{\cL}\}} - \score_{\cL }^{\initpi, \targetpi} + \eps'}{ \occstate^{\initpi, \targetpi\{s,a_{\cL}\}}(s)}}$
if $\occstate^{\pola, \targetpi}(s) > 0$ for all $\pola$
and $\targetpi(s, a_{\cL}) = 0$, while $\overline \chi_{\eps'}(s, a_{\cL}) = 0$ otherwise.\footnote{The first condition can be verified for any given state $s$ by optimizing over $\pola$ a reward function that is strictly negative in state $s$, and is equal to $0$ otherwise. The condition is satisfied iff the optimal value is $0$. In general, the condition holds if the underlying Markov chain is ergodic for $\targetpi$ and every policy $\pola$ (see Theorem \ref{thm.upper_bound_2}).} $\overline \chi_{\eps'}(s, a)$ is a measure of the utility gap between the target policy $\targetpi$ and the neighbor policy $\targetpi\{s,a\}$ given the default policy $\initpi$ and some offset $\eps'$. Together with $R_{\cL}$ and $V^{\initpi, \targetpi}$, $\overline \chi_{\eps'}$ can be used to obtain the following lower bound.  

\begin{theorem}\label{thm.cost_lower_bound}
    The attack cost of any solution to the optimization problem \eqref{prob.instance_1}, if it exists, satisfies 
    \begin{align*}
        \Cost(\pola, \initpi) \ge \frac{1-\gamma}{2} \cdot \frac{\norm{\overline \chi_{0}}_{\infty} }{\norm{R_2}_{\infty} + \gamma \cdot \norm{V^{\initpi, \targetpi}}_{\infty}}.
    \end{align*}
\end{theorem}
The lower bound in Theorem \ref{thm.cost_lower_bound} is similar to the corresponding lower bound for general environment poisoning attacks~\cite{rakhsha2021policy}, albeit not being fully comparable given the differences between the settings and the definitions of $\overline \chi$. One notable difference is that the bound in Theorem \ref{thm.cost_lower_bound} additionally depends on the reward vector $R_2$ because the adversary only 
influences rewards through its actions.

\textbf{Upper Bound.}
Compared to environment poisoning attacks, providing an interpretable upper bound in our setting is more challenging since the attack model of this paper cannot in general successfully force a target policy $\targetpi$. This is in stark contrast to, e.g., reward poisoning attacks, which remain feasible even under the restriction that rewards obtained by following $\targetpi$ are not modified. Additionally, as per Theorem \ref{thm.computational_complexity}, the feasibility of the attack problem is computationally intractable. Due to the latter challenge, we consider a special case when transitions are independent of policy $\pola$ (i.e., $P(s, a_{\cA}, a_{\cL}) = P(s, a_{\cA}', a_{\cL})$ for all $a_{\cA}$ and $a_{\cA}'$) and the Markov chain induced by $\targetpi$ and any policy $\pola$ is ergodic. In the appendix, we show that \eqref{prob.instance_1} can be efficiently solved in this case.

To state our formal result, we first define two quantities, $\alpha_2^{\pola}(s,a) = \score_2^{\pola, \targetpi} - \score_2^{\pola, \targetpi\{s, a\}}$,  and $\alpha_2^* = \sup_{\pola} \min_{s, a} \alpha_2^{\pola}(s,a)$. Intuitively, $\alpha_2^{\pola}$ measures the utility gap between $\targetpi$ and its neighbor policy $\targetpi\{s, a\}$ for a given policy $\pola$, whereas $\alpha_2^*$ denotes the optimal guaranteed gap that can be achieved. Note that there exists $\pola^{*}$ s.t.  $\alpha_2^{\pola^{*}} = \alpha_2^*$, and in the appendix we provide a linear program for finding $\pola^{*}$.
\begin{theorem}\label{thm.upper_bound_2}
Assume that $P(s, a_{\cA}, a_{\cL}) = P(s, a_{\cA}', a_{\cL})$ for all $a_{\cL}$ and $a_{\cA}'$, and that for $\targetpi$ and every policy $\pola$ the underlying Markov chain is ergodic, i.e., $\occstate^{\pola, \targetpi}(s)>0$ for all $\pola$. 
If $\alpha_2^* \ge \eps$, the optimization problem \eqref{prob.instance_1} is feasible and the cost of an optimal solution satisfies
\begin{align*}
         \Cost(\pola, \initpi) \le 2 \cdot \norm{\frac{\overline \chi_{\eps}}{\chi_2^*  + \overline \chi_{\eps}}}_{\infty} \cdot |S|^{1/p} 
\end{align*}
     with the element-wise division (equal to $0$ if $\chi_{\eps}(s, a_\cL) = \chi_2^*(s,a_\cL) = 0$), where $ \chi_{2}^*(s, a_\cL) = \frac{\alpha^{*}_2(s,  a_\cL) - \eps}{\occstate^{\initpi, \targetpi \{ s, a_\cL \}}(s)}$.
\end{theorem}
As with the lower bound, the upper bound is not directly comparable to the bounds obtained in prior work~\cite{rakhsha2021policy}. In the appendix, we analyze another special case, when both $\pola$ and $\poll$ do not influence transitions, and obtain a slightly tighter bound. In that case, we obtain the  upper bound $2 \cdot \norm{\frac{\overline \chi_{\eps}}{\chi_2^*  + \overline \chi_{\eps}}}_{p, \infty}$, where $\overline \chi_{\eps}$ and $\chi_2^*$ are now treated as matrices with $|A_{\cL}| \times |S|$ entries.  We leave for the future work whether it is possible to improve the result in Theorem \ref{thm.upper_bound_2} and match this bound.


\section{Algorithms}\label{sec.algorithm}

In this section, we study two algorithmic approaches for solving the optimization problem \eqref{prob.instance_1}: a model-based approach with tabular policies for solving \eqref{prob.instance_1}, and a model-free approach with neural policies for solving \eqref{prob.instance_2}.

\subsection{Conservative Policy Search for Implicit Attacks}\label{sec.algorithm.planning}

In this subsection, we propose an algorithm for \eqref{prob.instance_1}.
To simplify the exposition, we focus on a version of the algorithm that applies to ergodic environments---in the appendix, we provide an extension to non-ergodic environments. 

\begin{algorithm}
\caption[labelsep=period]{\conspolsearcherg}\label{alg.conserv_pol_iter_main_text}
\begin{algorithmic}
\REQUIRE $\mathcal M = (\{1, 2\}, S, A, P, R, \gamma, \sigma)$, $\eps$, $\delta_\eps$, $\initpi$, $\lambda$, $p$
\ENSURE Policy of the adversary, $\pola$
\STATE Initialize $t = 0$
\FOR{$t = 0$ to $T-1$}
\STATE Calculate state occupancy measures $\occstate^{\pola^t, \targetpi}$ and $\occstate^{\pola^t, \targetpi \{ s, a_{\cL}\}}$
\STATE Evaluate the gap $\eps_{\pola^t} = \min_{\eps'} \eps'$ s.t. $\score_\cL^{\pola^t, \targetpi} \ge \score_\cL^{\pola^t, \targetpi\{s, a_{\cL}\}} + \eps'$ 
\STATE Solve the optimization problem \eqref{prob.instance_1.rel} to obtain $\pola^{t+1}$ \IF{$\pola^{t+1} = \pola^t$}
    \STATE {\bf break}
\ENDIF
\ENDFOR
\STATE Set the result $\pola$ to solution $\pola^{t}$ that minimizes $\norm{\pola^{t} - \initpi}_{1, p}$ while satisfying $\eps_{\pola^t} \ge \eps$  
\end{algorithmic}
\end{algorithm}

To design an efficient algorithmic procedure for finding a solution to \eqref{prob.instance_1}, we utilize the fact that \eqref{prob.instance_1} can be efficiently solved  
when policy $\pola$ does not affect the transition dynamics. 
Inspired by conservative policy iteration~\cite{kakade2002approximately} and similar approaches in RL~\cite{schulman2015trust}, we propose an algorithm that alternates between two phases. \begin{enumerate}
    \item 
In the first phase, we obtain the occupancy measures of the current solution $\pola^{t}$ and policies $\targetpi$ and $\targetpi\{s, a\}$. That is, we calculate $\occstate^{\pola^{t}, \targetpi}$ and $\occstate^{\pola^{t}, \targetpi\{s, a\}}$. 
\item 
In the second phase, we update the current solution $\pola^t$ by solving a relaxed version of \eqref{prob.instance_1}, i.e., 
\begin{align*}
    &\min_{\pola \in \mathcal B(\pola^t, \dpi), \eps'} \quad \Cost(\pola, \initpi) - \lambda \cdot \min \{\eps', \eps \cdot (1 + \delta_{\eps}) \}
    \\
    \label{prob.instance_1.rel}\tag{P1'}
    &\quad\quad\quad \mbox{ s.t. } \quad \hat \score_\cL^{\pola, \targetpi} \ge \hat \score_\cL^{\pola, \targetpi \{s, a_\cL\}} + \eps',
\end{align*}
for all $s \text{ s.t. } \occstate^{\pola, \targetpi}(s) > 0$ and $a_\cL \text{ s.t. } \targetpi(s, a_\cL) = 0$.
Here,  
$B(\pola^t, \dpi) = \{\pola \text{ s.t. } |\pola(s, a_\cA) - \pola^t(s, a_\cA)| \le \dpi \}$, $\hat \score_{\cL}^{\pola, \poll}$ is obtained via Eq. \eqref{eq.return_occupance} but by using $\occstate^{\pola^t, \poll}$ instead of $\occstate^{\pola, \poll}$, and $\delta_{\eps} \ge 0$ is a positive offset which adjusts $\eps$ (and whose role is explained later in the text). 
\end{enumerate}

The optimization problem \eqref{prob.instance_1.rel} is a relaxation of \eqref{prob.instance_1} since we optimize over the margin parameter $\eps'$, which can take negative values. Hence, \eqref{prob.instance_1.rel} is always feasible. Critically, when solving \eqref{prob.instance_1.rel}, the state occupancy measures are fixed to $\occstate^{\pola^{t}, \targetpi}$ and $\occstate^{\pola^{t}, \targetpi\{s, a\}}$, which implies that we can solve \eqref{prob.instance_1.rel} efficiently since the objective is convex, while the constraints are linear in $\pola$ and $\eps'$. The conservative update is reflected in the constraint $\pi_1 \in B(\pola^t, \dpi)$, which ensures that solutions to \eqref{prob.instance_1.rel} approximately satisfy the constraints of the original problem \eqref{prob.instance_1} (e.g., see Lemma 14.1 in \cite{agarwal2019reinforcement}). We can control the quality of this approximation through the hyperparameter $\delta_{\eps}$: for higher values of $\delta_{\eps}$ and $\eps' \ge \eps \cdot (1 + \delta_{\eps})$, solution $\pola$ to \eqref{prob.instance_1.rel} is more likely to be be a feasible solution to \eqref{prob.instance_1}. 

The final step of each iteration is to evaluate the true gap $\score_2^{\pola^t, \targetpi} - \score_2^{\pola^t, \poll}$ that each solution $\pola^t$ achieves.
The output of the algorithm is the solution that minimizes the cost while ensuring that the target gap $\eps$ is achieved. 

Algorithm \ref{alg.conserv_pol_iter_main_text} summarizes the main steps of conservative policy search for ergodic environments.\footnote{While the algorithm is well defined for any $\pola^t$, in the experiments we only consider $\pola^t$ that are fully stochastic, i.e., $\pola(s, a_\cA) > 0$ for any $s$ and $a_\cA$. In this case, the set of states $s$  s.t. $\occstate^{\pola^t, \targetpi}(s)$ does not change over time, and can be precalculated.}
 The algorithm assumes access to the model of the environment, i.e., the corresponding MDP parameters (rewards and transition probabilities), needed for obtaining relevant quantities, such as occupancy measures. The algorithm also takes the learner's parameter $\eps$ as its inputs; in practice, one can use a conservative estimate of the true parameter instead.

\subsection{Alternating Policy Updates  for Implicit Attacks}\label{sec.algorithm.learning}

We now turn to  \eqref{prob.instance_2}. First, note that we can view \eqref{prob.instance_2} as a parametric relaxation of \eqref{prob.instance_1.rel}. Namely, \eqref{prob.instance_2} is equivalent to the bi-level optimization problem:
 \begin{align}\label{prob.instance_2.bilevel}\tag{P2'}
     &\min_{\theta} \mathcal L_I(\theta, \initpi) - \lambda \cdot \left[ \score_{\cL}^{\polap, \targetpi} -  \score_{\cL}^{\polap, \pollpopt} \right]\\
     &\quad\quad\text{s.t.}\quad\quad \pollpopt \in \arg\max_{\phi} \score_2^{\polap, \pollp}\nonumber.
 \end{align}
The second term, $\score_{\cL}^{\polap, \targetpi} -  \score_{\cL}^{\polap, \pollpopt}$, measures the sub-optimality gap of the target policy, and corresponds to parameter $\eps'$ in \eqref{prob.instance_1.rel}, while the first term corresponds to the cost of the attack. This bi-level structure also motivates our algorithmic approach for finding an optimal $\theta$. 

In general, the objective of \eqref{prob.instance_2} is nonconvex-nonconcave, so the order of $\min$ and $\max$ is important~(e.g., see \cite{jin2020local}). To solve the optimization problem \eqref{prob.instance_2}, we alternate between minimizing the loss function $\mathcal L (\theta, \phi)$ over parameters $\theta$ while keeping parameters $\phi$ fixed, and maximizing $\score_{\cL}^{\polap, \pollp}$ over parameters $\phi$ while keeping $\theta$ fixed. Each optimization subroutine optimizes for a few episodes, with the latter one using more episodes. As shown by \cite{rajeswaran2020game}, this type of alternating optimization can be more effective in solving  game-theoretic bi-level optimization problems in RL similar to \eqref{prob.instance_2.bilevel} than a gradient descent-ascent approach that, in our setting, would simultaneously update $\theta$ and $\phi$. 
Algorithm \ref{alg.alter_updates_main_text} summarizes the main steps of our alternating policy updates (APU) approach. In our implementation, we pre-train policy $\pollp$ for  $\phi\text{-pretrain timesteps}$, which is typically larger than the number of timesteps ($\phi\text{-update timesteps}$) used for updating $\pollp$ in each epoch ($\phi\text{-pretrain timesteps} = 10000$ and $\phi\text{-update timesteps} = 5000$ in our experiments). 

\begin{algorithm}
    \caption[labelsep=period]{\alterupdates}\label{alg.alter_updates_main_text}
        \begin{algorithmic}
        \REQUIRE $epochs$, $\lambda$, $\epsilon$, $\targetpi$, $\initpi$, $\phi\text{-pretrain timesteps}$, $\phi\text{-update}$ timesteps
            \STATE Initialize $\polap$ and  $\pollp$
            \STATE Train $\pollp$ for $\phi\text{-pretrain timesteps}$ with PPO that optimizes performance under $R_\cL$ and $\polap$
            \FOR{epoch = 0, 1, $\ldots$}
                \STATE Update $\pollp$ for $\phi\text{-update}$ timesteps with PPO that optimizes performance under $R_\cL$ and $\polap$
        
                \STATE Collect trajectory $\tau^\phi$ with $\polap$ and $\pollp$
                \STATE Collect trajectory $\tau^\dagger$ with $\polap$ and $\targetpi$
                
                \STATE $bc\_loss \gets cost\_fn(\initpi, \polap)$
                \COMMENT{either cross entropy or kl-divergence}
                
                \STATE $loss^\dagger \gets \mathcal L^{PPO}(\tau^\dagger, \polap)$
                \STATE $loss^\phi \gets \mathcal L^{PPO}(\tau^\phi, \polap)$
                \STATE $policy\_loss \gets \frac{1}{|\tau^\phi| + |\tau^\dagger|} \cdot (loss^\phi - loss^\dagger)$ 
 \COMMENT{where $|\tau|$ is the length of trajectory $\tau$, i.e., the number of timesteps in $\tau$}
                \STATE $loss \gets  bc\_loss + \lambda \cdot policy\_loss$ 
                \STATE Update $\polap$ with gradients of $loss$
                \STATE Update critic network of adversary with $\tau^\phi$ and $\tau^\dagger$
            \ENDFOR
        \end{algorithmic}
    
\end{algorithm}

\section{Experiments}

In this section, we demonstrate the efficacy of our algorithmic approaches through simulation-based experiments. As explained in the introduction, our setting differs from those studied in prior work, so our algorithms are not directly comparable to approaches from prior work.  Hence, we compare our algorithms against their simplified versions and naive baselines.
Additional results and implementation details, including running times and training parameters, are provided in the appendix.\footnote{The code for this paper is available at \url{https://github.com/gradanovic/rl-implicit-poisoning-attacks}.}

\subsection{Experiments for Conservative Policy Search}\label{sec.experiments.CPS}

\begin{figure*}[ht]
	\centering
	\begin{subfigure}{0.96\textwidth}
    	\centering
		\resizebox{0.3\linewidth}{!}{\includegraphics{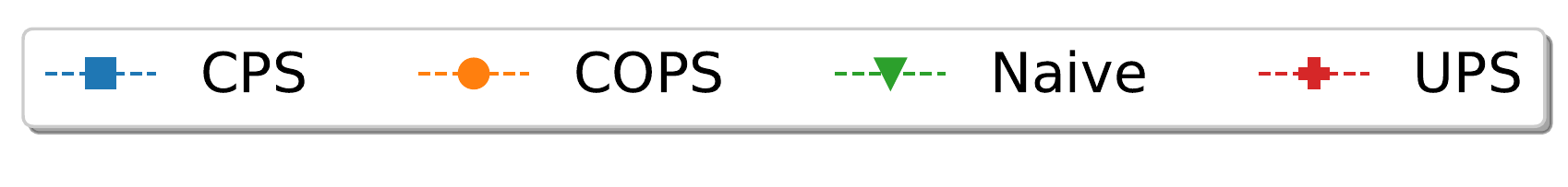}}
		\label{fig.general_legend}
	\end{subfigure}
	\begin{subfigure}{0.24\textwidth}
    	\centering
		\resizebox{0.98\linewidth}{!}{\includegraphics{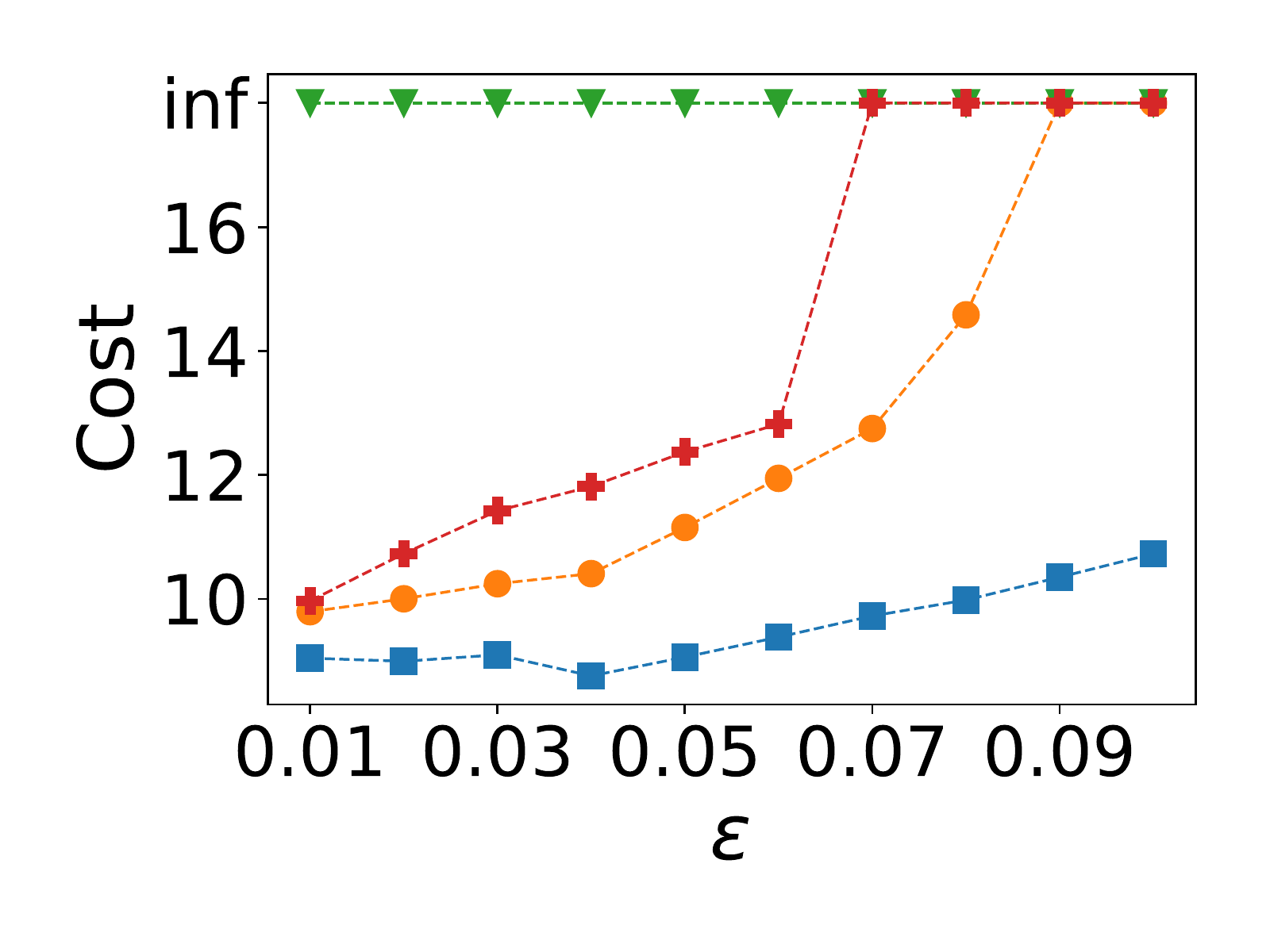}}
		\caption{Navigation, $\eps$}\label{fig.cost_vs_eps_navigation}
	\end{subfigure}
	\begin{subfigure}{0.24\textwidth}
    	\centering
		\resizebox{0.98\linewidth}{!}{\includegraphics{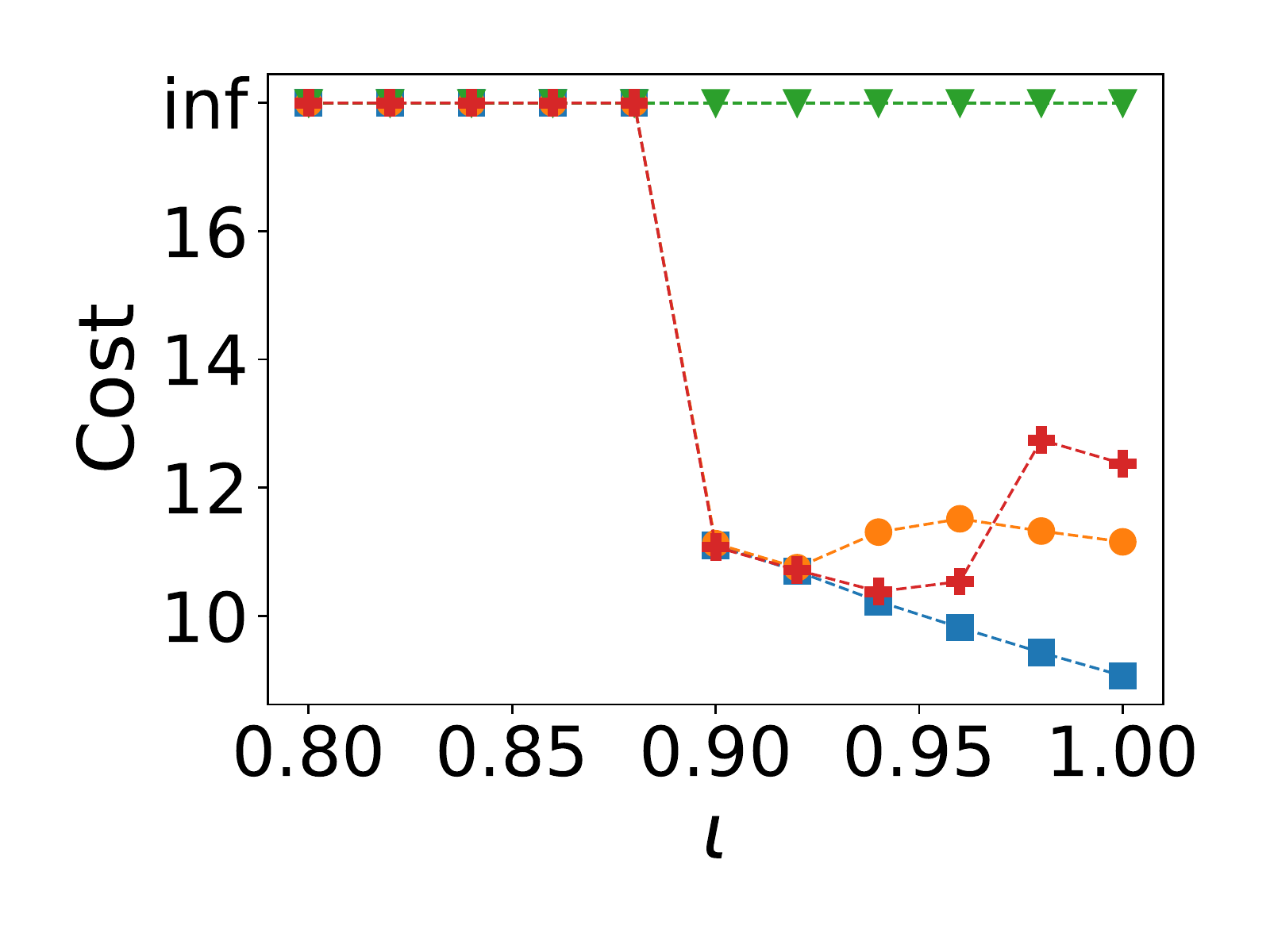}}
		\caption{Navigation, Influence}\label{fig.cost_vs_inf_navigation}
	\end{subfigure}
	\begin{subfigure}{.24\textwidth}
    	\centering
		\resizebox{0.98\linewidth}{!}{\includegraphics{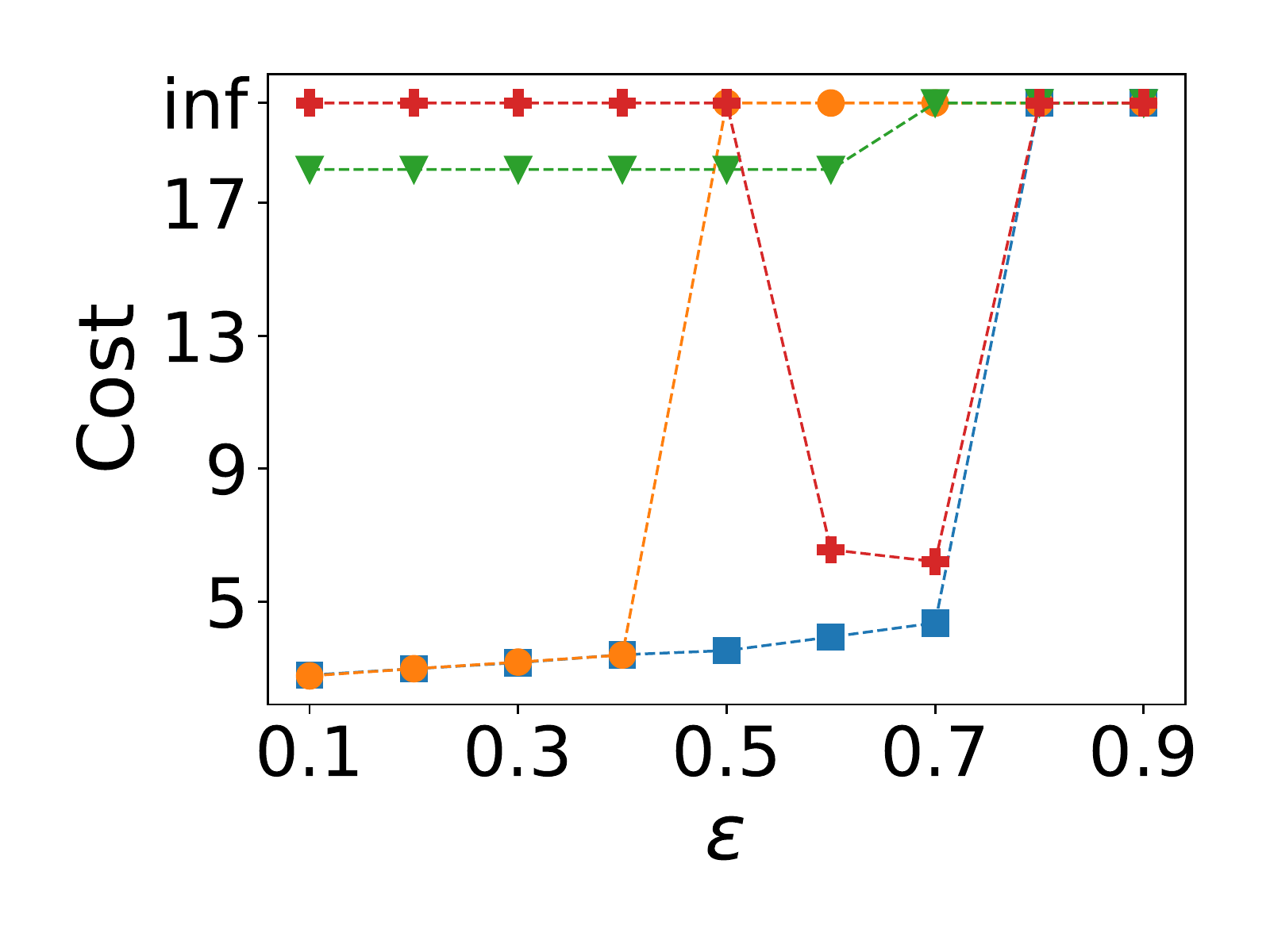}}
		\caption{Inventory, $\eps$}\label{fig.cost_vs_eps_inventory}
	\end{subfigure}
	\begin{subfigure}{.24\textwidth}
    	\centering
		\resizebox{0.98\linewidth}{!}{\includegraphics{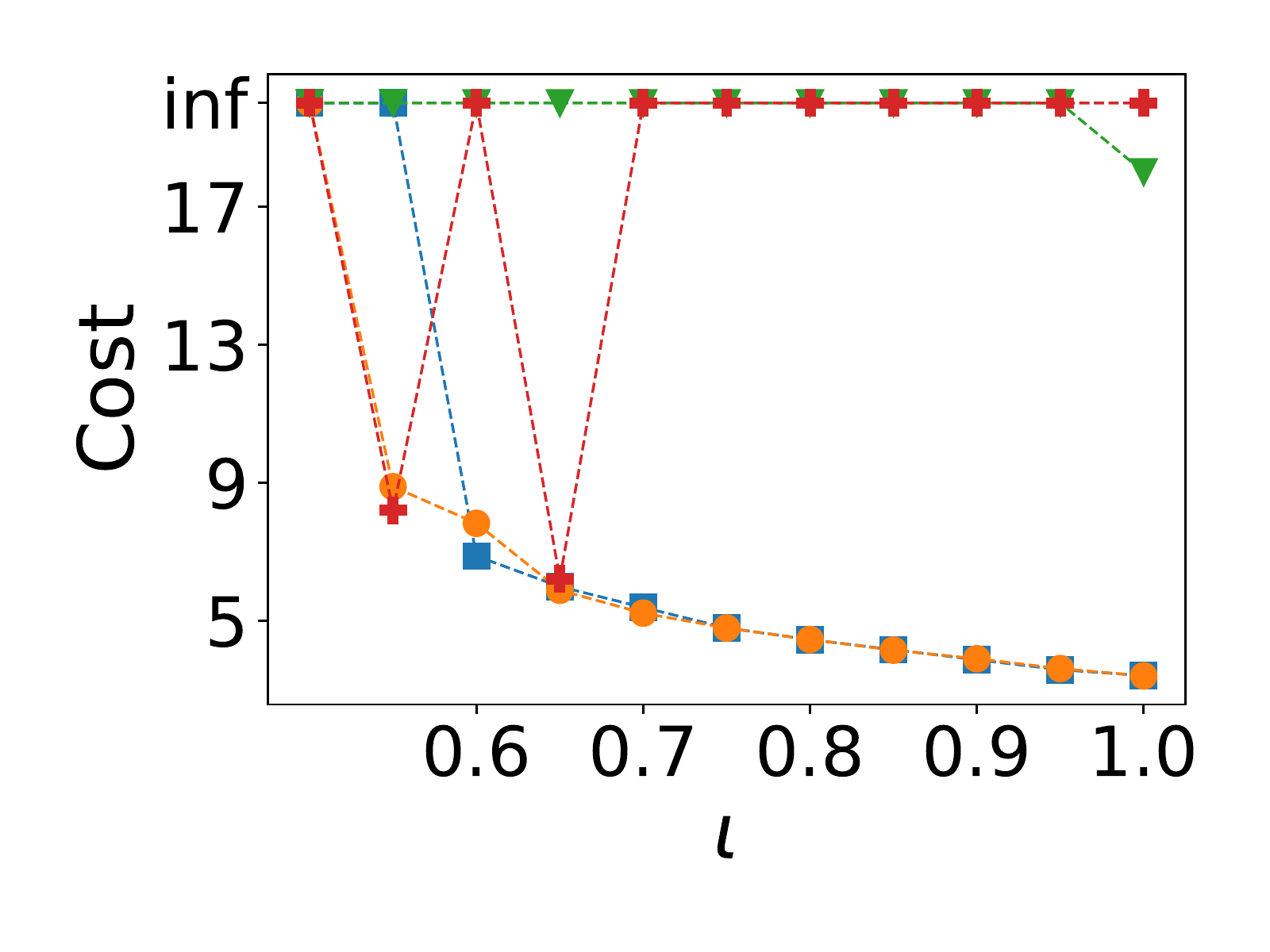}}
		\caption{Inventory, Influence}\label{fig.cost_vs_inf_inventory}
	\end{subfigure}
    \caption{The cost of the attack as a function of the victim’s sub-optimality and the adversary’s influence over the victim's peer agent.
    We use the same cost function as for the characterization results. The default value of $\eps$ is $0.05$ for Navigation and $0.4$ for Inventory. When $\text{inf}$ is reached, no solution was found. As explained in the text, $\eps$ parameter controls the sub-optimality of the learner. Fig. \ref{fig.cost_vs_eps_navigation} and Fig. \ref{fig.cost_vs_eps_inventory} show that the more sub-optimal the learner is, the harder it is to force a target policy. We further vary the influence of the attacker $\iota$ by executing policy $\pola^{\iota}(s, a) = (1-\iota) \cdot \initpi(s, a) + \iota \cdot \pola(s, a)$ instead of $\pola$. 
    Fig. \ref{fig.cost_vs_inf_navigation} and Fig. \ref{fig.cost_vs_inf_inventory} show
    that the lower the influence of the attacker is, the harder it is to force a target policy. The default value of $\iota$ is $1.0$.
    }
	\label{fig.performance}
\end{figure*}

\begin{figure*}[ht]
	\centering
		\begin{subfigure}{0.96\textwidth}
    	\centering
		\resizebox{0.7\linewidth}{!}{\includegraphics{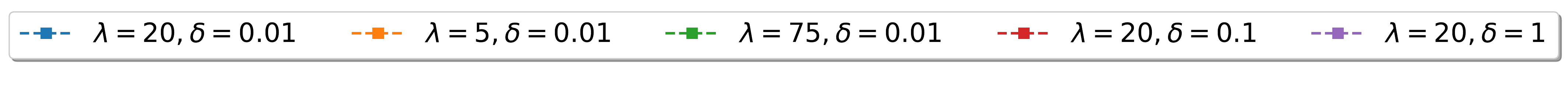}}
		\label{fig.specific_legend}
	\end{subfigure}
	\begin{subfigure}{0.24\textwidth}
    	\centering
		\resizebox{0.98\linewidth}{!}{\includegraphics{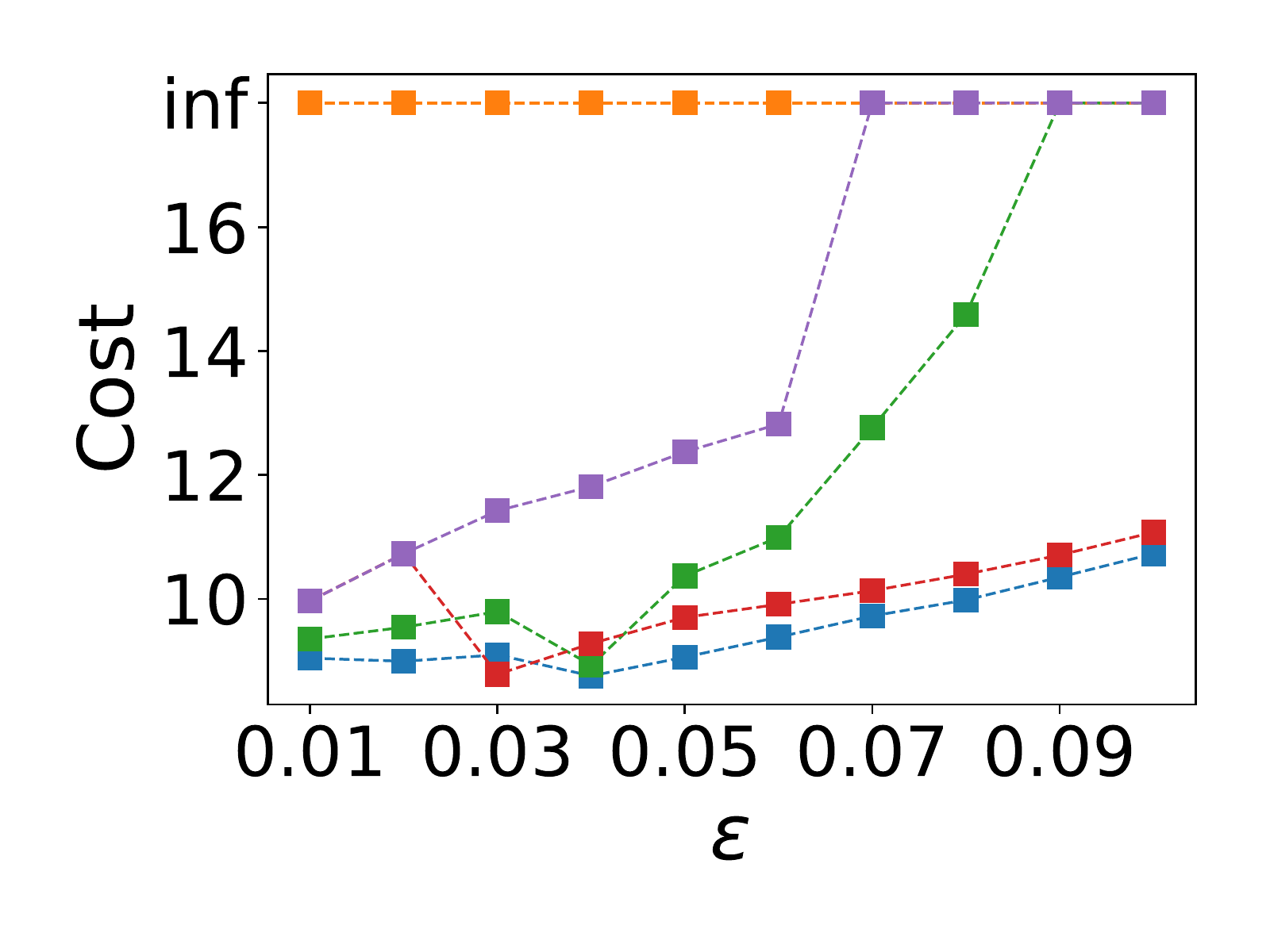}}
		\caption{Navigation, $\eps$}\label{fig.cost_vs_eps_navigation_b}
	\end{subfigure}
	\begin{subfigure}{0.24\textwidth}
    	\centering
		\resizebox{0.98\linewidth}{!}{\includegraphics{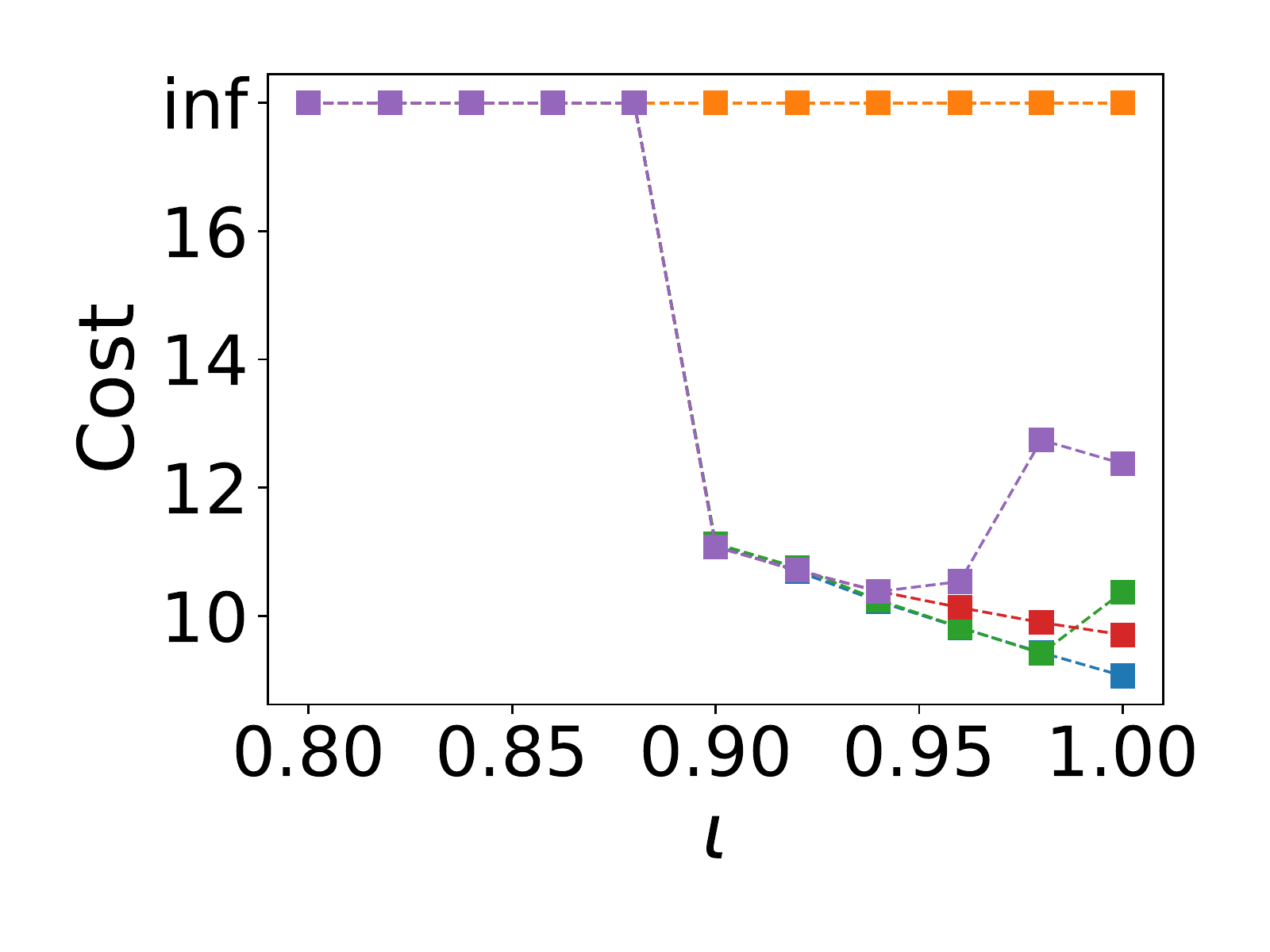}}
		\caption{Navigation, Influence}\label{fig.cost_vs_inf_navigation_b}
	\end{subfigure}
	\begin{subfigure}{.24\textwidth}
    	\centering
		\resizebox{0.98\linewidth}{!}{\includegraphics{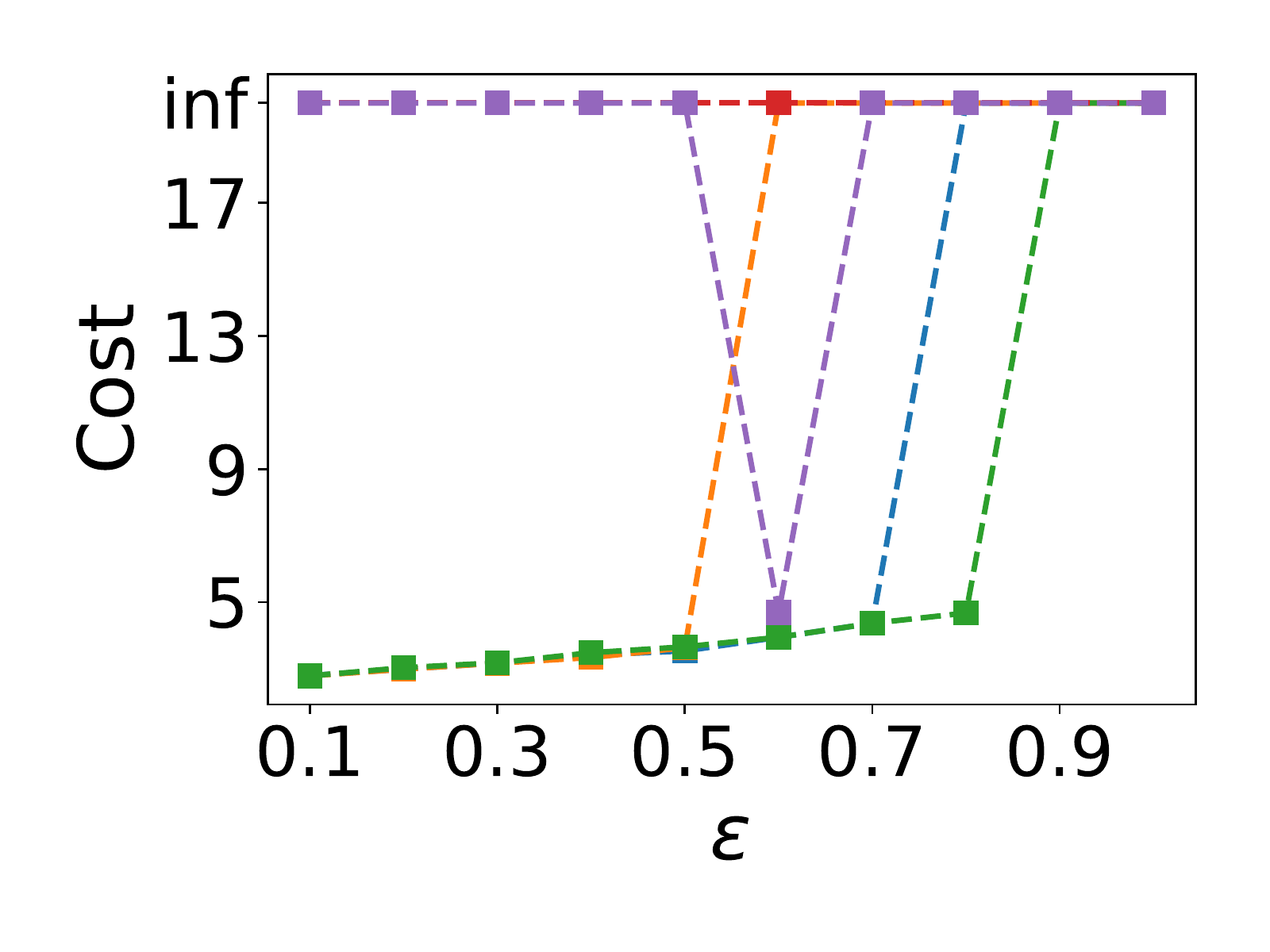}}
		\caption{Inventory, $\eps$}\label{fig.cost_vs_eps_inventory_b}
	\end{subfigure}
	\begin{subfigure}{.24\textwidth}
    	\centering
		\resizebox{0.98\linewidth}{!}{\includegraphics{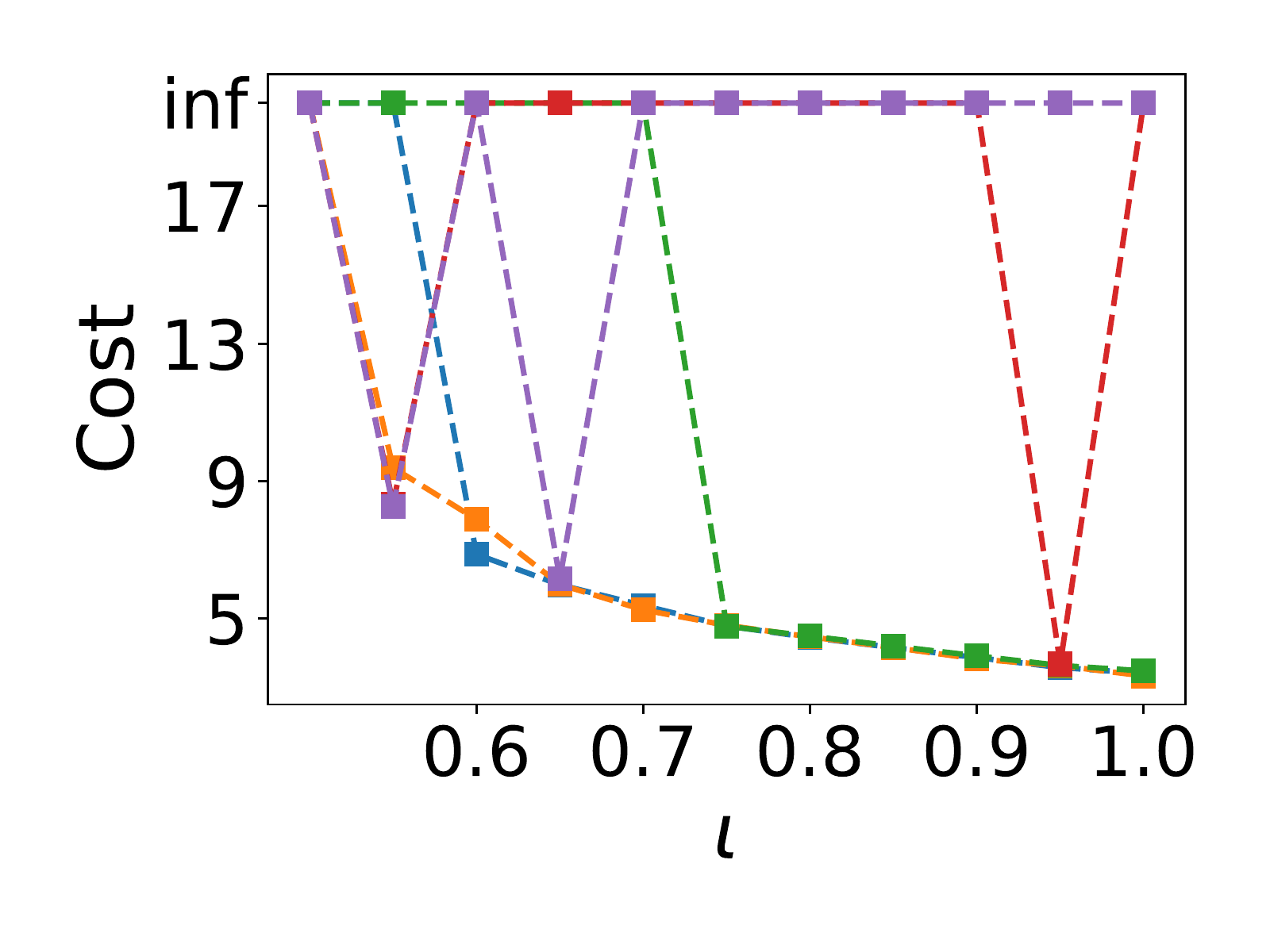}}
		\caption{Inventory, Influence}\label{fig.cost_vs_inf_inventory_b}
	\end{subfigure}
    \caption{ The effect of hyperparameters $\lambda$ and $\delta$ on the performance of conservative policy search (CPS). The plots correspond to those from Figure \ref{fig.performance}. 
    The results suggest that $\lambda$ and $\delta$ affect the success rate of CPS in finding a feasible solution.  
    To find a suitable hyperparameters, one can 
    run a meta-search over hyperparameters.
    }
	\label{fig.ablation}
\end{figure*}

We consider two environments based on or inspired by prior work \cite{Puterman1994,rakhsha2021policy}, but modified to fit the two-agent setting of this paper. 

\textbf{Navigation Environment.} This environment is based on the navigation environment from \cite{rakhsha2021policy}, developed for testing environment poisoning attacks on a single RL agent in a tabular setting. We refer the reader to \cite{rakhsha2021policy} for the description of the original environment and to the appendix for the full description of the two-agent variant that we introduce. The original environment is ergodic, contains $9$ states and the action space of an agent specifies in which direction (``left'' or ``right'') the agent should move. The two-agent variant has an extended action space to include the actions of the attacker, who has the same action space as the victim agent. Rewards and transitions  primarily depend on whether the actions of the two agents match, e.g., if the agents' actions match, the victim agent moves in the desired direction with high probability and obtains a positive reward. The default policy of the attacker is to always take ``left'', while the target policy is that the victim takes action ``right'' in each state.

\textbf{Inventory Management.} 
We consider a modified version of the inventory management environment from \cite{Puterman1994}, with two agents. As in the original version, we have a manager of a warehouse that decides on the current inventory of a warehouse (the number of stocks/items in the warehouse). In our two agent version of the environment, the victim agent is controlling the amount of stock in the inventory and the attacker is controlling the demand. The victim's actions are ``buy'' actions that select between $0$ and $M-1$ items. The attacker's actions are ``create demand'' of $0$ to $M-1$ items. In our experiments, we set $M = 10$ and $\gamma = 0.9$. 
The default policy of the attacker is to select the demand uniform at random over all possible values. 
The target policy is defined by the following rule: if there are more than $k = 7$ items, do not buy anything, otherwise buy $k - s$ items.
Other details of this environment are explained in the appendix. Note that this is a non-ergodic environment.

\begin{figure*}[ht]
    \begin{subfigure}{1.0\textwidth}
		\resizebox{0.5\linewidth}{!}{\includegraphics{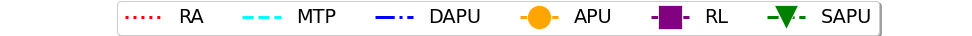}}
		\label{fig.Push1D_legend}
		\resizebox{0.5\linewidth}{!}{\includegraphics{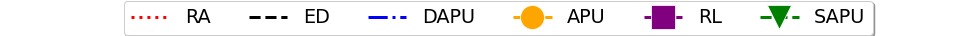}}
		\label{fig.Push2D_legend}
	\end{subfigure}
	\centering
	\begin{subfigure}{0.24\textwidth}
    	\centering
		\resizebox{0.98\linewidth}{!}
		{\includegraphics{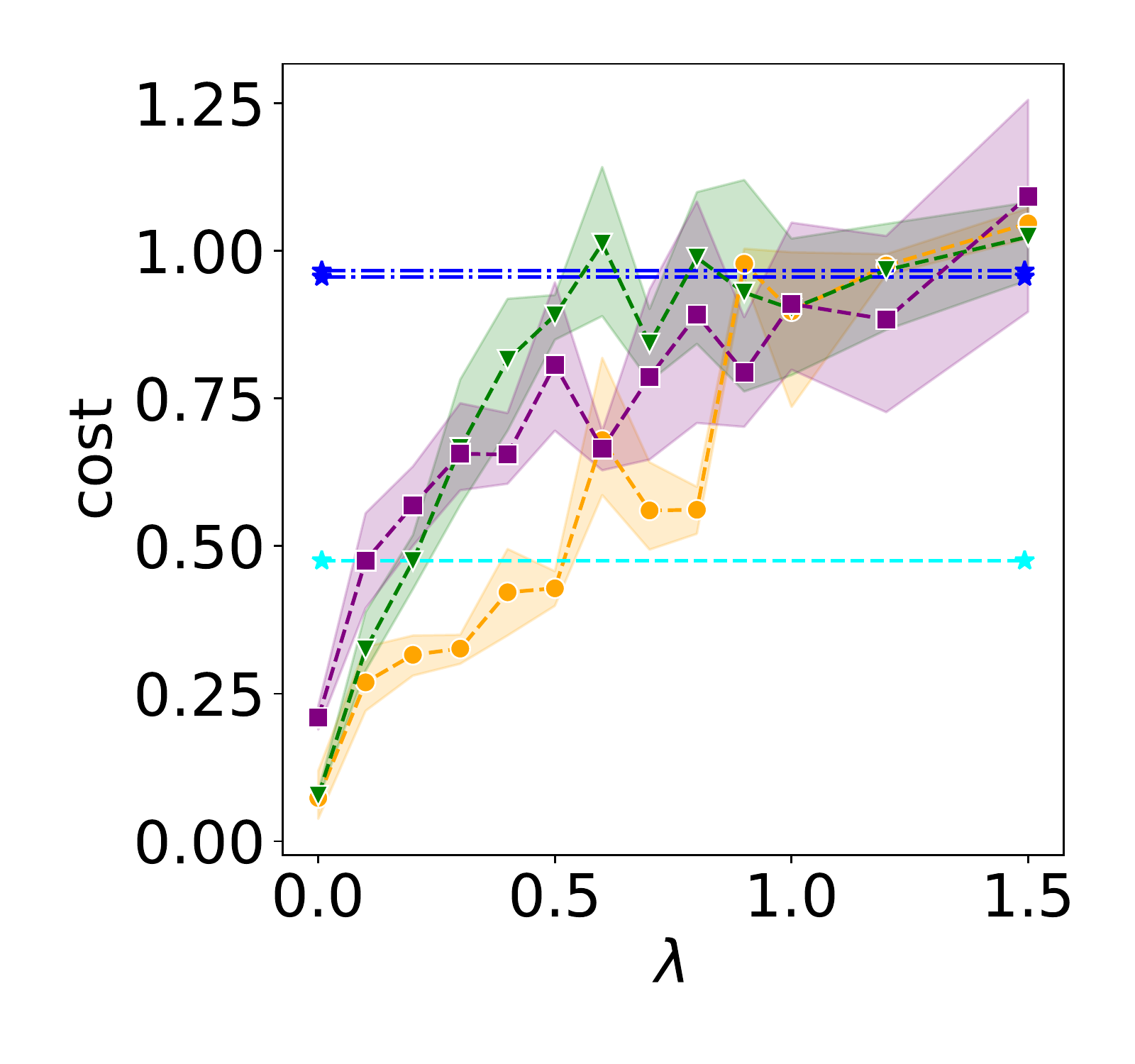}}
		\caption{1D Push: Cost, $\lambda$}\label{fig.dist_vs_lambda_g}
	\end{subfigure}
	\begin{subfigure}{.24\textwidth}
    	\centering
		\resizebox{0.98\linewidth}{!}
		{\includegraphics{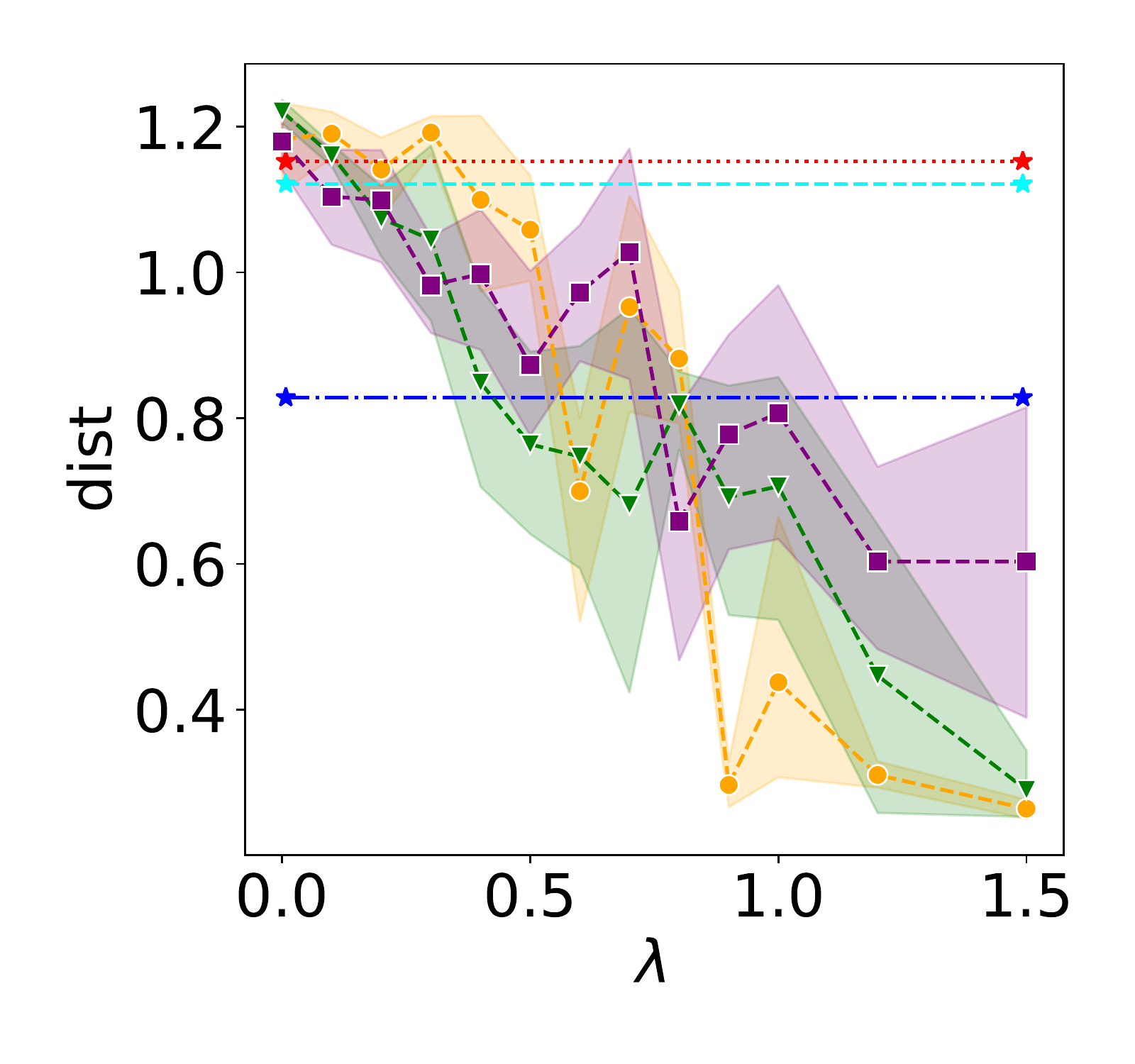}}
		\caption{1D Push: Distance, $\lambda$}\label{fig.success_vs_inf_g}
	\end{subfigure}
    \begin{subfigure}{0.24\textwidth}
    	\centering
		\resizebox{0.98\linewidth}{!}
		{\includegraphics{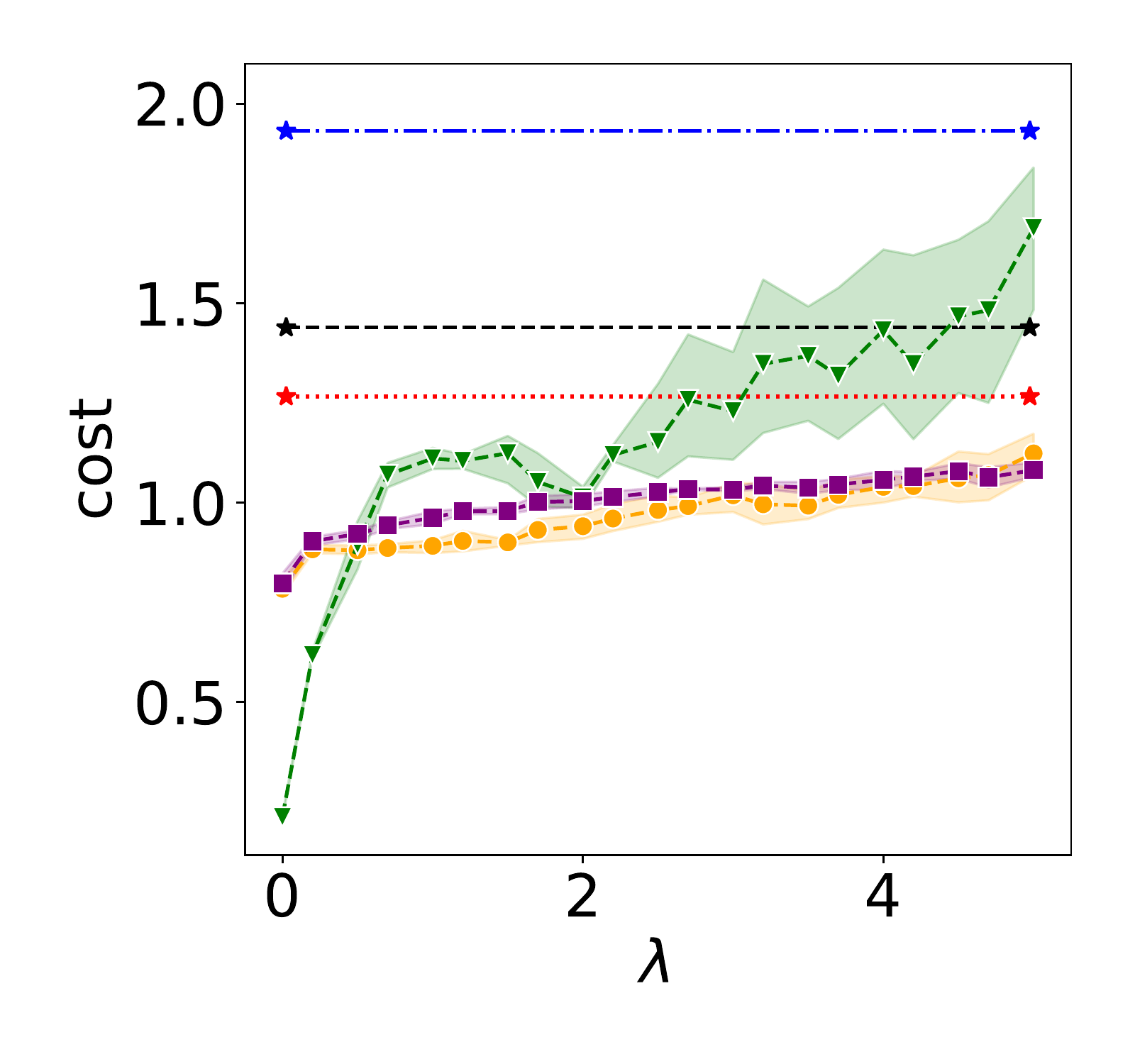}}
		\caption{2D Push: Cost, $\lambda$}\label{fig.cost_vs_lambda_g}
	\end{subfigure}
	\begin{subfigure}{.24\textwidth}
    	\centering
		\resizebox{0.98\linewidth}{!}
		{\includegraphics{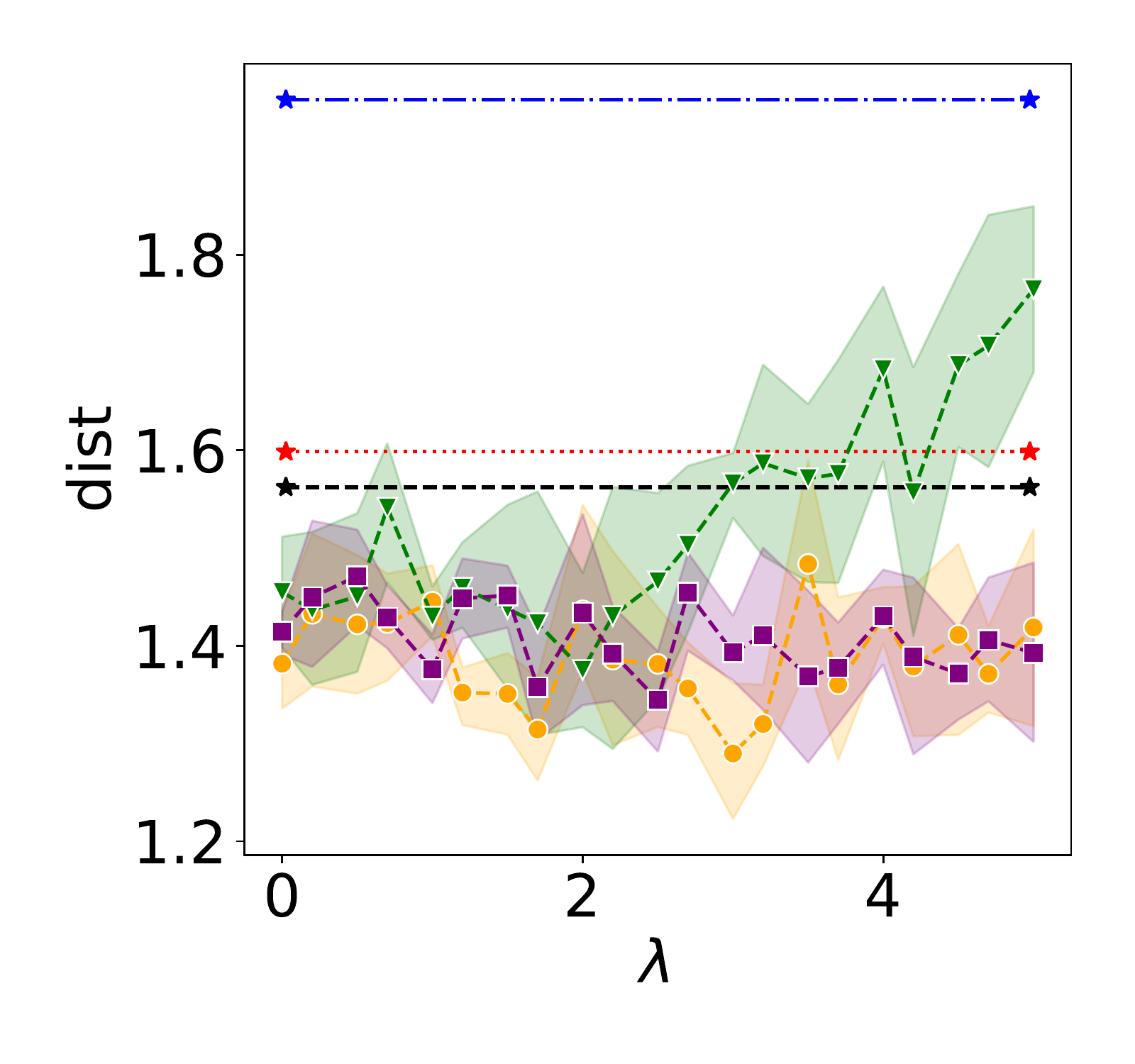}}
		\caption{2D Push: Distance, $\lambda$}\label{fig.lambda_vs_cost_g}
	\end{subfigure}
    \caption{
    Figures show the test-time performance of each adversary in 1D and 2D Push, for different values of $\lambda$. Two empirical performance metrics are plotted, $\cost$ and $\dist$, which reflect the cost of the attack and the distance to the attack goal, respectively. More concretely, for a given adversary-victim pair $(\theta, \phi)$, we sample $K$ trajectories $\tau_k$ of length $T$ and calculate $\cost(\{\tau_k\}) = \frac{1}{T\cdot K} \sum_{s \in \tau_k}\norm{\polap(s, \cdot ) - \initpi(s, \cdot )}_{1}$ and $\dist(\{\tau_k\}) = \frac{1}{T\cdot K} \sum_{s \in \tau_k}\norm{\pollp(s, \cdot ) - \targetpi(s, \cdot )}_{1}$. 
    For each $\lambda$, we train $5$ adversarial policies (using $5$ different random seeds). For each adversarial policy, we train $5$ victim policies (using $5$ different random seeds) against this adversarial policy. The results show the mean and 95\% confidence intervals of the  obtained data points. 
    In the appendix, we provide the confidence intervals 
    for baselines whose behavior does not change with $\lambda$. 
    }
	\label{fig.alternating_updates}
\end{figure*}

\textbf{Results. } In order to show the efficacy of our conservative policy search algorithm, 
we consider $4$ different algorithms: {\em Naive} baseline--in the Navigation environment it sets $\pi_1$ to always take ``right'', and in the Inventory Management, $\pi_1$ buys $7$ items; 
b) {\em Conservative PS} ({\bf CPS})--the policy search algorithm from  
the previous section
that sets $\lambda = 20$, $\delta = 0.01$, and $\delta_{\eps} = 0.1$; c) {\em Constraints Only PS} ({\bf COPS})--a modification of CPS that ignores $\Cost$; d) {\em Unconservative PS} ({\bf UPS})--a modification of CPS that sets $\delta = 1$.\footnote{To solve \eqref{prob.instance_1.rel}, we use CVXPY solver in our experiments (see \cite{diamond2016cvxpy,agrawal2018rewriting}). We provide additional details in the appendix.} 

Fig. \ref{fig.performance} compares these algorithmic approaches along two dimensions. We test the effect of the victim's sub-optimality on the cost and the effect the attacker's influence over the victim's peer on the cost. The results show that our algorithmic approach can lead to a significant cost reduction compared to the baselines. These results demonstrate the importance of having: a) a cost-guided search that does not only aim to satisfy the constraint of the optimization problem, but also minimizes the attack cost (CPS outperforms COPS), b) conservative updates that account for the change in occupancy measures when adopting a new solution (CPS outperforms UPS). 
Fig. \ref{fig.ablation} shows the effect that the  hyperparameters have on the performance of CPS. These results confirm that conservative updates are important, especially in non-ergodic environments (Fig. \ref{fig.cost_vs_eps_inventory_b} and Fig. \ref{fig.cost_vs_inf_inventory_b} for $\delta = 0.1$ and $\delta = 1.0$), where the performance critically depends on $\eps$ and $\iota$. We observe similar instabilities for UPS in Fig. \ref{fig.cost_vs_eps_inventory} and Fig. \ref{fig.cost_vs_inf_inventory}.

\subsection{Experiments for Alternating Policy Updates}\label{sec.experiments.apu}

\textbf{Push Environments.} We consider two multi-agent RL environments inspired by prior work \cite{mordatch2018emergence,terry2021pettingzoo}. We refer to them as Push environments. Both of them have a continuous state space, and are modifications of 
environments from \cite{terry2021pettingzoo}.
 In  Push environments, the victim is rewarded based on the distance to a given goal location. 
 The target policy stands still if the distance to the goal is within a certain interval, and otherwise moves towards this area. The default policy of the adversary moves towards the goal and stays there. 
 We consider two variants. In 1D Push, the agents can move move left, right, or stand still, on a line segment. In 2D Push, the agents have two additional actions, up and down, and are located 
in a plane. Compared to the 1D version, the reward of the victim has an additional penalty term since the adversary cannot easily ``block'' the learner from reaching the goal. Note that in 2D Push the target policy is stochastic and encodes the direction to the goal (while minimizing its support). I.e., outside of the annulus where the victim should stay still, the target policy is identified by the vector that connects the victim's position and the closest point of the annulus. We specify other details in the appendix. 
 
\textbf{Results.} To test the efficacy of our alternating policy updates approach, we consider $4$ different algorithms trained with Proximal Policy Optimization (PPO)~\cite{schulman2017proximal}: a) {\em Random Adversary} ({\bf RA}) baseline--the adversary takes actions uniformly at random; b) {\em Move to Target Position} ({\bf MTP}) baseline for 1D Push--the adversary follows a hard-coded policy that moves to the target position;
 c) {\em Equal Distance} ({\bf ED}) baseline for 2D Push--the adversary follows a hard-coded policy that keeps the same distance to the victim and goal; 
d) Alternating Policy Updates ({\bf APU})--our approach from the previous section, where PPO is used for policy updates and the victim is trained for 5 times as many episodes per epoch as the adversary; e) Random Learner ({\bf RL})--a modification of APU which fixes the victim's parameters $\phi$ to random values; 
f) Symmetric APU ({\bf SAPU})--a modification of APU in which $\theta$ and $\phi$ are updated in a symmetric manner, i.e., using the same number of episodes; g) Distance-only APU ({\bf DAPU})--a modification of APU that does not use the imitation learning loss.\footnote{We use the implementation from stable-baselines3~\cite{raffin2019stable}. We provide additional training details in the appendix.}
Fig. \ref{fig.alternating_updates} compares the test-time performance of these approaches for different values of $\lambda$. For larger values of $\lambda$, our approach outperforms naive baselines (RA, MTP, ED) both in terms of the attack cost and success; 
only MTP has a comparable attack costs for large $\lambda$ in 1D Push. In terms of the attack cost, APU achieves similar performance as its modifications in most cases, while outperforming SAPU in 2D Push. However, in terms of the success rate, it outperforms most of them for large enough $\lambda$. One exception is RL, which achieves similar performance in 2D Push. These results suggest that: a) it is important to train (the model of) the victim alongside the attacker (APU vs. RL in 1D Push), b) it is important to have asymmetric update rules that more conservatively update the adversary's policy (APU vs. SAPU), c) it is important to have a cost guided optimization that does not only aim to optimize the attack success (APU vs. DAPU). 

\begin{remark} 
Alternating Policy Updates can also be applied to Navigation and Inventory Management, and we report the experimental results for these two environments in the appendix. 
\iftoggle{longversion}{For Inventory Management, the results are qualitatively similar to the ones we obtain for Push, with Alternating Policy Updates achieving significantly smaller distance $\dist$ than the baselines. For Navigation, we do not observe significant difference between the tested methods. }{} 
\end{remark}

\section{Conclusion}

In this paper, we studied a novel form of poisoning attacks in reinforcement learning based on adversarial policies. In this attack model, the attacker utilizes the presence of another agent to influence the behavior of a learning agent. We showed that such an implicit form of poisoning differs from the standard environment poisoning attack models in RL. In particular, the implicit attack model appears to be more restrictive in that it is not always feasible, while determining its feasibility is a computationally challenging problem. In contrast, and as argued by prior work, this type of attack may be more practical as the aspects that are controlled by an attacker are expressed through an agency, i.e., the learner's peer.
Hence, we believe that our results contribute valuable insights important for understanding trade-offs between different attack models. One of the most interesting future research directions is to consider settings with more than two agents. In such settings, an attacker has to reason about the agents' equilibrium behavior, which brings additional computational challenges. 
On the other hand, the attacker could potentially use the conflicting goals of the agents in its own favor, which may decrease the cost of the attack.

\section{Acknowledgements}

This research was, in part, funded by the Deutsche Forschungsgemeinschaft (DFG, German Research Foundation) – project number 467367360.

\bibliographystyle{ACM-Reference-Format} 
\bibliography{main}

\iftoggle{longversion}{
\clearpage
\onecolumn
\appendix 
{\allowdisplaybreaks
\section{Appendix Overview}

The content of the appendix of this paper is split in the following way: 
\begin{itemize}
    \item Appendix \ref{sec.app.experiments} -- \nameref{sec.app.experiments} contains additional information about the experiments, including additional results that were not presented in the main text.

    \item Appendix \ref{sec.app.background} -- \nameref{sec.app.background} provides additional details about the setting, including supporting lemmas for proving our formal results from the main text.

    \item Appendix \ref{sec.app.comp_complexity} -- \nameref{sec.app.comp_complexity} provides the proof of our NP-hardness results in Section \ref{sec.characterization}
    .
    We also argue in this section that the optimization problem \eqref{prob.instance_1} can be efficiently solved under the assumptions of Theorem \ref{thm.upper_bound_2}.  

    \item Appendix \ref{sec.app.lower_bound} -- \nameref{sec.app.lower_bound} provides the proof of the lower bound in Section \ref{sec.characterization}. 
    
    \item Appendix \ref{sec.app.upper_bounds} -- \nameref{sec.app.upper_bounds} provides the proofs of the upper bound in Section \ref{sec.characterization}.
    This appendix provides additional result, stated in Proposition \ref{prop.upper_bound_1}, which was referenced in in Section \ref{sec.characterization}.
    
    \item Appendix \ref{sec.app.algorihtms} -- \nameref{sec.app.algorihtms} provides additional details about our algorithmic approaches from Section \ref{sec.algorithm}, including a more general version of the conservative policy search algorithm, applicable to non-ergodic environments.
\end{itemize}

\section{Experiments: Additional Details and Results}\label{sec.app.experiments}

In this section of the appendix, we provide additional details about our experiments. We start by providing a detailed description of the environments.

\subsection{Detailed Description of the Environments}

In this subsection, we provide additional details about the experimental test-beds that we used in our experiments.

\textbf{Navigation Environment.} This environment is based on the navigation environment from \cite{rakhsha2021policy}, developed  
~\\
\begin{minipage}{0.75\textwidth} 
 for testing environment poisoning attacks on a single RL agent in a tabular setting--we refer the reader to \cite{rakhsha2021policy} for the description of the original environment.
To make this environment two agent, we modify the action space to include the actions of the attacker. 
The figure on the right depicts the environment.
The attacker has the same action space as the victim (take ``left'' (blue) or ``right'' (red)), and it affects both the reward of the victim and the transition dynamics in the following way. Compared to the original environment, the new environment has a loop between states $s_1$, $s_2$, and $s_3$. In all the states, if the adversary chooses the same action as the victim, the victim gets  $r_{\text{match}} = 5.0$.  Otherwise, if their actions disagree, the obtained reward is equal to $r_{\text{mismatch}} = -5.0$. 
Additionally, states $s_0$, $s_4$, and $s_8$ have a base (action-independent) reward equal to $r_{\text{base}_1} = 5$, while state $s_2$ has base reward $r_{\text{base}_2} = 50$. In all the states, except $s_1$, $s_2$, and $s_3$, the agents move in the victim's desired direction (transitions to the intended state as determined by its action) with probability $\bar p=0.9$ if the victim's actions match that of the adversary. Otherwise, the next state is chosen uniformly at random.
\end{minipage}
\begin{minipage}{0.25\textwidth}
\captionsetup{type=figure}
\captionsetup{aboveskip=2pt,belowskip=2pt}
\resizebox{\textwidth}{!}{%
\begin{tikzpicture}[main/.style = {draw, circle}, node distance={15mm}, thick] 
\node[main] (0) {$s_0$};
\node[main] (1) [right of=0] {$s_1$};
\node[main] (2) [below right of=1] {$s_2$}; 
\node[main] (3) [above right of=2] {$s_3$};
\node[main] (4) [above of=3] {$s_4$};
\node[main] (5) [above left of=4] {$s_5$};
\node[main] (6) [above of=5] {$s_6$};
\node[main] (7) [above right of=4] {$s_7$};
\node[main] (8) [above of=7] {$s_8$};

\draw (0) edge[->, bend left, sloped, red] (1);
\draw (1) edge[->, bend left, sloped, red] (2);
\draw (2) edge[->, bend left, sloped, red] (3);
\draw (3) edge[->, bend left, sloped, red] (4);
\draw (4) edge[->, bend left, sloped, red] (7);
\draw (7) edge[->, bend left, sloped, red] (8);
\draw (8) edge[->, loop left, sloped, red] (8);
\draw (6) edge[->, bend left, sloped, red] (5);
\draw (5) edge[->, bend left, sloped, red] (4);

\draw (0) edge[->, loop left, sloped, blue] (0);
\draw (1) edge[->, bend left, sloped, blue] (0);
\draw (2) edge[->, bend right, sloped, blue] (3);
\draw (3) edge[->, bend right, sloped, blue] (1);
\draw (4) edge[->, bend left, sloped, blue] (5);
\draw (5) edge[->, bend left, sloped, blue] (6);
\draw (7) edge[->, bend left, sloped, blue] (4);
\draw (8) edge[->, bend left, sloped, blue] (7);
\draw (6) edge[->, loop right, sloped, blue] (6);
\end{tikzpicture} 
}%

\caption{Navigation}
\label{fig.navigation}
\end{minipage}
  The initial state is $s_0$. In state $s_1$ (resp. $s_3$), the adversary (resp. victim) controls the transitions and the agents move in the intended direction, $s_0$ or $s_2$ (resp. $s_1$ or $s_4$), with probability $\bar p=0.9$, otherwise, they transition to a random state (chosen u.a.r.). From state $s_2$, agents transition to $s_3$ with probability $\bar p=0.9$ regardless of their actions, and with probability $1-\bar p=0.1$, the next state is chosen uniformly at random. The default policy of the attacker is to always take ``left'', while the target policy is that the victim takes action ``right'' in each state. We use this environment for testing the efficacy of our conservative policy search algorithm. Note that the environment is ergodic.

\textbf{Inventory Management.} 
We consider a modified version of the inventory management environment from \cite{Puterman1994}, with two agents. As in the original version we have a manager of a warehouse that decides on the current inventory of a warehouse (the number of stocks/items in the warehouse). In our two agent version of the environment, the victim is controlling the amount of stock on the inventor and the attacker is controlling the demand. In our formalism, state $s \in \{0, 1, …, M-1\}$ indicates the number of items in the warehouse. In our experiment, $M=10$. The victim's actions are ``buy'' actions that select between $0$ and $M-1$ items. Since the capacity of the warehouse is $M-1$, some of the actions are not valid in all the states, i.e., $s + a_2 < M$. The attacker's actions are ``create demand'' of $0$--$M-1$ items. Note that if the demand exceeds the supply, it will be rejected. Rewards for this environment is defined as $R_\cL(s, a_\cA, a_\cL) = \ind{a_1 \le a_2 + s} \cdot \text{sell}(a_1) - \text{hold}(s + a_2) - \text{buy}(a_2)$, where $\text{sell}(x) = 10 \cdot x$, $\text{hold}(x) = x$, and $\text{buy}(x) = 4+ 2\cdot x$ (for $x>0$). Transitions are defined by $P(s, a_1, a_2, s) = 1.0$ if $s + a_2 < a_1$,  and  otherwise $P(s, a_1, a_2, s+a_2-a_1) = 1.0$. The target policy is defined by the following rule: if there are more than $k = 7$ items, don’t buy anything, otherwise buy $k - s$ items. The discount factor is set to $\gamma = 0.9$ and the starting state is $s = 0$. 

\textbf{Push Environments.} We consider two multi-agent RL environments inspired by prior work \cite{mordatch2018emergence,terry2021pettingzoo}. We refer to them as 1D and 2D Push; both of them have a continuous state space, and are modifications of the environments from \cite{terry2021pettingzoo}.\footnote{In particular, ``Simple Adversary'' environment. }
 In Push environments, the victim is rewarded based on the distance to the goal with a potential penalty if it is close to the adversary. The target policy stands still if the distance to the goal is within a certain interval, and otherwise moves towards this area. The default policy of the adversary moves towards the goal and stays there. We consider two variants, 1D Push and 2D Push, specified as follows and shown in Fig. \ref{fig.1D_env} and Fig. \ref{fig.2D_env}, respectively. In the 1D version, the agents can move left, right, or stand still, on a line segment ($L = 5$ units long).
 Relative to the left end of the line segment, the goal is located $x_{g} =  3$ units away, the victim is initially located at a random position left of the target, the adversary at a random position right of the target.
 Transitions are defined as in \cite{terry2021pettingzoo} (``Simple Adversary'' environment).
Both agents observe their positions, velocities, and the position of the goal. 
The victim's reward depends on its distance to the goal and the adversary and is equal to $r^t = -(x_{l}^t-x_{g}^t)^2$.
Note that the victim cannot go through the adversary, i.e., the adversary can block the victim from reaching the goal. The target policy is deterministic and takes left if $x_{l} < 2.25$, right if $x_{l} > 2.75 $, and no-op (stand still) otherwise. The 2D version is an extension of 1D. In this version, the agents have two additional actions, up and down, and are located 
in a plane.
The agents' initial locations, $\mathbf x_{l}^0$ and $\mathbf x_{a}^0$,  are selected randomly, relative to the goal's position $\mathbf x_{g}$.
The transitions are defined as in 1D, but extended to vertical direction (for up and down actions). The reward of the victim at time $t$ is $r^t = -(\mathbf x_{l}^t-\mathbf x_{g}^t)^2 - 20 \cdot \ind{(\mathbf x_{l}^t-\mathbf x_{a}^t)^2 \le 0.5}$. Compared to the 1D version, we have an additional penalty term since the adversary cannot easily ``block'' the victim from reaching the goal. Outside of the annulus where the victim should stay still, the target policy is identified by the vector that connects the victim's position and the closest point of the annulus: the target policy is stochastic and encodes the direction of this vector (while minimizing its support). 

\begin{figure*}[ht]
    	\centering
		\resizebox{0.8\linewidth}{!}
		{\frame{\includegraphics{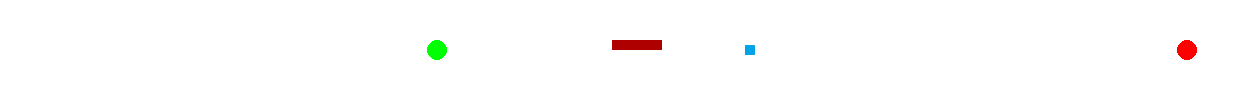}}}
		
		\caption{1D Push environment. The goal is near the center of the line segment and is colored in blue, the victim is colored in green, the adversary is colored in red, and the target area is colored in dark red.
		}\label{fig.1D_env}
\end{figure*}

\begin{figure*}[ht]
    \centering
		\resizebox{0.5\linewidth}{!}
		{\frame{\includegraphics{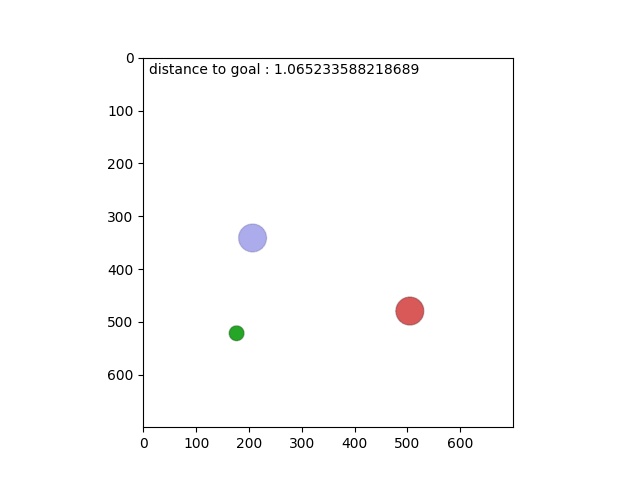}}}
		
		\caption{2D Push environment. The goal is colored in green, the victim is colored in red, the adversary is colored in blue. }\label{fig.2D_env}
\end{figure*}

\subsection{Additional Results for Alternating Policy Updates}

In this subsection, we provide additional experimental results for Alternating Policy Updates (APU), on all the environments we considered in this work. 

\textbf{Navigation and Inventory Management.} As mentioned in the main part of the paper, APU can also be applied to Navigation and Inventory Management. In Fig. \ref{fig.additional_APU_results}, we show the same set of resutls as in Fig. \ref{fig.alternating_updates} but for Navigation and Inventory Management. As can be seen from Fig. \ref{fig.Navigation_cost_APU} and \ref{fig.Navigation_dist_APU}, the tested methods have similar performance in Navigation. In Inventory Management (Fig. \ref{fig.IM_cost_APU} and Fig. \ref{fig.IM_dist_APU}), Alternating Policy Updates outperforms the baseline in terms of $dist$, which indicates that it is more successful in forcing the target policy. On the other hand, most of the method perform similarly in terms of $dist$, with Random Adversary (RA) having significant fluctuations. 

\begin{figure*}[ht]
    \centering
		\begin{subfigure}{0.96\textwidth}
		\resizebox{0.5\linewidth}{!}{\includegraphics{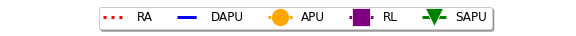}}
		\label{fig.Navigation_legend}
		\resizebox{0.5\linewidth}{!}{\includegraphics{figures/alternating_policy_updates/appendix/Inventorylegend.png}}
		\label{fig.Inventory_legend}
	\end{subfigure}
	\centering
	\begin{subfigure}{0.24\textwidth}
    	\centering
		\resizebox{0.98\linewidth}{!}
		{\includegraphics{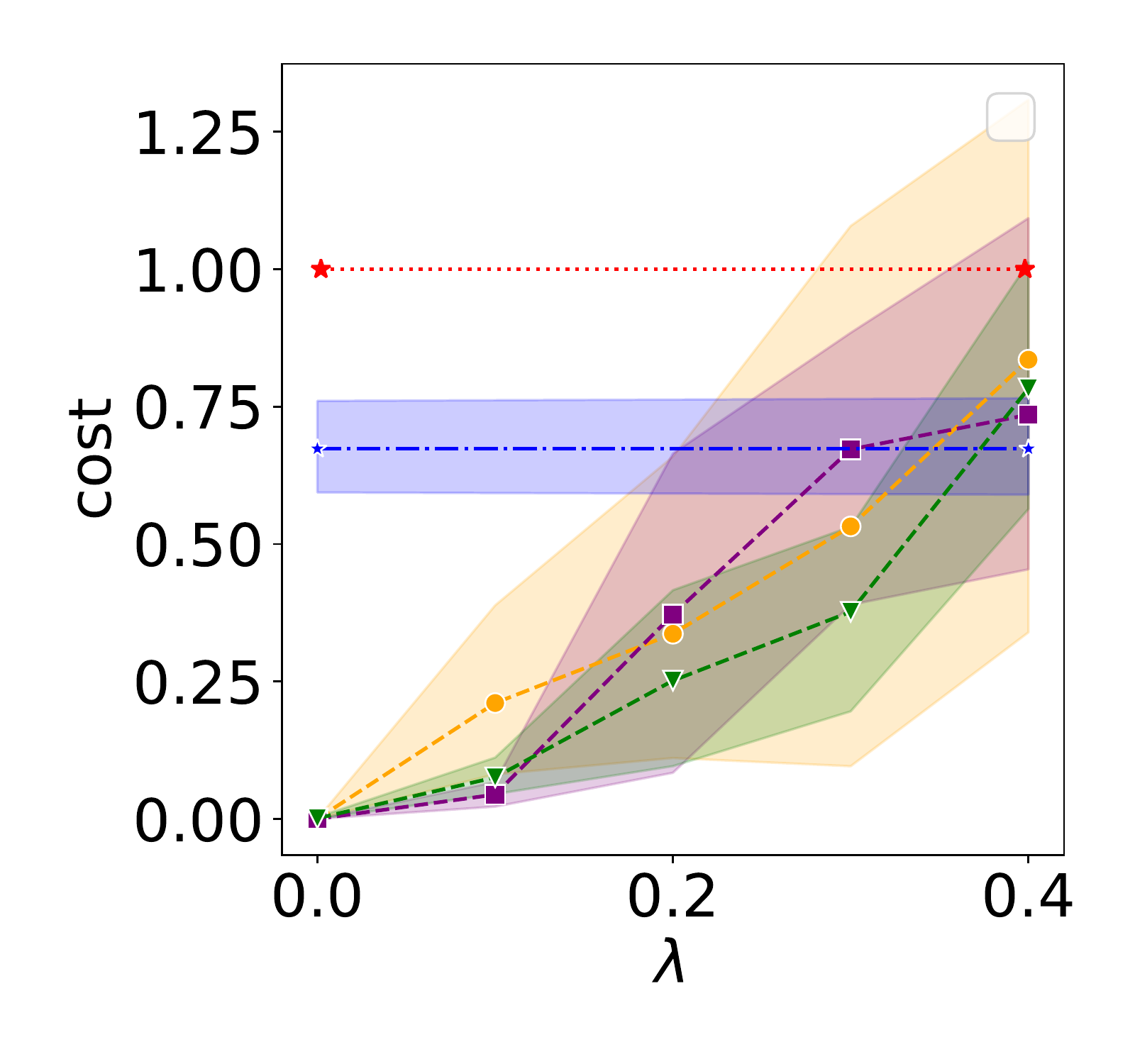}}
		
		\caption{Navigation: Cost, $\lambda$}\label{fig.Navigation_cost_APU}
	\end{subfigure}
	\begin{subfigure}{.24\textwidth}
    	\centering
		\resizebox{0.98\linewidth}{!}
		{\includegraphics{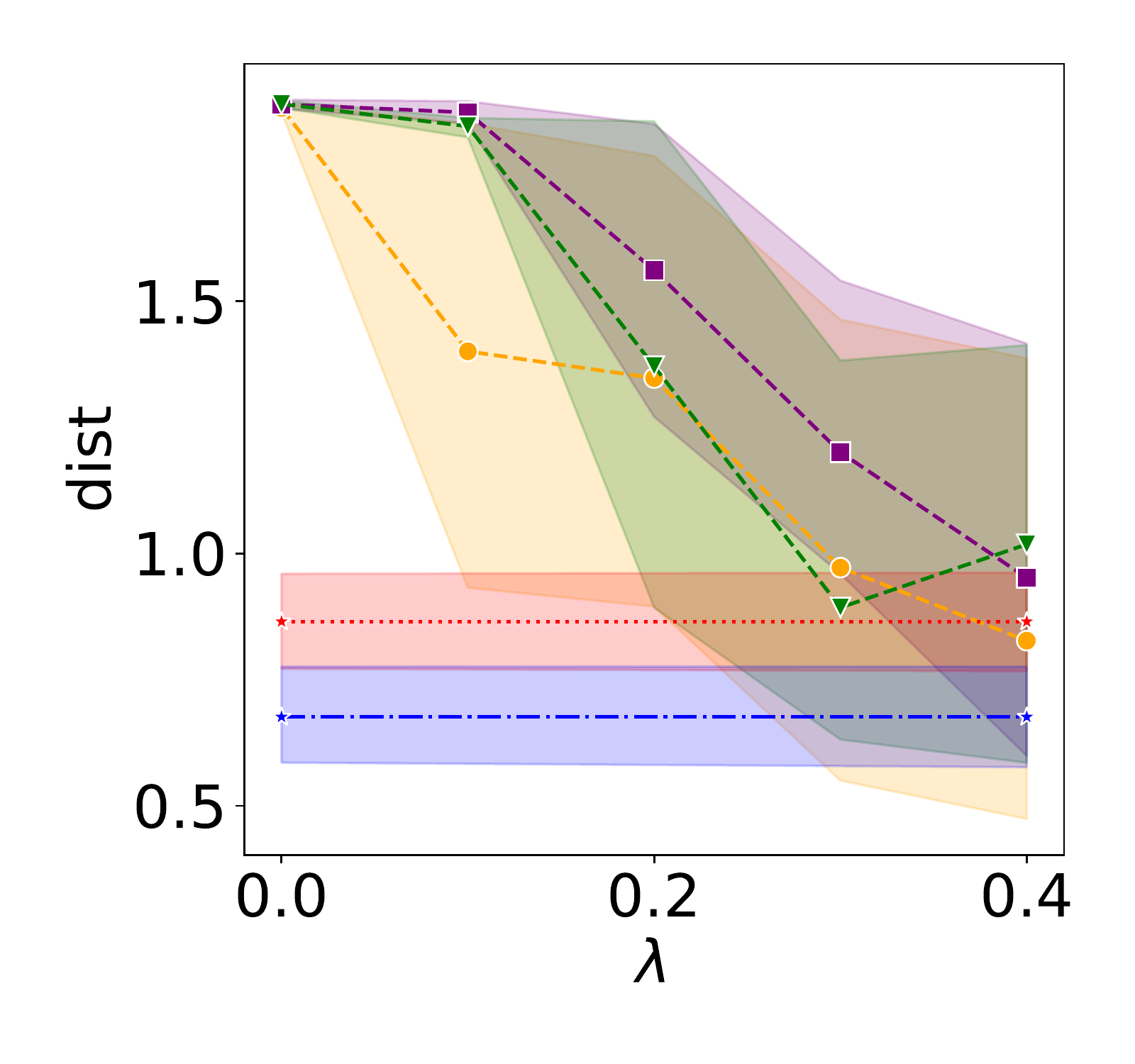}}
		
		\caption{Navigation: Distance, $\lambda$}\label{fig.Navigation_dist_APU}
	\end{subfigure}
    \begin{subfigure}{0.24\textwidth}
    	\centering
		\resizebox{0.98\linewidth}{!}
		{\includegraphics{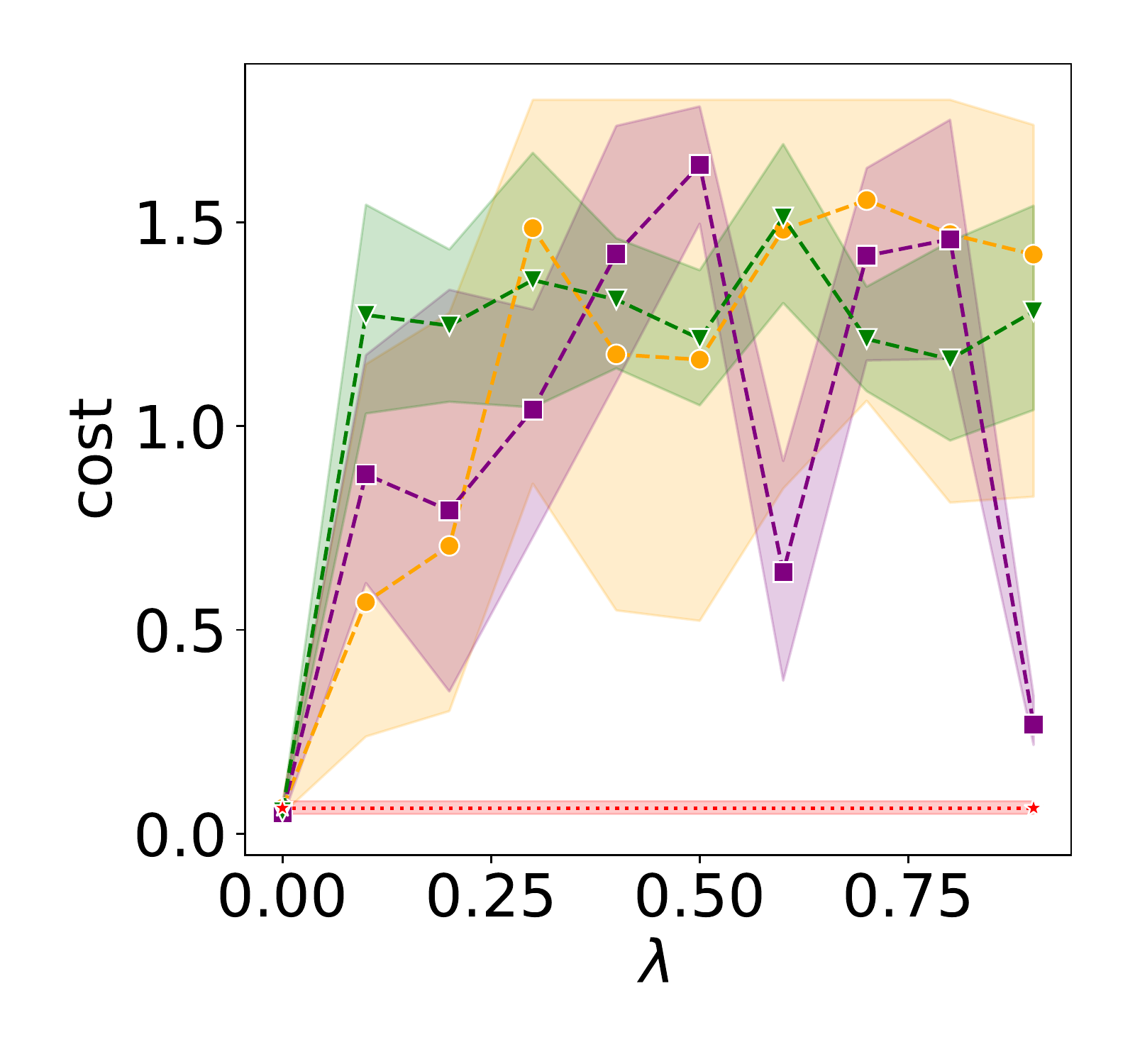}}
		
		\caption{Inventory: Cost, $\lambda$}\label{fig.IM_cost_APU}
	\end{subfigure}
	\begin{subfigure}{.24\textwidth}
    	\centering
		\resizebox{0.98\linewidth}{!}
		{\includegraphics{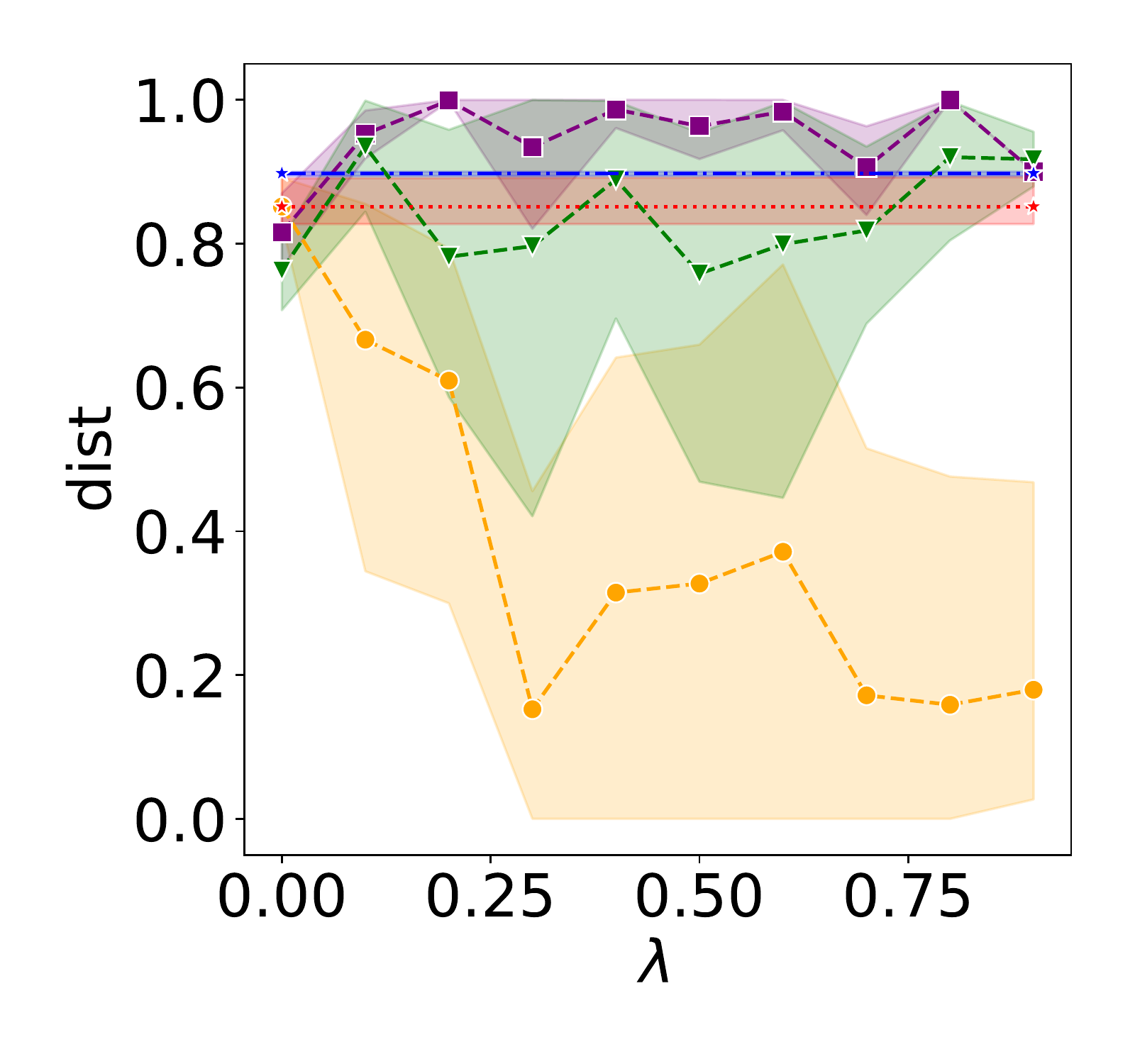}}
		
		\caption{Inventory: Distance, $\lambda$}\label{fig.IM_dist_APU}
	\end{subfigure}
		\caption{The figure show the same set of results as Fig. \ref{fig.alternating_updates} but for Navigation and Inventory Management.   }\label{fig.additional_APU_results}
\end{figure*}

\textbf{Push Environments.} Next, we provide additional results for the  Push environments. In Fig. \ref{fig.alternating_updates.additinoal_results.naive_confidence} we show a more complete version of Fig. \ref{fig.alternating_updates} from the main text, that includes 95\% confidence intervals for baselines whose behavior does not change with $\lambda$. Fig. \ref{fig.push.objective} compares APU, RL and SAPU w.r.t. the objective in \eqref{prob.instance_2}: APU generally finds better solutions.

\begin{figure*}[ht]
    \centering
		\begin{subfigure}{0.96\textwidth}
		\resizebox{0.5\linewidth}{!}{\includegraphics{figures/alternating_policy_updates/1Dlegend.png}}
		\label{fig.Push1D_legend.additional.naive}
		\resizebox{0.5\linewidth}{!}{\includegraphics{figures/alternating_policy_updates/2Dlegend.png}}
		\label{fig.Push2D_legend.additional.naive}
	\end{subfigure}
	\centering
	\begin{subfigure}{0.24\textwidth}
    	\centering
		\resizebox{0.98\linewidth}{!}
		{\includegraphics{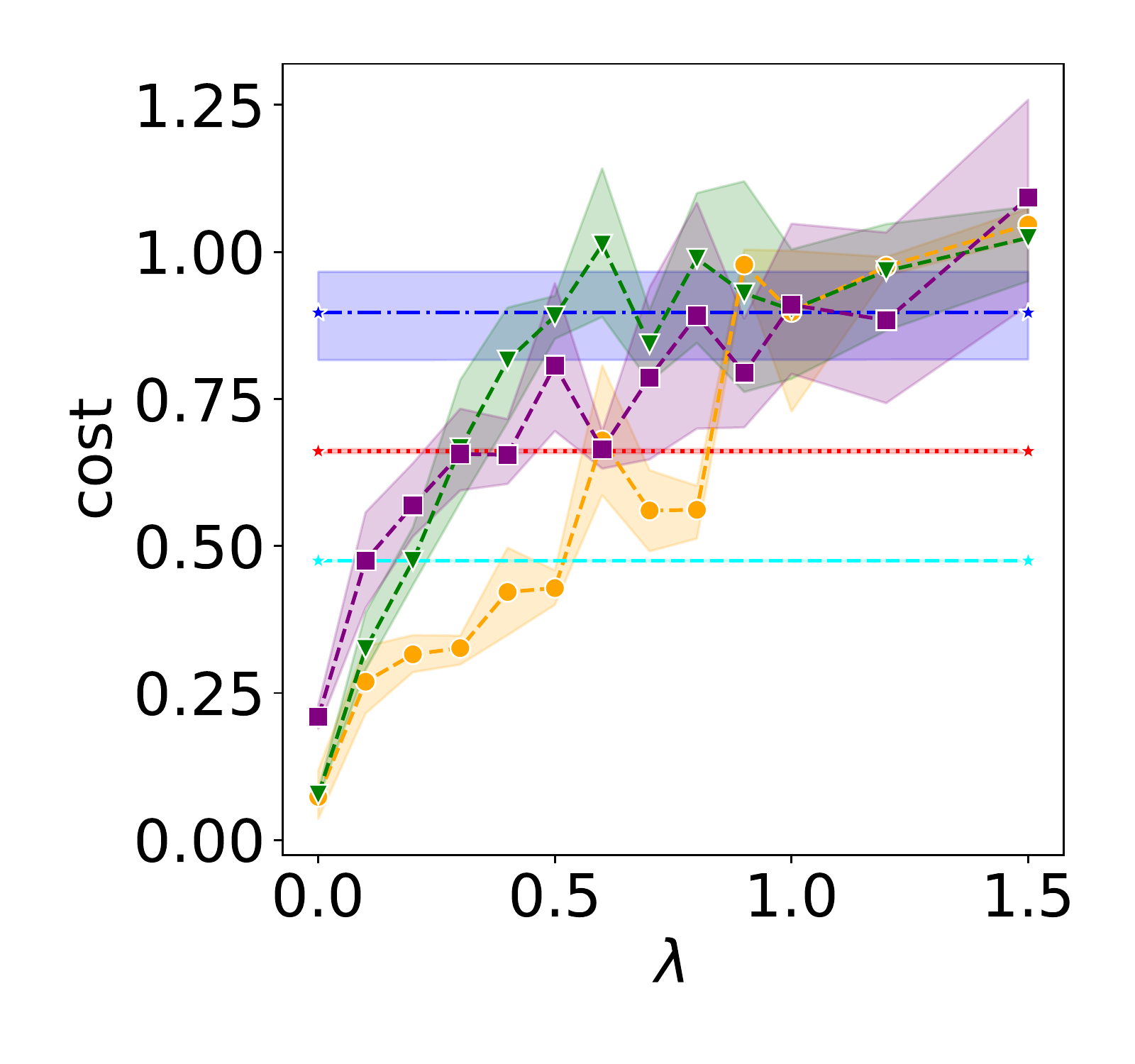}}
		
		\caption{1D Push: Cost, $\lambda$}\label{fig.dist_vs_lambda_g.additional.naive}
	\end{subfigure}
	\begin{subfigure}{.24\textwidth}
    	\centering
		\resizebox{0.98\linewidth}{!}
		{\includegraphics{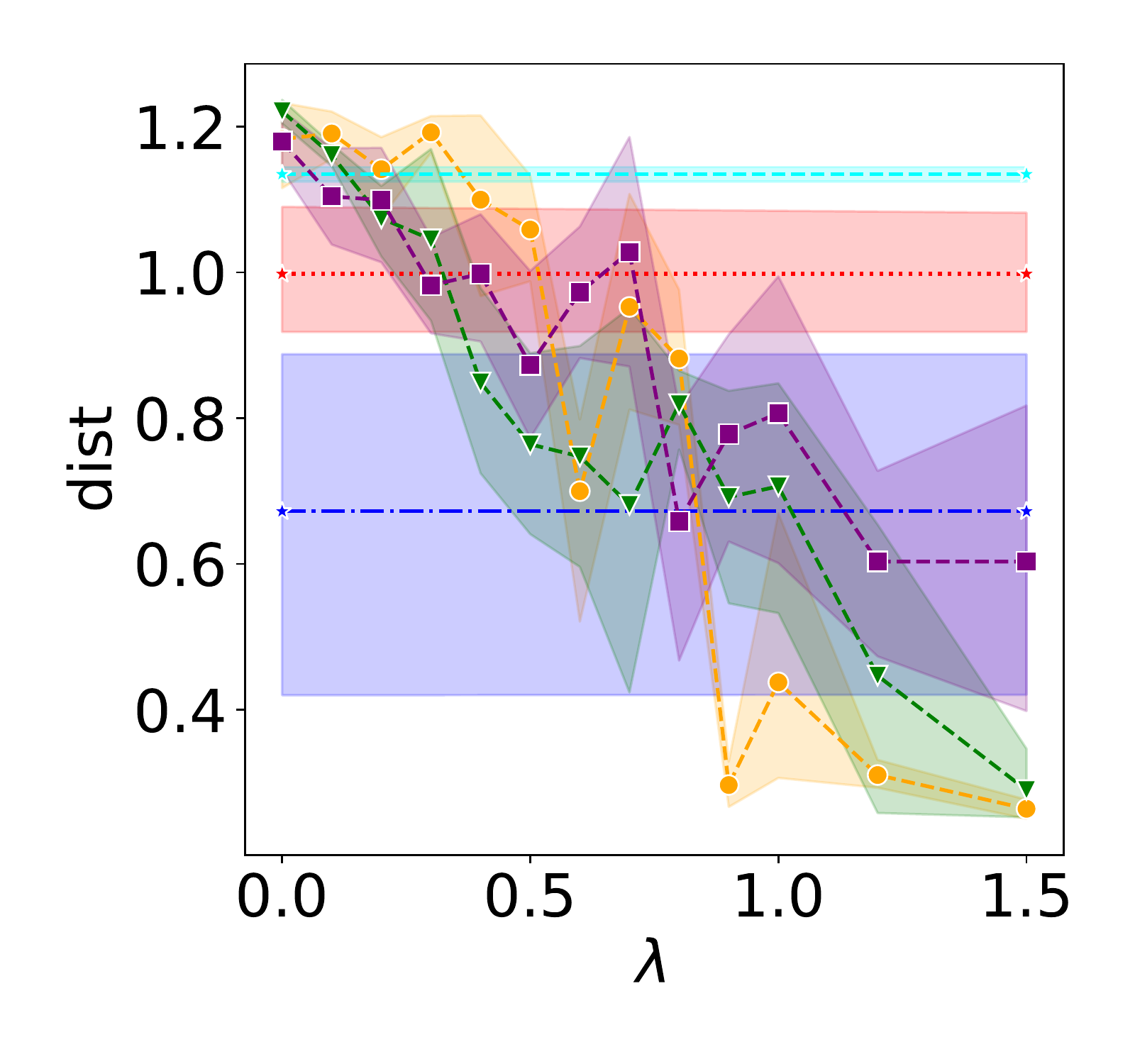}}
		
		\caption{1D Push: Distance, $\lambda$}\label{fig.success_vs_inf_g.additional.naive}
	\end{subfigure}
    \begin{subfigure}{0.24\textwidth}
    	\centering
		\resizebox{0.98\linewidth}{!}
		{\includegraphics{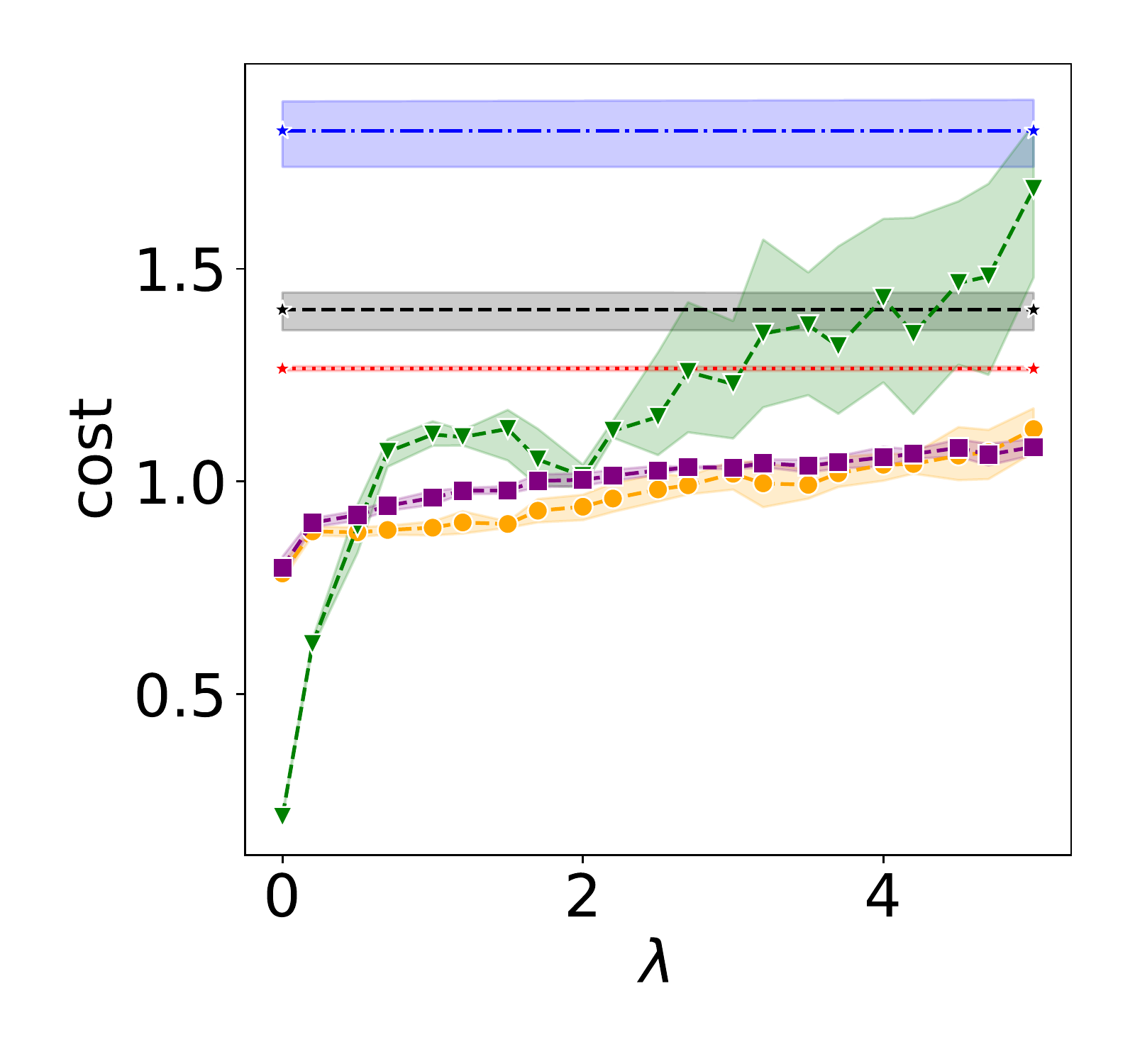}}
		
		\caption{2D Push: Cost, $\lambda$}\label{fig.cost_vs_lambda_g.additional.naive}
	\end{subfigure}
	\begin{subfigure}{.24\textwidth}
    	\centering
		\resizebox{0.98\linewidth}{!}
		{\includegraphics{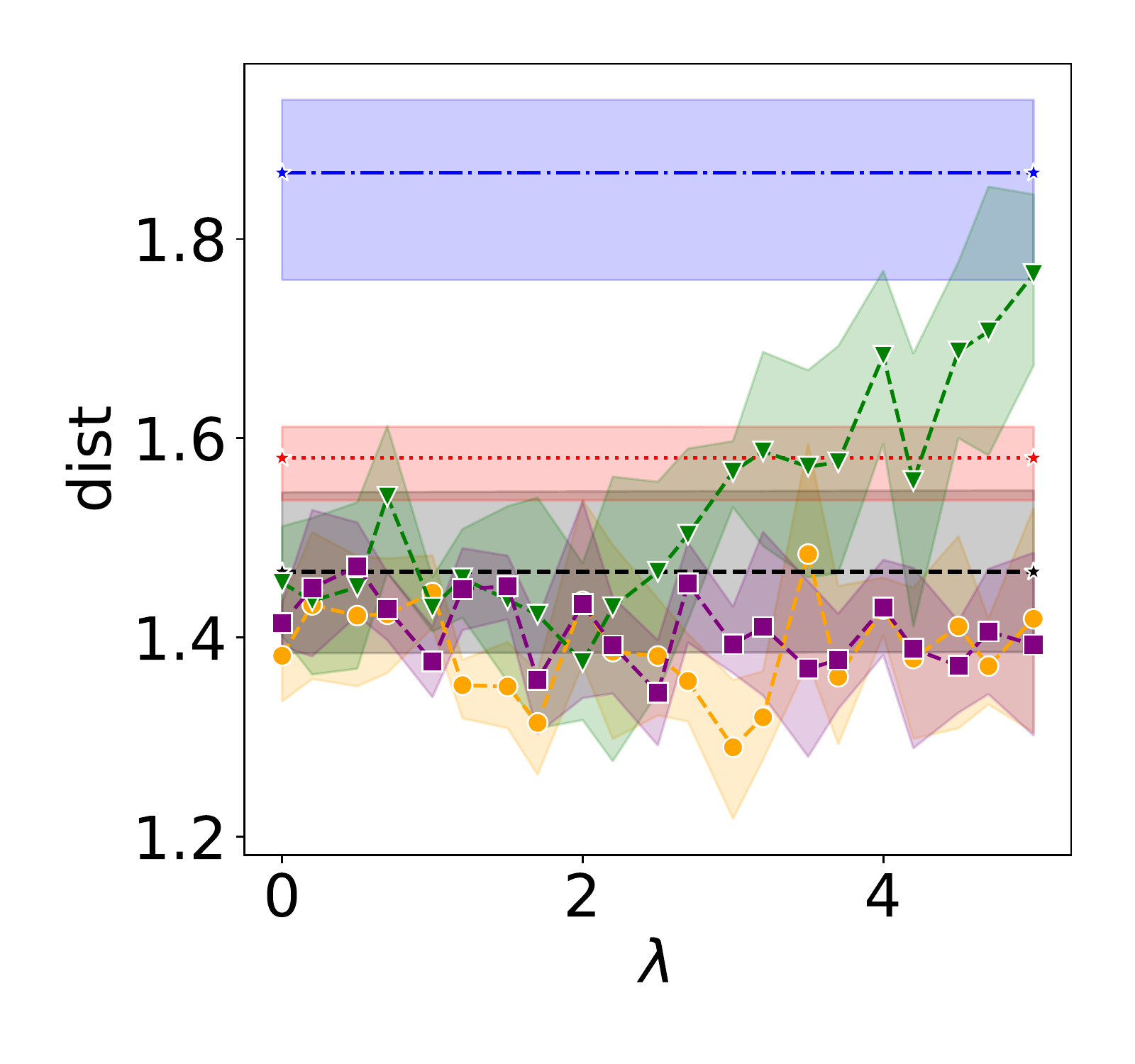}}
		
		\caption{2D Push: Distance, $\lambda$}\label{fig.lambda_vs_cost_g.additional.naive}
	\end{subfigure}
		\caption{The figure is the same as Fig. \ref{fig.alternating_updates}, but additionally contains 95\% confidence intervals for baselines whose behavior does not change with $\lambda$. }\label{fig.alternating_updates.additinoal_results.naive_confidence}
\end{figure*}

\begin{figure*}[ht]
    \centering
    \begin{subfigure}{0.96\textwidth}
        \centering
		\resizebox{0.5\linewidth}{!}{\includegraphics{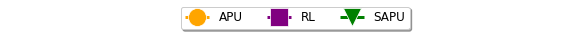}}
		\label{fig.push_objective_legend2}
	\end{subfigure}
    \centering
        \begin{subfigure}{0.24\textwidth}
    	\centering
		\resizebox{0.98\linewidth}{!}
		{\includegraphics{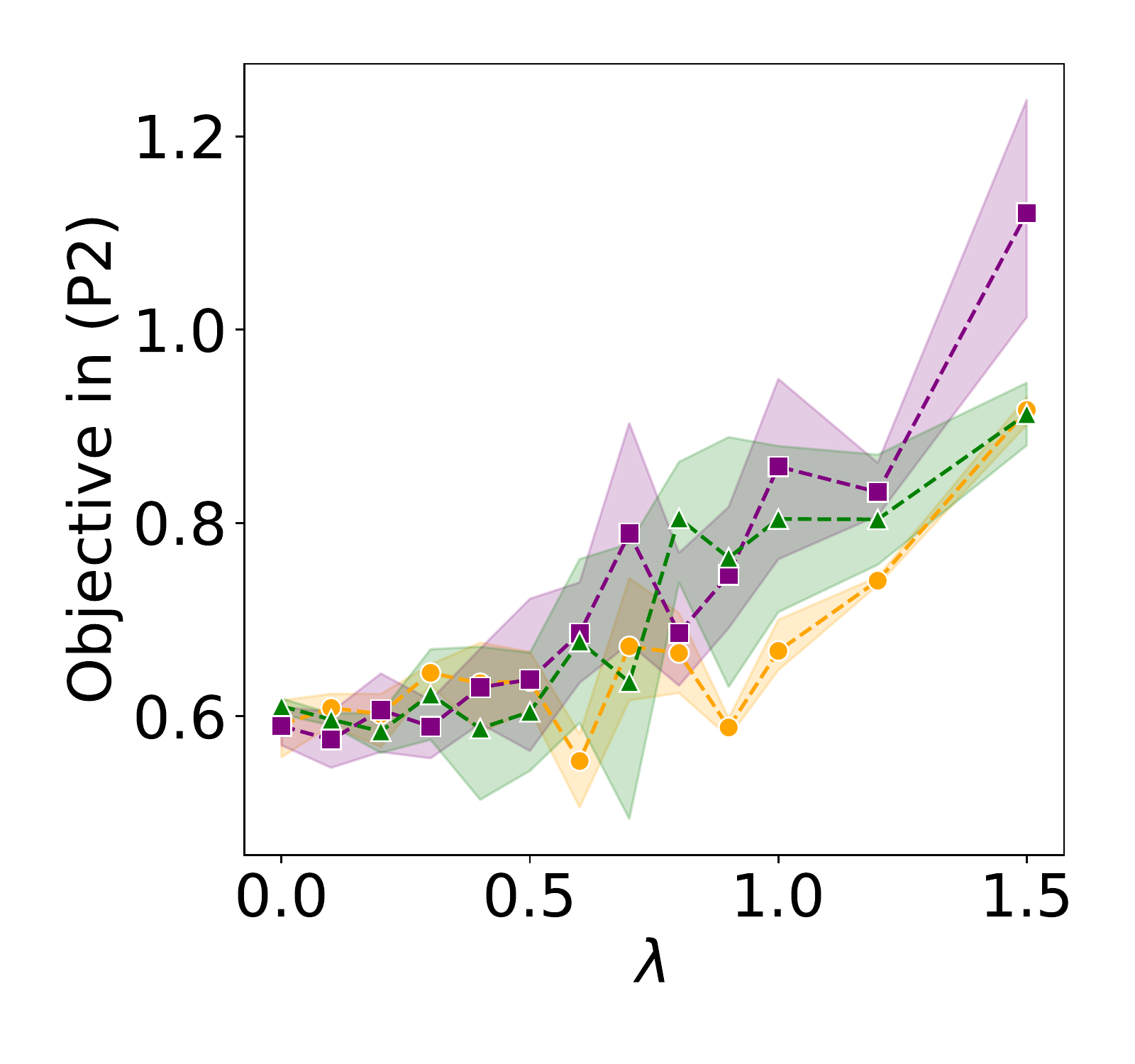}}
		\caption{1D Push}\label{fig.lambda_vs_objective.oned}
	\end{subfigure}
	\begin{subfigure}{.24\textwidth}
    	\centering
		\resizebox{0.98\linewidth}{!}
		{\includegraphics{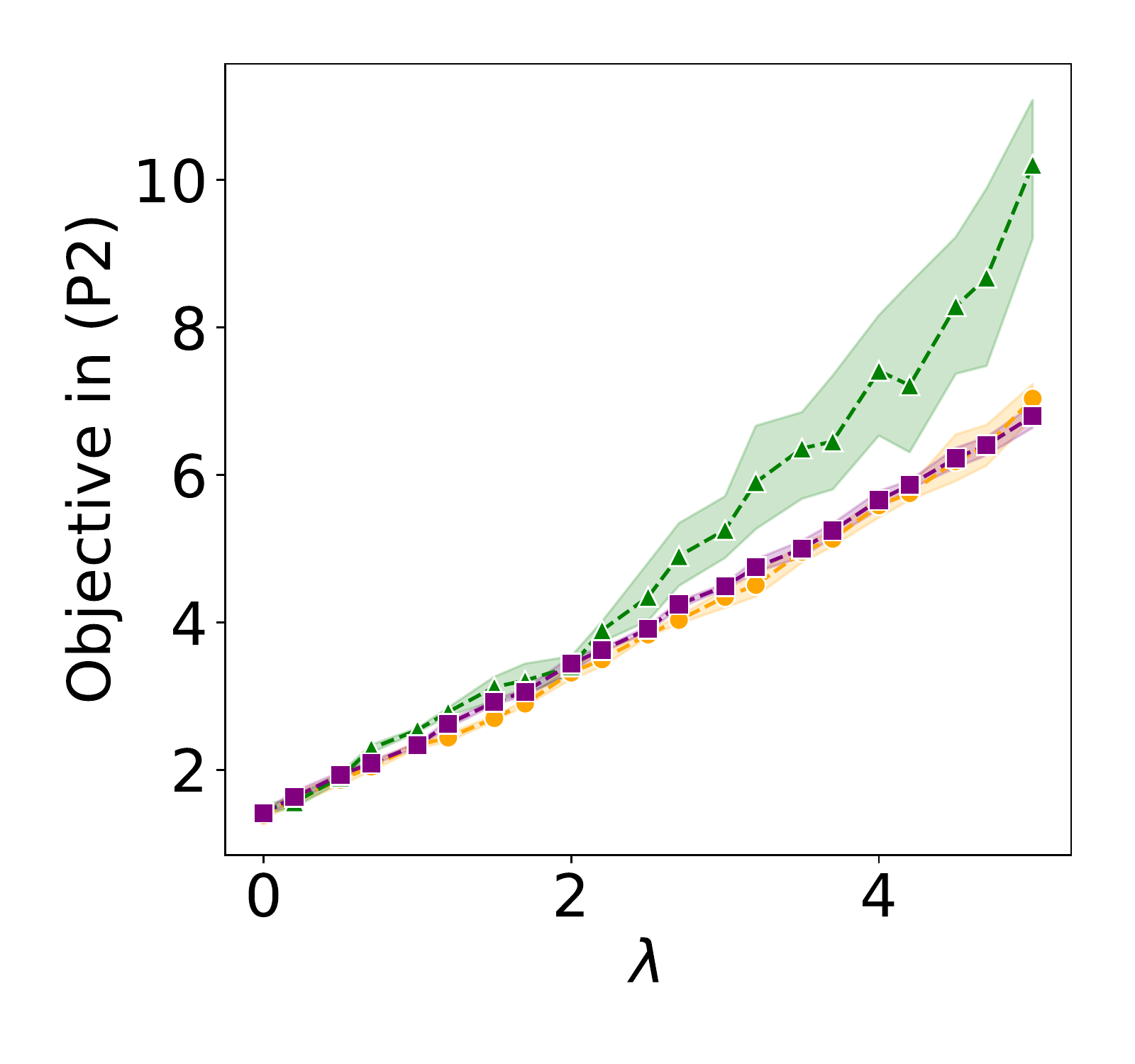}}
		
		\caption{2D Push}\label{fig.lambda_vs_objective_g.twod}
	\end{subfigure}
		\caption{The figures show the test-time performance of APU, RL, and SAPU adversary from Section \ref{sec.experiments.apu} in 1D and 2D Push, for different hyperparameters $\lambda$ in terms of the objective in \eqref{prob.instance_2}. The confidence intervals are obtained in the same way as in Fig. \ref{fig.alternating_updates}. In 1D Push, APU outperforms RL for larger values of $\lambda$. In 2D Push, APU outperforms SAPU for larger value of $\lambda$. In other cases, it APU leads to similar performance as RL and SPU. Hence, APU generally finds better solutions to \eqref{prob.instance_2} than its modifications. 
	}
		\label{fig.push.objective}
\end{figure*}

Fig. \ref{fig.2D_push.vary_penalty} shows how the attack performance changes as we vary the weight of the penalty term in 2D Push. As we can see from the figure, they are similar to the influence results from Section \ref{sec.experiments.CPS}. Namely, the cost of the attack decreases with the increase of penalty weight, since the attacker has more influence and hence the attack is more successfull.

\begin{figure*}[ht]
    \centering
    \begin{subfigure}{0.96\textwidth}
        \centering
		\resizebox{0.5\linewidth}{!}{\includegraphics{figures/alternating_policy_updates/2Dlegend.png}}
		\label{fig.Push2D_legend2.additional.naive}
	\end{subfigure}
    \centering
    \begin{subfigure}{0.24\textwidth}
    	\centering
		\resizebox{0.98\linewidth}{!}
		{\includegraphics{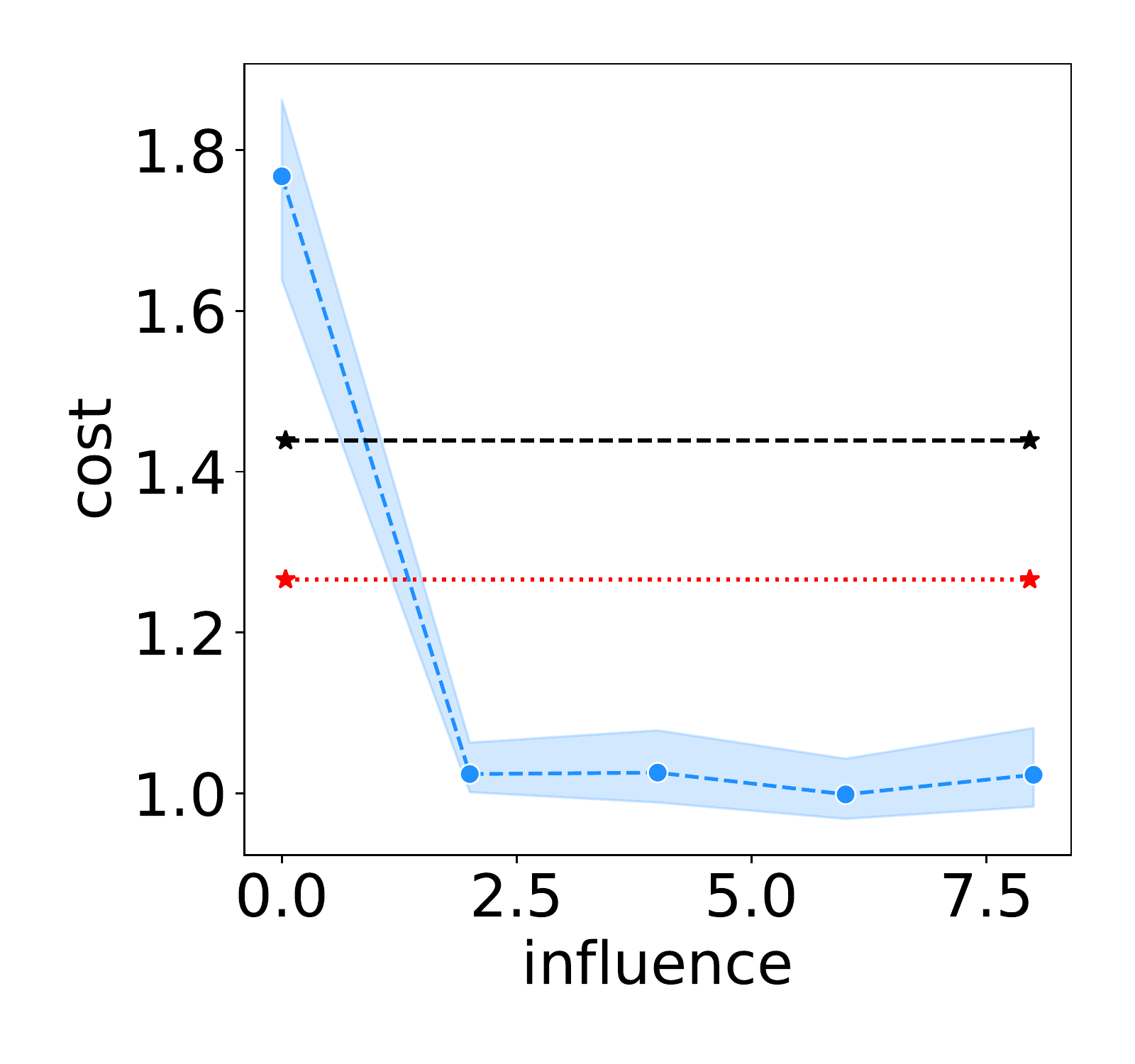}}
		\caption{2D Push: Cost, $w_{p}$}\label{fig.cost_vs_lambda_g.additional.naive}
	\end{subfigure}
	\begin{subfigure}{.24\textwidth}
    	\centering
		\resizebox{0.98\linewidth}{!}
		{\includegraphics{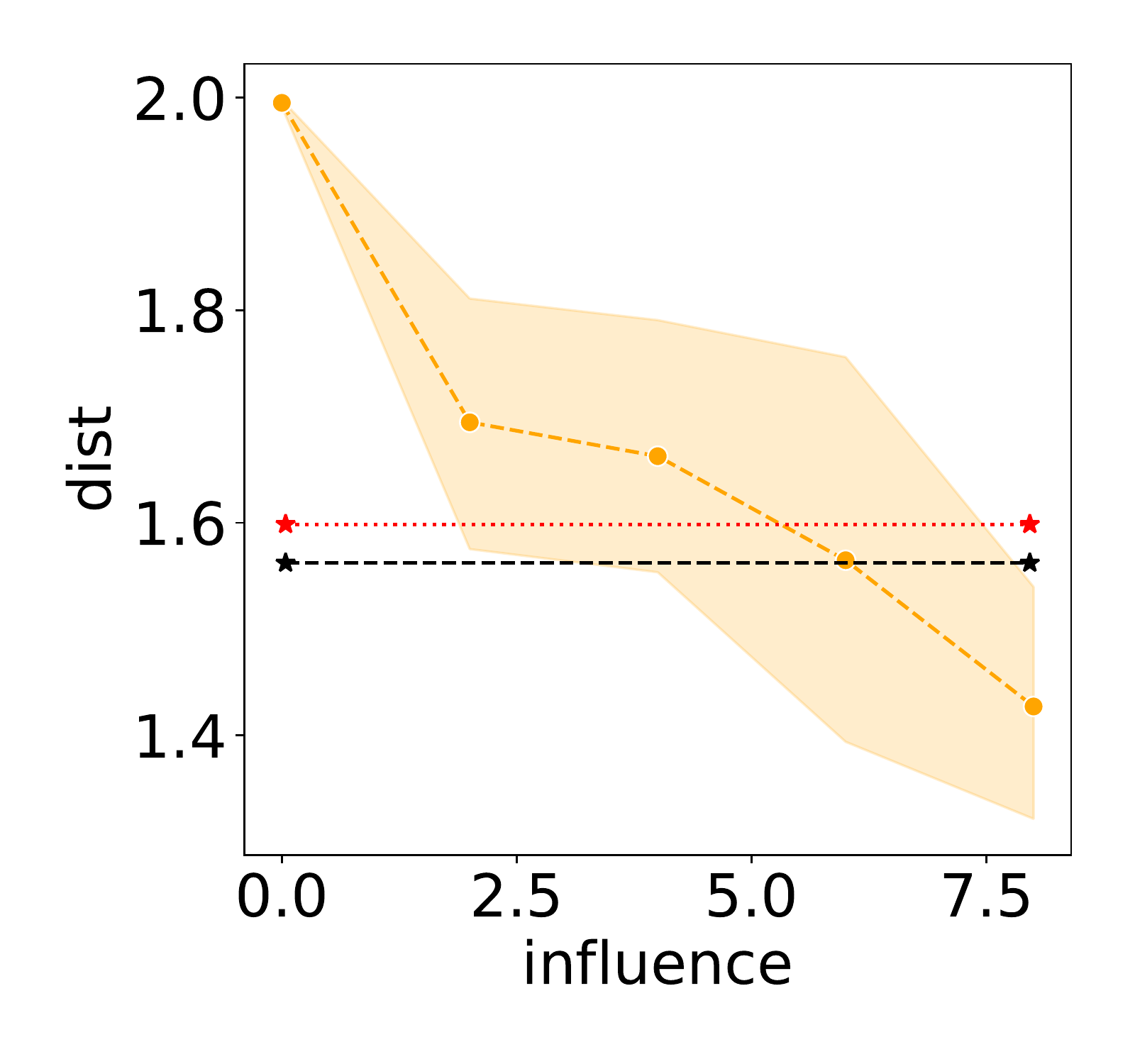}}
		
		\caption{2D Push: Distance, $w_{p}$}\label{fig.lambda_vs_cost_g.additional.naive}
	\end{subfigure}
		\caption{The figure shows how the increase in the weight on the penalty term $w_{p}$ in the victim's reward function $r^t = -(\mathbf x_{l}^t-\mathbf x_{g}^t)^2 - w_{p} \cdot \ind{(\mathbf x_{l}^t-\mathbf x_{a}^t)^2 \le 0.5}$ affects the attack performance of Alternating Policy Updates (APU) in 2D Push. Performance metrics, $\cost$ and $\dist$, are defined in Fig. \ref{fig.alternating_updates}, and we apply the same statistical analysis. We consider a subset of baselines from Fig. \ref{fig.alternating_updates} and we set parameter $\lambda$ from \eqref{prob.instance_2} to $1$. As we can see, APU outperforms baselines for most of the values of $w_p$, with increasing performance (decreasing $\cost$ and $\dist$) as $w_p$ increases.    
		}\label{fig.2D_push.vary_penalty}
\end{figure*}

Finally, we test whether solving \eqref{prob.instance_2} under non-deterministic target policies indeed forces the victim to adopt the target behavior (i.e., behavior under the target policy). Namely, when the target policy is non-deterministic, even if the victim's policy is ``close'' to the target policy in each state, its long-term behavior can be quite different from the target behavior since the errors accumulate over time. We consider Push 2D since this is the only environment with a non-deterministic target policy. The target policy specifies that the victim should go to locations that are approximately $3$ units away from the goal. Hence, we can measure how well an adversary forces the victim to follow this behavior by measuring its average distance to the goal. Fig. \ref{fig.2D_push.distance_to_goal} show the results for different adversaries from Section \ref{sec.experiments.apu}. As we can see, APU outperforms baselines for larger values of $\lambda$, obtaining almost the same average distance to the goal as the target policy.

\begin{figure*}[ht]
    \centering
    \begin{subfigure}{0.96\textwidth}
        \centering
		\resizebox{0.5\linewidth}{!}{\includegraphics{figures/alternating_policy_updates/appendix/DistanceGoalLegend.png}}
		\label{fig.Push2D_legend2.additional.naive}
	\end{subfigure}
    \centering
    \begin{subfigure}{0.24\textwidth}
    	\centering
		\resizebox{0.98\linewidth}{!}
		{\includegraphics{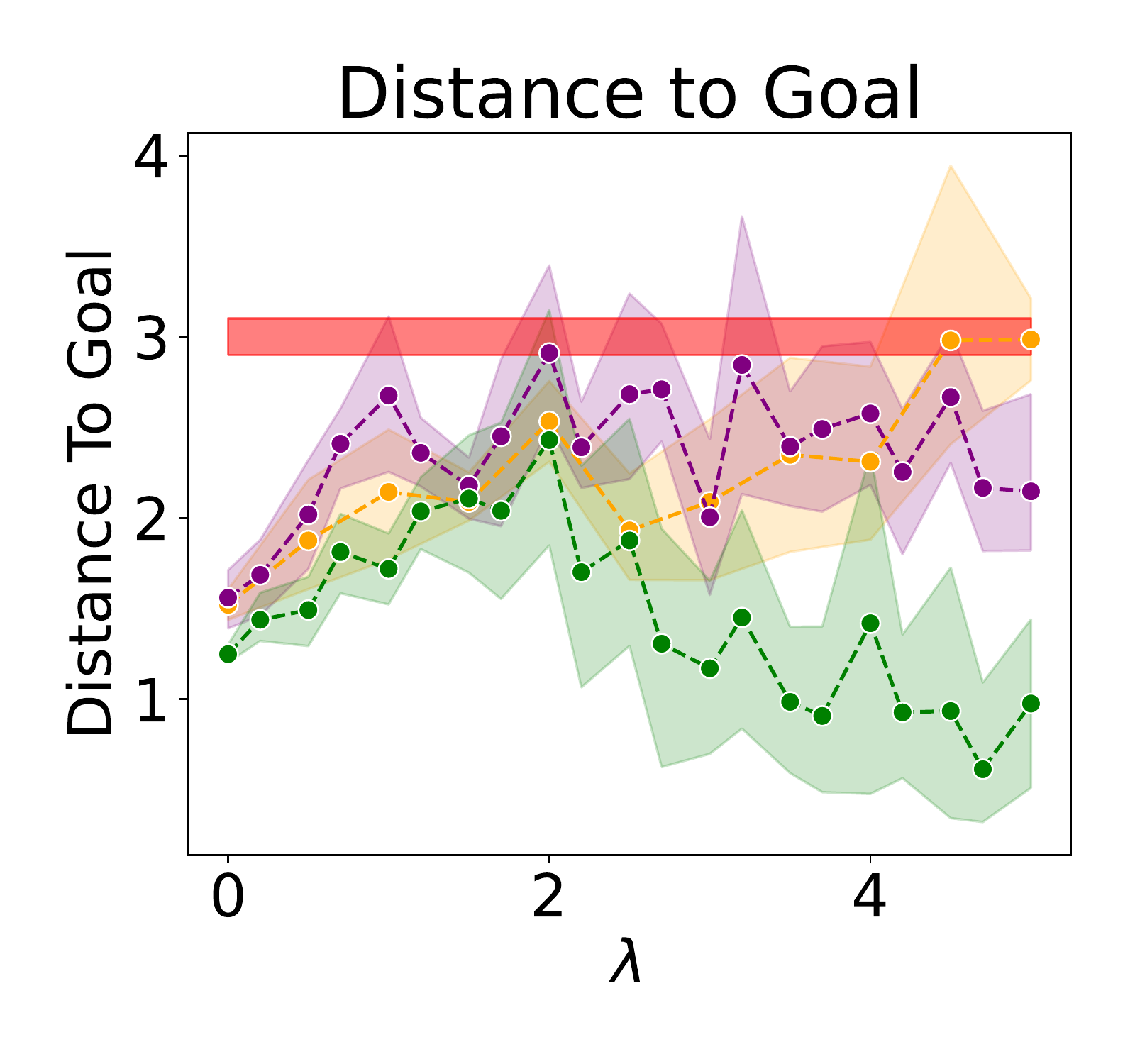}}
		
	\end{subfigure}
		\caption{The figure shows the test-time performance of APU, RA, and SAPU adversary from Section \ref{sec.experiments.apu} in 2D Push, for different hyperparameters $\lambda$ in terms of the average $L_2$ distance of the victim to goal. The red area shows the average distance to the goal of the target policy. To obtain confidence intervals, we follow the same approach as in Fig. \ref{fig.alternating_updates}. For each $\lambda$, we train $5$ adversarial policies (using $5$ different random seeds). For each adversarial policy, we train $5$ victim policies (using $5$ different random seeds) against this adversarial policy. The results show the mean and 95\% confidence intervals of the  obtained data points.
	}
		\label{fig.2D_push.distance_to_goal}
\end{figure*}

\subsection{Additional Implementation Details}

In this section, we provide additional implementation details, including those important for reproducibility of our results. We report this information separately for each of our algorithmic approaches.  

{\bf Conservative Policy Search (CPS).} As mentioned in the main text, to solve \eqref{prob.instance_1.rel}, we use CVXPY solver \cite{diamond2016cvxpy,agrawal2018rewriting}. The experiments that test CPS do not have a source of randomness, so we do not report confidence intervals for these results. We run the experiments on a Dell XPS-13 personal computer with 16 Gigabytes of memory and a 1.3 GHz Intel Core i7 processor. A single iteration of CPS takes on average about $0.04s \pm 0.13ms$ 
on Navigation (which uses Algorithm \ref{alg.conserv_pol_iter_main_text} for ergodic environments) and about $0.44s \pm 1.79ms$ on Inventory. 
Reported numbers are average of 10 runs (each having 200 iterations); standard error is shown with $\pm$. Similar running times are obtained for CPS-based baselines (i.e., Constraints Only PS (COPS) and Unconservative
PS (UPS)). A single run of CPS consists of multiple iterations, whose number we control. For Navigation and Inventory, we use 200 iterations. In each iteration we run a convex program, i.e., \eqref{prob.instance_1.rel} (or \eqref{prob.instance_1.rel_2}, see Section \nameref{sec.app.algorihtms}), as explained in the main text.

~\\
{\bf Alternating Policy Updates (APU).} As mentioned in the main text, APU applies Proximal Policy Optimization (PPO)~\cite{schulman2017proximal} for updating policies; Algorithm \ref{alg.alter_updates_main_text} specifies the details. We run the experiments on a Dell PowerEdge R730 with a M40 Nvidia Tesla GPU, a Intel Xeon E5-2667 v4 CPU and 512GB of memory. To train a single adversary, it takes about  
$38 \pm 0.67$ 
minutes for 1D Push and  
$35 \pm 0.52$ minutes 
for 2D Push. We obtain similar results for Navigation ($36.72\pm0.76$) and Inventory Management ($31.69\pm0.6$). To train a single victim agent for a fixed adversary, it takes about 
$8.42 \pm 0.13$ 
minutes for all environments. Similar running times are obtained for APU-based baselines (Random Learner (RL), Symmetric APU (SAPU), and Distance-only APU (DAPU)). All reported runtimes are the average over 5 runs; standard error is shown with $\pm$. We train APU and APU-based baselines for $20$ epochs in 2D Push and $50$ epochs in 1D Push. The number of training step per epoch is specified in Table \ref{tab.training_parameters_apu}. The experiments that test APU do have a source of randomness, so for each algorithm we estimate their mean performance and the corresponding 95\% confidence intervals.


\section{Additional Background Details}\label{sec.app.background}

Apart from the quantities introduced in the main text, we also consider standard value function, $V$, and state-action value function, $Q$. In our setting these are defined as:
\begin{align*}
    V^{\pola, \poll}(s) &= \expct{\sum_{t = 1}^{\infty} \gamma^{t-1} \cdot R_{\cL}(s_t, a_{\cA, t}, a_{\cL, t}) | s_1 = s, \pola, \poll}\\
    Q^{\pola, \poll}(s, a_\cA, a_\cL) &= \expct{\sum_{t = 1}^{\infty} \gamma^{t-1} \cdot R_{\cL}(s_t, a_{\cA, t}, a_{\cL, t})|s_1 = s, a_{\cA, 1} = a_\cA, a_{\cL, 1} = a_\cL, \pola, \poll},
\end{align*}
where the expectations are taken over trajectories  $\tau = (s_1 := s, a_{\cA, 1}, a_{\cL, 1}, s_2, a_{\cA, 2}, a_{\cL, 2}, ...)$ for $V^{\pola, \poll}$ and $\tau' = (s_1 := s, a_{\cA, 1} := a_{\cA}, a_{\cL, 1} := a_{\cL}, s_2, a_{\cA, 2}, a_{\cL, 2}, ...)$ for $Q^{\pola, \poll}$. Trajectory $\tau$ is obtained by executing policy policies $\pola$ and $\poll$ starting in state $s$. Trajectory $\tau'$ is obtained by executing policies $\pola$ and $\poll$ starting in state $s$, in which we take actions $a_\cA$ and $a_\cL$. 
$V^{\pola, \poll}$ and $Q^{\pola, \poll}$ satisfy the following Bellman's equations:
\begin{align*}
Q^{\pola, \poll}(s, a_{\cA}, a_\cL) &= R_2(s, a_{\cA}, a_\cL) + \gamma \cdot \sum_{s'} P(s, a_{\cA}, a_\cL, s') \cdot V^{\pola, \poll}(s')\\
V^{\pola, \poll}(s) &= \sum_{ a_{\cA}, a_\cL} \pola(s, a_\cA) \cdot \poll(s, a_\cL) \cdot Q^{\pola, \poll}(s, a_{\cA}, a_\cL).
\end{align*}
We also introduce the state-action value function of $\poll$ in two-agent MDP $\cM$ with a fixed policy $\pola$:
\begin{align*}
    Q^{\poll}_{\pola}(s, a_\cL) &= \expct{\sum_{t = 1}^{\infty} \gamma^{t-1} \cdot R_{\cL}(s_t, a_{\cA, t}, a_{\cL, t}) | s_1 = s, a_{\cL, 1} = a_\cL, \pola, \poll},
\end{align*}
where the expectation is taken over trajectory $\tau'' = (s_1 := s, a_{\cA, 1}, a_{\cL, 1}:= a_{\cL}, s_2, a_{\cA, 2}, a_{\cL, 2}, ...)$. I.e., trajectory $\tau''$ is obtained by executing policies $\pola$ and $\poll$ starting in state $s$, in which we take action $a_\cL$, but follow policy $\pola$ to obtain $a_{\cA, 1}$. Note that for a fixed policy $\pola$, two-agent MDP $\cM = (\{ \cA, \cL \}, S, A, P, R_{\cL}, \gamma, \sigma)$ degenerates to a single agent MDP $\cM_{\cL} = (S, A_{\cL}, P^{\pola}, R_{\cL}^{\pola}, \gamma, \sigma)$, where $P^{\pola}(s, a_{\cL}, s') = P(s, \pola, a_{\cL}, s') = \sum_{a_{\cA}} \pola(s, a_{\cA}) \cdot P(s, a_{\cA}, a_{\cL}, s')$ and $R_{\cL}^{\pola}(s, a_{\cL}) = R_{\cL}(s, \pola, a_{\cL}) = \sum_{a_{\cA}} \pola(s, a_{\cA}) \cdot R_{\cL}(s, a_{\cA}, a_{\cL})$. 

As mentioned in Remark \ref{rm.notation}, we, in part, utilize vector notation when convenient, in particular, for the tabular setting. In this case, $R_{\cL}$ can be thought of as a vector with $|S|\cdot |A|$ components, $R_{\cL}^{\pola}$ as a vector with $|S|\cdot |A_{\cL}|$ components,  $V^{\pola, \poll}$ as a vector with $|S|$ components, $Q^{\pola, \poll}$ as a vector with $|S| \cdot |A|$ components, $Q^{\poll}_{\pola}$ as a vector with $|S| \cdot |A_{\cL}|$ components. We can think of policy $\pola$ as as a matrix with $|A_{\cA}| \times |S|$ entries, transition model $P$ as a matrix with $|S| \times |S|\cdot|A|$ entries, and transition model $P^{\pola}$ as a matrix with $|S| \times |S|\cdot|A_\cL|$ entries. We will treat $\overline \chi_{\eps}$ and $\chi_2^*$ as vectors with $|S|\cdot |A_{\cL}|$ components, except in Proposition \ref{prop.upper_bound_1}, where we treat them as matrices with $|A_{\cL}| \times |S|$ entries. Note that $\norm{x}_p$ (resp. $\norm{x}_{p, q}$) is the usual $\ell_p$ (resp. $\ell_{p, q}$) norm when treating $x$ as a vector (resp. as a matrix). I.e., $\norm{x}_{p} = \left (\sum_{i = 1} x_i^{p} \right )^{1/p}$ for vector $x$ and $\norm{x}_{p, q} = \left (\sum_{j} \norm{x_j}_p^{q} \right )^{1/q}$ for matrix $x$, where in the latter case $x_j$ are the columns of $x$. Moreover, $\norm{x/y}_{p,q}$ and $\norm{x/y}_{p}$ denote element-wise division.

Note that in vector notation, the Bellman equation for $ Q^{\poll}_{\pola}$ can be expressed as 
\begin{align*}
 Q^{\poll}_{\pola} &= R_{\cL}^{\pola} + \gamma \cdot (P^{\pola})^T \cdot V^{\pola, \poll}.     
\end{align*}

Furthermore, we can bound the influence of policy $\pola$ on the effective rewards and transitions of agent $\cL$ relative to some other policy $\pola'$ as follows. 

\begin{lemma}\label{lm.policy_reward_transitions}
Consider two-agent MDP $\cM = (\{ \cA, \cL \}, S, A, P, R_{\cL}, \gamma, \sigma)$ and two degenerate single agent MDPs, $\cM_{\cL} = (S, A_{\cL}, P^{\pola}, R_{\cL}^{\pola}, \gamma, \sigma)$ and $\cM_{\cL}' = (S, A_{\cL}, P^{\pola'}, R_{\cL}^{\pola'}, \gamma, \sigma)$. The following inequalities hold
\begin{align*}
    \norm{R_{\cL}^{\pola} - R_{\cL}^{\pola'}}_{\infty} &\le \norm{\pola - \pola'}_{1, \infty} \cdot \norm{R_{\cL}}_{\infty} \\
    \norm{P^{\pola} - P^{\pola'}}_{1, \infty} &\le \norm{\pola - \pola'}_{1, \infty}.
\end{align*}
\end{lemma}
\begin{proof}
The first inequality follows from: 
\begin{align*}
	\norm{R_{\cL}^{\pola} - R_{\cL}^{\pola'}}_{\infty} &= \max_{s, a_2} |R_{\cL}^{\pola}(s, a_{\cL}) - R_{\cL}^{\pola'}(s, a_{\cL})|\\
	&= \max_{s, a_\cL}|\sum_{a_\cA} \pola(s, a_\cA) \cdot R_{\cL}(s, a_\cA, a_\cL) - \sum_{a_\cA} \pola'(s, a_\cA) \cdot R_{\cL}(s, a_\cA, a_\cL)| \\
	&= \max_{s, a_\cL}|\sum_{a_\cA} ( \pola(s, a_\cA) - \pola'(s, a_\cA) ) \cdot R_{\cL}(s, a_\cA, a_\cL)|  \\
	&\underbrace{\le}_{(i)} \max_{s, a_\cL}\sum_{a_\cA} | (\pola(s, a_\cA) - \pola'(s, a_\cA) ) \cdot R_{\cL}(s, a_\cA, a_\cL)| \\
	&= \max_{s, a_\cL}\sum_{a_\cA} |\pola(s, a_\cA) - \pola'(s, a_\cA) | \cdot | R_{\cL}(s, a_\cA, a_\cL)| \\
	&\le \max_{s}\sum_{a_\cA} | \pola(s, a_\cA) - \pola'(s, a_\cA) | \cdot \max_{a_\cL, a_\cA} |R_{\cL}(s, a_\cL, a_\cA)|\\
	&\le \max_{s}\sum_{a_\cA} | \pola(s, a_\cA) - \pola'(s, a_\cA) | \cdot \norm{R_{\cL}}_{\infty}\\
	& = \norm{\pola - \pola'}_{1, \infty} \cdot \norm{R_{\cL}}_{\infty}
\end{align*}
where $(i)$ is due to the triangle inequality. Similarly, we obtain
\begin{align*}
    \norm{P^{\pola} - P^{\pola'}}_{1, \infty} &= \max_{s, a_\cL}\sum_{s'}|P^{\pola}(s, a_\cL, s') - P^{\pola'}(s, a_\cL, s')| \\
	&= \max_{s, a_\cL}\sum_{s'}|\sum_{a_\cA} \pola(s, a_\cA) \cdot P(s, a_\cA, a_\cL, s') - \sum_{a_\cA}\pola'(s, a_\cA) \cdot P(s, a_\cA ,a_\cL, s')| \\
	&\underbrace{\le}_{(i)} \max_{s, a_\cL}\sum_{s'}\sum_{a_\cA} | \pola(s, a_\cA)  - \pola'(s, a_\cA) | \cdot | P(s, a_\cA ,a_\cL, s')| \\
    &= \max_{s, a_\cL}\sum_{a_\cA} | \pola(s, a_\cA)  - \pola'(s, a_\cA) | \cdot \sum_{s'} | P(s, a_\cA ,a_\cL, s')|\\
	&\underbrace{=}_{(ii)} \max_{s}\sum_{a_\cA} | \pola(s, a_\cA)  - \pola'(s, a_\cA) |\\
	&= \norm{\pola - \pola'}_{1, \infty},
\end{align*}
where $(i)$ is follows from the triangle inequality, while $(ii)$ follows from the fact that $\sum_{s'} P(s, a_1, a_2, s') = 1$.
\end{proof}

We further make use of the following well known result (see \cite{even2005experts} and \cite{schulman2015trust}):
\begin{lemma}\label{lm.score_advantage}
Any policies $\pola$, $\poll$, $\pola'$, and $\poll'$ satisfy
\begin{align*}
    \score_2^{\pola', \poll'} -  \score_2^{\pola, \poll} = \sum_{s} \occstate^{\pola', \poll'}(s) \cdot \left ( Q^{\pola, \poll}(s, \pola', \poll') - V^{\pola, \poll}(s) \right ).
\end{align*}
\end{lemma}

\subsection{Neighbor Policies}

As stated in the main text, it is useful to consider the notion of neighbor policies for reducing the number of constraints in the optimization problem \eqref{prob.instance_1}. A neighbor policy $\pi\{s,a\}$ of policy $\pi$ is equal to $\pi$ in all states except in $s$, where it is defined as $\pi\{s,a\}(s, a) = 1.0$. We now prove a couple of results akin to Lemma 1 from \cite{rakhsha2021policy}, but for the setting of interest. These results are important for our algorithmic approach based on policy search.
We start with a lemma that introduces a quantity which is used in Lemma \ref{lm.neighbor_policies}, and whose proof technique is instructive for that lemma.  

\begin{lemma}\label{lm.neighbor_policies.basic}
Consider a policy $\pola$, target policy $\targetpi$,
and $\tilde Q_{\pi_1}^{*}$ defined by the Bellman equations:
\begin{align}\label{eq.belman_equations}
    &\tilde Q_{\pola}^{*}(s, a_{\cA}, a_\cL) = R_2(s, a_{\cA}, a_\cL) + \gamma \cdot \sum_{s'} P(s, a_{\cA}, a_\cL, s') \cdot \tilde V_{\pola}^*(s')\\
    &\tilde V_{\pola}^{*}(s) =\begin{cases}
        \max_{a_{\cL}'} \sum_{a_{\cA}'} \pola(s, a_{\cA}') \cdot  \tilde Q_{\pola}^{*}(s, a_{\cA}', a_\cL') \quad &\text{\em  if  } \quad \occstate^{\pola, \targetpi}(s) = 0, \\
        \tilde Q_{\pola}^{*}(s, \pola, \targetpi) \quad &\text{\em  if  } \quad \occstate^{\pola, \targetpi}(s) > 0.\notag
    \end{cases}
\end{align}
Then, $\pola$ satisfies the constraints of the optimization problem \eqref{prob.instance_1} if 
\begin{align}\label{eq.neighbor_constrain}
     \score_2^{\pola, \targetpi} \ge \score_2^{\pola, \targetpiproxy \{s,a_\cL \}} +  \frac{\eps}{q_{\mu}^{\pola}} \quad \forall s \text{ \em  s.t. } \occstate^{\pola, \targetpi}(s) > 0 \text{ \em and } \forall a_\cL \text{ \em s.t. } \targetpi(s, a_\cL) = 0,
\end{align}
where $\targetpiproxy \in \Pi_2^{\dagger}(\pola)$ that for $\occstate^{\pola, \targetpi}(s) = 0$ takes action $a_\cL$, i.e., $\targetpiproxy(s, a_\cL) > 0$, only if $a_\cL \in \argmax_{a_\cL'} \sum_{a_{\cA}} \pola(s, a_{\cA}') \cdot \tilde Q_{\pi_1}^{*}(s, a_{\cA}, a_\cL')$, and $q_{\mu}^{\pi_1} = \min_{\pi_2 \in \Polld} \sum_{s, a_2|\occstate^{\pola, \targetpi}(s) > 0 \land \poll(s, .) \ne \targetpiproxy(s, .) \land \poll(s, a_2) = 1.0} \frac{ \occstate^{\pola, \poll}(s)}{\occstate^{\pola, \targetpiproxy \{s, a_2\}}(s)} $. 
\end{lemma}
\begin{proof}
Consider a deterministic policy $\poll^*$ s.t. $\poll^*(s, .) \ne \targetpi(s, .)$ for at least one state $s$ that satisfies $\occstate^{\pi_1, \targetpi}(s) > 0$, and suppose that the conditions of the lemma hold. Denote $a_s$ action for which $\poll^*(s, a_s) = 1.0$. We have that:
\begin{align*}
     \score_2^{\pola, \poll^*} - \score_2^{\pola, \targetpi} &= \score_2^{\pola, \poll^*} - \score_2^{\pola, \targetpiproxy}\\ &\underbrace{=}_{(i)} \sum_{s} \occstate^{\pola, \poll^*}(s) \cdot \left ( Q^{\pola, \targetpiproxy}(s, \pola, \poll^*) - V^{\pola, \targetpiproxy}(s) \right ) \\
     &\underbrace{\le}_{(ii)} \sum_{s|\occstate^{\pola, \targetpi}(s) > 0}  \occstate^{\pola, \poll^*}(s) \cdot \left ( \tilde Q_{\pola}^{*}(s, \pola, \poll^*) - \tilde V_{\pola}^{*}(s) \right)\\
     &\underbrace{=}_{(i)}\sum_{s|\occstate^{\pola, \targetpi}(s) > 0 \land \poll^*(s, .) \ne \targetpiproxy(s, . )} \occstate^{\pola, \poll^*}(s) \cdot \frac{\score_2^{\pola, \targetpiproxy \{s, a_s\}} -  \score_2^{\pola, \targetpiproxy}}{\occstate^{\pola, \targetpiproxy \{s, a_s\}}(s)}\\
     &\underbrace{\le}_{(iii)} - \frac{\eps}{q_{\mu}^{\pola}} \cdot \sum_{s|\occstate^{\pola, \targetpi}(s) > 0 \land \poll^*(s, .) \ne \targetpiproxy(s, .)} \frac{ \occstate^{\pola, \poll^*}(s)}{\occstate^{\pola, \targetpiproxy \{s, a_s\}}(s)}  \underbrace{\le}_{(iv)} -\eps
\end{align*}
where we applied Lemma \ref{lm.score_advantage} over the neighboring policies (for $(i)$),
used the fact that $\targetpiproxy$ satisfies \eqref{eq.belman_equations} (for $(ii)$), used the assumption that the constraints in Lemma \ref{lm.neighbor_policies.basic} are satisfied (for $(iii)$), and the fact that the last summation is bounded below by $q_{\mu}^{\pola}$. 
\end{proof}

For establishing our characterization results, we rely on a different version of this result, stated in the following lemma.

\begin{lemma}\label{lm.neighbor_policies}
Consider a policy $\pola$,
target policy $\targetpi$, its neighbor $\targetpi\{s,a_{\cL}\}$,
and $\bar Q_{\pi_1}^{*\{s,a_{\cL}\}}$ defined by the Bellman equations:
\begin{align}\label{eq.belman_equations.ntargetproxy}
    &\bar Q_{\pola}^{*\{s,a_{\cL}\}}(s', a_{\cA}, a_\cL') = R_2(s', a_{\cA}, a_\cL') + \gamma \cdot \sum_{s''} P(s', a_{\cA}, a_\cL', s'') \cdot \bar V_{\pi_1}^*(s'')\\
    &\bar V_{\pola}^{*\{s,a_{\cL}\}}(s') =\begin{cases}
        \max_{a_{\cL}'} \sum_{a_{\cA}'} \pola(s', a_{\cA}') \cdot \bar Q_{\pi_1}^{*\{s,a_{\cL}\}}(s', a_{\cA}', a_\cL') \quad &\text{\em  if  } \quad \occstate^{\pola, \targetpi}(s') = 0, \\
        \bar Q_{\pola}^{*\{s,a_{\cL}\}}(s', \pola, \targetpi\{s,a_{\cL}\}) \quad &\text{\em  if  } \quad \occstate^{\pola, \targetpi}(s') > 0.\notag
    \end{cases}
\end{align}
Then, $\pola$ satisfies the constraints of the optimization problem \eqref{prob.instance_1} if and only if
\begin{align*}
     \score_2^{\pola, \targetpi} \ge \score_2^{\pola, \ntargetpiproxy \{s,a_\cL \}} + \eps \quad \forall s \text{ \em  s.t. } \occstate^{\pola, \targetpi}(s) > 0 \text{ \em and } \forall a_\cL \text{ \em s.t. } \targetpi(s, a_\cL) = 0,
\end{align*}
where $\ntargetpiproxy\{s,a_\cL\} \in \Pi_2^{\dagger}(\pola)$ is equal to $\targetpi\{s,a_\cL\}$ in states $s$ s.t. $\occstate^{\pola, \targetpi}(s') > 0$, while in states $s$ s.t.
$\occstate^{\pola, \targetpi}(s') = 0$, 
$\ntargetpiproxy\{s,a_\cL\}$ takes action $a_\cL'$, i.e., $\ntargetpiproxy\{s,a_\cL\}(s', a_\cL') > 0$ only if  $a_\cL' \in \argmax_{a_\cL''} \sum_{a_{\cA}} \bar Q_{\pi_1}^{*\{s,a_\cL\}}(s, a_{\cA}, a_\cL'')$.
\end{lemma}
\begin{proof}
We use similar arguments as in the proof of Lemma \ref{lm.neighbor_policies.basic} and in the proof of Lemma 1 in \cite{rakhsha2021policy}. 
The necessity trivially follows as otherwise the constraints of the optimization problem \eqref{prob.instance_1} would be violated given that $\ntargetpiproxy$ can be deterministic. 

Now, consider a deterministic policy $\poll^*$ s.t. $\poll^*(s, .) \ne \targetpi(s, .)$ for at least one state $s$ that satisfies $\occstate^{\pi_1, \targetpi}(s) > 0$, and suppose that the conditions of the lemma hold. Denote $a_s$ action for which $\poll^*(s, a_s) = 1.0$. We have that:
\begin{align*}
     \score_2^{\pola, \poll^*} - \score_2^{\pola, \targetpi} &= \score_2^{\pola, \poll^*} - \score_2^{\pola, \targetpiproxy}\\ &\underbrace{=}_{(i)} \sum_{s} \occstate^{\pola, \poll^*}(s) \cdot \left ( Q^{\pola, \targetpiproxy}(s, \pola, \poll^*) - V^{\pola, \targetpiproxy}(s) \right ) \\
     &\underbrace{\le}_{(ii)} \sum_{s|\occstate^{\pola, \targetpi}(s) > 0}  \occstate^{\pola, \poll^*}(s) \cdot \left ( \tilde Q_{\pola}^{*}(s, \pola, \poll^*) - \tilde V_{\pola}^{*}(s) \right)\\
     &\underbrace{=}_{(i)}\sum_{s|\occstate^{\pola, \targetpi}(s) > 0 \land \poll^*(s, .) \ne \targetpiproxy(s, . )} \occstate^{\pola, \poll^*}(s) \cdot \frac{\score_2^{\pola, \targetpiproxy \{s, a_s\}} -  \score_2^{\pola, \targetpiproxy}}{\occstate^{\pola, \targetpiproxy \{s, a_s\}}(s)}\\
     &\underbrace{\le}_{(iii)} \sum_{s|\occstate^{\pola, \targetpi}(s) > 0 \land \poll^*(s, .) \ne \targetpiproxy(s, . )} \occstate^{\pola, \poll^*}(s) \cdot \frac{\score_2^{\pola, \ntargetpiproxy \{s, a_s\}} -  \score_2^{\pola, \targetpiproxy}}{\occstate^{\pola, \targetpiproxy \{s, a_s\}}(s)}\\
     &\underbrace{\le}_{(iv)} - \eps \cdot \sum_{s|\occstate^{\pola, \targetpi}(s) > 0 \land \poll^*(s, .) \ne \targetpiproxy(s, .)} \frac{ \occstate^{\pola, \poll^*}(s)}{\occstate^{\pola, \targetpiproxy \{s, a_s\}}(s)} \le 0,
\end{align*}
where $\targetpiproxy$ is defined in Lemma \ref{lm.neighbor_policies.basic}. To obtain the inequalities, we applied Lemma \ref{lm.score_advantage} over the neighboring policies (for $(i)$),
used the fact that $\targetpiproxy$ satisfies \eqref{eq.belman_equations} (for $(ii)$), and used the assumption that the constraints in Lemma \ref{lm.neighbor_policies} are satisfied (for $(iv)$). 
For $(iii)$, it suffices that  $\score_2^{\pi, \ntargetpiproxy\{s,a_2\}} \ge \score_2^{\pi, \targetpiproxy\{s,a_2\}}$. We can show this by using Lemma \ref{lm.score_advantage} and the definition of $\ntargetpiproxy\{s, a_2\}$:
\begin{align}\label{eq.ntargetpiproxy_opt}
    \score_2^{\pola, \targetpiproxy\{s,a_2\}} - \score_2^{\pola, \ntargetpiproxy\{s,a_2\}} &= \sum_{s'} \occstate^{\pola, \targetpiproxy\{s, a_{\cL}\}}(s') \cdot \left ( Q^{\pola, \ntargetpiproxy}(s', \pola, \targetpiproxy\{s,a_2\}) - V^{\pola, \ntargetpiproxy}(s') \right ) \\
    &= \sum_{s'|\occstate^{\pola, \targetpi}(s') = 0} \occstate^{\pola, \targetpiproxy\{s, a_{\cL}\}}(s') \cdot \left ( Q^{\pola, \ntargetpiproxy}(s', \pola, \targetpiproxy\{s,a_2\}) - V^{\pola, \ntargetpiproxy}(s') \right ) \nonumber
    \\&\le 0, \nonumber
\end{align}
where the inequality is due to the fact that $\ntargetpiproxy$ satisfies Eq. \eqref{eq.belman_equations.ntargetproxy}. Altogether, we have that $\score_2^{\pola, \poll^*} \le  \score_2^{\pola, \targetpi}$.

Now, notice that
\begin{align*}
    \score_2^{\pola, \poll^*} - \score_2^{\pola, \targetpi} &= \sum_{s} \occstate^{\pola, \targetpi}(s) \left ( Q^{\pola, \poll^*}(s, \pola, \targetpi) - V^{\pola, \poll^*}(s) \right ) \\ 
    &=\sum_{s|\occstate^{\pola, \targetpi}(s) > 0} \occstate^{\pola, \targetpi}(s) \left ( Q^{\pola \poll^*}(s, \pola, \targetpi) - V^{\pola, \poll^*}(s) \right )\\
     &=\sum_{s|\occstate^{\pola, \targetpi}(s) > 0 \land \poll^*(s, .) \ne \targetpi(s, .)} \occstate^{\pola, \targetpi}(s) \left ( Q^{\pola, \poll^*}(s, \pola, \targetpi) - V^{\pola, \poll^*}(s) \right ),
\end{align*}
which due to $\score_2^{\pola, \poll^*} \le  \score_2^{\pola, \targetpi}$ and considered $\poll^*$ implies that there exists $s$ and $a^*_\cL$ that satisfy
\begin{align*}
    \occstate^{\pola, \targetpi}(s) > 0 \land \poll^*(s, .) \ne \targetpi(s, .) \land \targetpi(s, a^*_\cL) = 1.0,
\end{align*}
as well as $Q^{\pola, \poll^*}(s, \pola, a_\cL^*) - V^{\pola, \poll^*}(s) \ge 0$.
In turn, this means that the neighboring policy $\poll^*\{s, a^*_\cL\}$ satisfies $\score_2^{\pola, \poll^*} \le \score_2^{\pola, \poll^*\{s, a^*_\cL\}}$ because
\begin{align*}
    \score_2^{\pola, \poll^*\{s, a_\cL^*\}} - \score_2^{\pola, \poll^*} = \occstate^{\pola, \poll^*\{s, a_\cL^*\}}(s) \cdot \left ( 
    Q^{\pola, \poll^*}(s, \pola, a_\cL^*) - V^{\pola, \poll^*}(s)
    \right ) \ge 0.
\end{align*}
Denote this policy by $\poll^{*, k-1}$ where $k$ is the number of states $s$ that satisfy $\occstate^{\pola, \targetpi}(s) > 0 \land \poll^*(s, .) \ne \targetpi(s, .)$. By induction and by the definition of $\targetpiproxy$, we further have that
\begin{align*}
    \score_2^{\pola, \poll^*} - \score_2^{\pola, \targetpi} &\le \score_2^{\pola, \poll^{*, 1}} - \score_2^{\pola, \targetpi}\\ 
    &\underbrace{\le}_{(i)} \score_2^{\pola, \ntargetpiproxy\{s, a_2\}} - \score_2^{\pola, \targetpi} \underbrace{\le}_{(ii)} \eps,
\end{align*}
where $s$ satisfies $\occstate^{\pola, \targetpi}(s) > 0 \land \poll^{*, 1}(s, .) \ne \targetpi(s, .)$ and $a_\cL$ satisfies $\poll^{*, 1}(s, a_\cL) = 1.0$. The fist inequality ($(i)$) holds because $\score_2^{\pola, \poll} \le \score_2^{\pola, \ntargetpiproxy\{s,a_2\}}$ for any policy $\poll$ that satisfies $\poll(s', \cdot) = \ntargetpiproxy\{s,a_2\}(s', \cdot)$ in states $s'$ for which $\occstate^{\pola, \targetpi}(s') > 0$---one can show this using the same analysis as in \eqref{eq.ntargetpiproxy_opt}:
\begin{align*}
    \score_2^{\pola, \poll} - \score_2^{\pola, \ntargetpiproxy\{s,a_2\}} &= \sum_{s'} \occstate^{\pola, \poll}(s') \cdot \left ( Q^{\pola, \ntargetpiproxy}(s', \pola, \poll) - V^{\pola, \ntargetpiproxy}(s') \right ) \\
    &= \sum_{s'|\occstate^{\pola, \poll}(s') = 0} \occstate^{\pola, \poll}(s') \cdot \left ( Q^{\pola, \ntargetpiproxy}(s', \pola, \poll) - V^{\pola, \ntargetpiproxy}(s') \right ) 
    \\&\le 0.
\end{align*}
The second inequality ($(ii)$) is due to the constraints. This concludes the proof. 
\end{proof}

Finally, we also present a version of the result for a special case of the studied setting, which is used in deriving an upper bound to the cost of optimal attack. 
\begin{lemma}\label{lm.neigbhor_policies_ergodic}
Consider a policy $\pola$, target policy $\targetpi$, and assume that the Markov chain induced by $\pola$ and $\targetpi$ is ergodic, i.e., $\occstate^{\pola, \targetpi}(s) > 0$, for all states $s$. Then, $\pola$ satisfies the constraints of the optimization problem \eqref{prob.instance_1} if and only if
\begin{align*}
     \score_2^{\pola, \targetpi} \ge \score_2^{\pola, \targetpi \{s,a_\cL \}} + \eps \quad \forall s, a_\cL \text{ \em s.t. } \targetpi(s, a_\cL) = 0.
\end{align*}
\end{lemma}
\begin{proof}
We can prove it by following the proof of Lemma \ref{lm.neighbor_policies}. Namely, we have that:
\begin{align*}
    \score_2^{\pola, \poll^*} - \score_2^{\pola, \targetpi} &= \sum_{s}  \occstate^{\pola, \poll^*}(s) \cdot \left ( Q^{\pola, \targetpi}(s, \pola, \targetpi) - V^{\pola, \targetpi}(s) \right)\\
     &=\sum_{s|\poll^*(s, .) \ne \targetpi(s, . )} \occstate^{\pola, \poll^*}(s) \cdot \frac{\score_2^{\pola, \targetpi \{s, a_s\}} -  \score_2^{\pola, \targetpi}}{\occstate^{\pola, \targetpi \{s, a_s\}}(s)}\\
     &\le - \eps \cdot \sum_{s| \poll^*(s, .) \ne \targetpi(s, .)} \frac{ \occstate^{\pola, \poll^*}(s)}{\occstate^{\pola, \targetpi \{s, a_s\}}(s)} \le 0. 
\end{align*}
where we used the same arguments as in the proof of Lemma \ref{lm.neighbor_policies}. Therefore, $\score_2^{\pola, \poll^*} \le \score_2^{\pola, \targetpi}$. We proceed as in the proof of Lemma \ref{lm.neighbor_policies}, i.e.,
\begin{align*}
    \score_2^{\pola, \poll^*} - \score_2^{\pola, \targetpi} &= \sum_{s} \occstate^{\pola, \targetpi}(s) \left ( Q^{\pola, \poll^*}(s, \pola, \targetpi) - V^{\pola, \poll^*}(s) \right ) \\ 
     &=\sum_{s|\poll^*(s, .) \ne \targetpi(s, .)} \occstate^{\pola, \targetpi}(s) \left ( Q^{\pola, \poll^*}(s, \pola, \targetpi) - V^{\pola, \poll^*}(s) \right ),
\end{align*}
which due to $\score_2^{\pola, \poll^*} \le  \score_2^{\pola, \targetpi}$ and considered $\poll^*$ implies that there exists $s$ and $a^*_\cL$ that satisfy $\poll^*(s, .) \ne \targetpi(s, .)$, $\targetpi(s, a^*_\cL) = 1.0$,
and $Q^{\pola, \poll^*}(s, \pola, a_\cL^*) - V^{\pola, \poll^*}(s) \ge 0$.
In turn, this means that the neighboring policy $\poll^*\{s, a^*_\cL\}$ satisfies $\score_2^{\pola, \poll^*} \le \score_2^{\pola, \poll^*\{s, a^*_\cL\}}$ because
\begin{align*}
    \score_2^{\pola, \poll^*\{s, a_\cL^*\}} - \score_2^{\pola, \poll^*} = \occstate^{\pola, \poll^*\{s, a_\cL^*\}}(s) \cdot \left ( 
    Q^{\pola, \poll^*}(s, \pola, a_\cL^*) - V^{\pola, \poll^*}(s)
    \right ) \ge 0.
\end{align*}
Denote this policy by $\poll^{*, k-1}$ where $k$ is the number of states $s$ that satisfy $\poll^*(s, .) \ne \targetpi(s, .)$. By induction we obtain
\begin{align*}
    \score_2^{\pola, \poll^*} - \score_2^{\pola, \targetpi} &\le \score_2^{\pola, \poll^{*, 1}} - \score_2^{\pola, \targetpi} 
    =\score_2^{\pola, \targetpi\{s, a_2\}} - \score_2^{\pola, \targetpi} \le -\eps,
\end{align*}
where $s$ and $a_\cL$ satisfy $\poll^*(s, .) \ne \targetpi(s, .)$ and $\poll^{*, 1}(s, a_\cL) = 1.0$. 
The last inequality is due to the conditions of the lemma.
\end{proof}

\section{Proof of Theorem \ref{thm.computational_complexity}}\label{sec.app.comp_complexity}

In this section, we provide the proof of Theorem \ref{thm.computational_complexity} from Section \nameref{sec.characterization}. First, we analyze two special cases. 

Let us consider a special case when transitions are independent of policy $\pola$, i.e., $P(s, a_{\cA}, a_{\cL}) = P(s, a_{\cA}', a_{\cL})$, while the Markov chain induced by $\targetpi$ and any policy $\pola$ is ergodic. In this case, Lemma \ref{lm.neigbhor_policies_ergodic} implies that the constraints of the optimization problem are
\begin{align*}
     \score_2^{\initpi, \targetpi} \ge \score_2^{\initpi, \targetpi \{s,a_\cL \}} + \eps \quad \forall s, a_\cL \text{ \em s.t. } \targetpi(s, a_\cL) = 0.
\end{align*}
From Eq. \eqref{eq.return_occupance}, it follows that these constraints are linear. Therefore,  \eqref{prob.instance_1} is in this case a convex optimization problem, and can be efficiently solved.

Let us now consider another special case when transitions are independent of policy $\pola$ and $\poll$, i.e., $P(s, a_{\cA}, a_{\cL}) = P(s, a_{\cA}', a_{\cL}')$. In this case, the state occupancy measure is independent of the agents' policies. Hence, it follows that $\targetpi(s, .) = \ntargetpiproxy(s, .)$ for states $s$ s.t. $\occstate^{\initpi, \targetpi}(s) > 0$, which by Lemma \ref{lm.neighbor_policies} implies that the constraints of the optimization problem are
\begin{align*}
     \score_2^{\initpi, \targetpi} \ge \score_2^{\initpi, \targetpi \{s,a_\cL \}} + \eps \quad \forall s, a_\cL \text{ \em s.t. } \targetpi(s, a_\cL) = 0.
\end{align*}
We can again use Eq. \eqref{eq.return_occupance} to conclude that \eqref{prob.instance_1} is in this case a convex optimization problem, and can be efficiently solved.

We now turn to the proof of Theorem \ref{thm.computational_complexity}.  

\textbf{Statement of Theorem \ref{thm.computational_complexity}}: {\em %
It is NP-hard to decide whether the optimization problem \eqref{prob.instance_1} is feasible, i.e., whether there exists solution $\pola$ s.t. the constraints of the optimization problem are satisfied.
}
\begin{proof}
We prove this statement by reducing a generic instance of the Boolean 3-SAT problem to the setting of interest, such that agent $\cA$ can force $\targetpi$ if and only if the 3-SAT instance is satisfiable.
As commonly known, in 3-SAT, we have $n$ binary variables $x_1$, ... , $x_n$, and $m$ clauses $C_1$, ..., $C_m$, each clause containing three literals (variables $x_j$ or their negations $\bar x_j$). We will denote $j$-th literal of $C_i$ by $C_{i, j}$. The decision problem is to determine whether there exists an assignment of variables $x_j$ s.t. all the clauses have at least one literal that evaluates to {\em true}. 

We consider the following reduction for a given 3-SAT instance. Let us encode this instance using the two-agent MDP model of the setting considered in this paper, i.e., $\cM = (\{ \cA, \cL \}, S, A, P, R, \gamma, \sigma)$. We start by describing the state space $S$, which contains:
\begin{itemize}
    \item initial state $s_I$, 
    \item states $s_{C_i}$, one associated to each clause $C_i$,
    \item states $s_{x_j}$, one associated to each positive literal $x_j$,
    \item states $s_{\bar x_j}$, one associated to each negative literal $\bar x_j$,
    \item states $s_{val(x_j)}$, one associated to variable $x_j$ (i.e., its value),
    \item and final state $s_{F}$.
\end{itemize}
Note that if $x$ (resp. $\bar x$) is $j$-th literal in $C_{i}$, we will also denote $s_{x}$ (resp. $s_{\bar x}$) by $s_{C_{i, j}}$. 
Given that $s_I$ is the initial state, the initial state distribution is $\sigma(s_I) = 1.0$, i.e., the initial state is $s_I$ with probability $1.0$. 

The action space of agents consists of three and two actions respectively, $A_\cA = \{ a_{\cA, 1}, a_{\cA, 2}, a_{\cA, 3}\}$ and $A_{\cL} = \{ a_{\cL, 1}, a_{\cL, 2}\}$. The reward function is defined as
\begin{align*}
     R_{\cL}(s, a_\cA, a_\cL) = \begin{cases}
        -\infty &\mbox{ if }  s \ne s_{x_j} \wedge s \ne s_{\bar x_j}\text{ and } a_\cL = a_{\cL, 2}\\
        \gamma &\mbox{ if }  s = s_{x_j} \text{ and } a_\cL = a_{\cL, 2}\\
        1+ \frac{\epsilon \cdot m}{(1-\gamma) \cdot \gamma^3} &\mbox{ if } s = s_{val(x_j)} \text{ and } a_\cA = a_{\cA, 1} \text{ and } a_\cL = a_{\cL, 1} \\
        - \frac{\epsilon \cdot m}{(1-\gamma) \cdot \gamma^3} &\mbox{ if } s = s_{val(x_j)} \text{ and } a_\cA \ne a_{\cA, 1} \text{ and } a_\cL = a_{\cL, 1} \\
        0 &\mbox{ otherwise }.
   \end{cases}
\end{align*}
The transition matrix is defined as:
\begin{align*}
   P(s, a_1, a_2, s') = \begin{cases}
        1/m &\mbox{ if } s = s_I \text{ and } s' = s_{C_i}\\
        1 &\mbox{ if } s = s_{C_i} \text{ and } a_\cA = a_{\cA, j} \text{ and } s' = s_{C_{i, j}} \\
        1 &\mbox{ if } s = s_{x_j} \text{ and } a_\cL = a_{\cL, 1} \text{ and } s' = s_{val(x_j)} \\
        1 &\mbox{ if } s = s_{x_j} \text{ and } a_\cL = a_{\cL, 2} \text{ and } s' = s_F \\
        1 &\mbox{ if } s = s_{\bar x_j} \text{ and } a_\cL = a_{\cL, 1} \text{ and } s' = s_{F} \\
        1 &\mbox{ if } s = s_{\bar x_j} \text{ and } a_\cL = a_{\cL, 2} \text{ and } s' = s_{val(x_j)} \\
        1 &\mbox{ if } s = s_{F} \text{ and }  s' = s_{F} \\
        0 &\mbox{ otherwise }
   \end{cases}
\end{align*} 
Finally, we define target policy $\targetpi$ as:
\begin{align*}
    \targetpi(s, a_\cL) = 
	\begin{cases} 
		1 &\mbox{ if } a = a_{\cL, 1},\\
		0 &\mbox{ otherwise}.
	\end{cases}
\end{align*}
Note that this construction has polynomial complexity in the number of variables of the 3-SAT problem, hence it is efficient. Figure \ref{fig.mdp3sat} provides intuition behind this construction, as well as some intuition behind the reduction used in the proof. To show NP-hardness, we need to prove that the 3-SAT problem is satisfiable if and only if the optimization problem \eqref{prob.instance_1} is feasible for the corresponding MDP.

\textbf{Direction 1}: We first show that if the optimization problem \eqref{prob.instance_1} is feasible for the above MDP, then the corresponding 3-SAT problem is satisfiable. Let $\pola$ be the policy that forces $\targetpi$. 
Now, suppose there exists $x_j$ for which $\occstate^{\pola, \targetpi}(s_{x_j}) > 0$. Note that $\targetpi\{s_{x_j}, a_{\cL, 2} \} = \targetpiproxy\{s_{x_j}, a_{\cL, 2} \} = \ntargetpiproxy\{s_{x_j}, a_{\cL, 2} \}$ in the MDP considered in the proof. Namely, agent 2 has a uniquely optimal action, $a_{\cL, 1}$, in all states other than $s_{x_j'}$ since the reward for taking action $a_{\cL, 1}$ in $s\ne s_{x_j'}$ is  $-\infty$. Due to condition \eqref{eq.neighbor_constrain} in Lemma \ref{lm.neighbor_policies} 
and the choice of the target policy, we have that 
\begin{align*}
    &\score_2^{\pola, \targetpi} - \score_2^{\pola, \targetpi\{s_{x_j}, a_{\cL, 2} \}}\ge \eps\\
    &\gamma \cdot \occstate^{\pola, \targetpi}(s_{x_j}) \cdot R_{\cL}(s_{val(x_j)}, \pola, a_{\cL, 1}) \ge \occstate^{\pola, \targetpi}(s_{x_j}) \cdot \gamma + \eps,
\end{align*}
where we used that $\occstate^{\pola, \targetpi}(s_{x_j}) = \occstate^{\pola, \targetpi\{s_{x_j}, a_{\cL} \}}(s_{x_j})$ (agent $\cL$ acts only after $x_j$ is reached), $\occstate^{\pola, \targetpi}(s_{x_j}) = 1/\gamma \cdot \occstate^{\pola, \targetpi}(s_{val(x_j)})$ (since $\targetpi$ takes $a_{\cL, 1}$ in $s_{x_j}$ transitioning to $s_{val(x_j)}$), and $\targetpi = \targetpiproxy$ (since agent $\cL$ has uniquely optimal policy for states other than $s_{x_j}$ and $s_{\bar x_j}$). This implies that $R_{\cL}(s_{val(x_j)}, \pola, a_{\cL, 1}) \ge 1 + \frac{\eps}{\gamma \cdot \occstate^{\pola, \targetpi}(s_{x_j})} > 1$. Similarly, if we consider $\bar x_j$ for which $\occstate^{\pola, \targetpi}(s_{\bar x_j}) > 0$, we obtain
\begin{align*}
    &\score_2^{\pola, \targetpi}- \score_2^{\pola,\targetpi\{s_{\bar x_j}, a_{\cL, 2} \}}\ge \eps\\
    &\gamma \cdot \occstate^{\pola, \targetpi}(s_{\bar x_j}) \cdot 0 \ge \gamma \cdot \occstate^{\pola, \targetpi}(s_{\bar x_j}) \cdot R_{\cL}(s_{val(x_j)}, \pola, a_{\cL, 1}) + \eps,
\end{align*}
which implies that $R_{\cL}(s_{val(x_j)}, \pola, a_{\cL, 1}) \le -\frac{\eps}{\gamma \cdot \occstate^{\pola, \targetpi}(\bar s_{x_j})} < 0$. To conclude, the probability of visiting $s_{x_i}$ is strictly positive only if $R_{\cL}(s_{val(x_j)}, \pola, a_{\cL, 1}) > 1$, and the probability of visiting $s_{\bar{x}_i}$ is strictly positive only if $R_{\cL}(s_{val(x_j)}, \pola, a_{\cL, 1}) < 0$. 

This means that $s_{x_j}$ and $s_{\bar{x}_j}$ cannot simultaneously have strictly positive occupancy measure for $\pola$ that forces $\targetpi$. Given that $\pola$ ``selects'' literals when taking actions in states $s_{C_i}$, and every $s_{C_i}$ has a strictly positive occupancy measure, this further implies that $\pola$: a) ``selects'' at least one literal for each $s_{C_i}$, b) either ``selects'' $x_j$ or $\bar{x}_j$ with $0$ probability across all $s_{C_i}$. Since $s_{C_i}$ corresponds to clause $C_i$, $\pola$'s selection corresponds to the solution of the corresponding 3-SAT problem since $\pola$ identifies for each $C_i$ which literals should evaluate to true (those for which $\pola(s_{C_i}, a_{\cA, j}) > 0$) and this assignment is not inconsistent across clauses $C_i$ ($\pola$ never selects  $x_{j}$ and $\bar x_{j}$ across $s_{C_i}$). Therefore, to obtain an assignment of variables that satisfies the corresponding 3-SAT problem, it is enough to set $x_j$ to {\em true} if $R_{\cL}(s_{val(x_j)}, \pola, a_{2, 1}) > 1$, and otherwise set $x_j$ to  {\em false}.

\textbf{Direction 2}: Next, we show that if a given instance of the 3-SAT problem is satisfiable, then the optimization problem \eqref{prob.instance_1} is feasible for the corresponding 2-agent MDP problem. Consider the following policy $\pi_1$
\begin{itemize}
\item If the literal $x_i$ is {\em true} in the 3-SAT solution, take action $a_{\cA, 1}$ in state $s_{val(x_i)}$. Otherwise, take another action (either $a_{\cA, 2}$ or $a_{\cA, 3}$).
\item In state $s_{C_i}$ choose one literal that satisfies the clause in the 3-SAT solution.
\end{itemize}
We can now utilize Lemma \ref{lm.neighbor_policies} again to prove the claim. First note that in all states other than $s_{x_j}$ and $s_{\bar x_j}$ any deviation from the target policy results in $-\infty$ reward. Hence, we focus on states $s_{x_j}$ and $s_{\bar x_j}$.

Consider $x_j$ for which $\occstate^{\pola, \targetpi}(s_{x_j}) > 0$. We have that 
\begin{align*}
    \score_2^{\pola, \targetpi} - \score_2^{\pola, \targetpi\{s_{x_j}, a_{\cL, 2} \}} &= 
     \mu^{\pola, \targetpi}(s_{x_j}) \cdot \left (1 + \frac{\epsilon \cdot m}{(1-\gamma) \cdot \gamma^3} \right )- \mu^{\pola, \targetpi}(s_{x_j}) \cdot 1\\
     &\ge \frac{(1-\gamma)\cdot \gamma^3}{m} \cdot \frac{\epsilon \cdot m}{(1-\gamma) \cdot \gamma^3} = \epsilon.
\end{align*}

Now, consider $\bar x_j$ for which $\occstate^{\pola, \targetpi}(s_{\bar x_j}) > 0$. We have that
\begin{align*}
    \score_2^{\pola, \targetpi} - \score_2^{\pola, \targetpi\{s_{\bar x_j}, a_{\cL, 2} \}} &= \mu^{\pola, \targetpi}(s_{\bar x_j}) \cdot 0 - \mu^{\pola, \targetpi}(s_{\bar x_j}) \cdot (- \frac{\epsilon \cdot m}{(1-\gamma) \cdot \gamma^3}) \\
    &= \frac{(1-\gamma) \cdot \gamma^3}{m} \cdot \frac{\epsilon \cdot m}{(1-\gamma)} = \eps
\end{align*}
As for the previous direction, $\targetpi\{s_{x_j}, a_{\cL, 2} \} = \targetpiproxy\{s_{x_j}, a_{\cL, 2} \} = \ntargetpiproxy\{s_{x_j}, a_{\cL, 2} \}$ since agent $\cL$ has uniquely optimal policy for states other than $s_{x_j}$ and $s_{\bar x_j}$.
Hence, by Lemma \ref{lm.neighbor_policies}, policy $\pola$ forces target policy $\targetpi$.

\begin{figure*}[!ht]
\centering
 \begin{tikzpicture}[main/.style = {draw, circle}, node distance={25mm}, thick] 
\node[main] (0) {$s_I$};
\node[main] (1) [above right of=0] {$s_{C_1}$};
\node[main] (2) [below right of=0] {$s_{C_2}$}; 
\node[main] (3) [right of=1] {$s_{x_2}$}; 
\node[main] (4) [above of=3] {$s_{\bar{x}_1}$}; 
\node[main] (5) [above of=4] {$s_{x_1}$}; 
\node[main] (6) [right of=2] {$s_{\bar{x}_2}$}; 
\node[main] (7) [below of=6] {$s_{x_3}$}; 
\node[main] (8) [below of=7] {$s_{\bar{x}_3}$}; 
\node[main] (9) [below right of=5] {$s_{val(x_1)}$}; 
\node[main] (10) [below right of=3] {$s_{val(x_2)}$}; 
\node[main] (11) [above right of=8] {$s_{val(x_3)}$}; 
\node[main] (12) [right=4cm of 10] {$s_F$}; 
\draw (0) edge[->, above, sloped, black] node{$P(s_{C_1}) = \frac{1}{2}$}  (1);
\draw (0) edge[->, below, sloped, black] node{$P(s_{C_2}) = \frac{1}{2}$}  (2);
\draw (1) edge[->, above, sloped, red] (5);
\draw (1) edge[->, below, sloped, green] (6);
\draw (1) edge[->, below, sloped, blue] (7);
\draw (2) edge[->, above, sloped, red] (4);
\draw (2) edge[->, below, sloped, green] (6);
\draw (2) edge[->, below, sloped, blue] (8);
\draw (5) edge[->, below, sloped, orange] (9);
\draw (5) edge[->, bend left = 50, above, sloped, purple] node{$r_2 = \gamma$} (12);
\draw (3) edge[->, below, sloped, orange] (10);
\draw (3) edge[->, above, sloped, purple] node{$r_2 = \gamma$} (12);
\draw (7) edge[->, below, sloped, orange] (11);
\draw (7) edge[->, above, sloped, purple] node{$r_2 = \gamma$} (12);
\draw (4) edge[->, below, sloped, orange] (12);
\draw (4) edge[->, below, sloped, purple] (9);
\draw (6) edge[->, below, sloped, orange] (12);
\draw (6) edge[->, below, sloped, purple] (10);
\draw (8) edge[->, bend right = 50, below, sloped, orange] (12);
\draw (8) edge[->, below, sloped, purple] (11);
\draw (9) edge[->, above, sloped, black] node{$r_2 = val$} (12);
\draw (10) edge[->, above, sloped, black] node{$r_2 = val$} (12);
\draw (11) edge[->, below, sloped, black] node{$r_2 = val$} (12);
\end{tikzpicture} 
\caption{The figure depicts an example of the 2-agent MDP encoding of a 3-SAT problem with $2$ clauses $C_i$ and $3$ variables $x_j$. i) In this MDP, the start state is $s_I$. Regardless of the agents' actions, the next state is either $s_{C_1}$ or $s_{C_2}$, chosen with equal probability. The reward of agent $\cL$ is equal to $0$ in state $s_I$.
ii) In state $s_{C_i}$, the action of agent $\cL$ does influence transitions nor rewards. Agent $\cA$ can select between its three actions, effectively choosing on of the literals in $C_i$, e.g., by taking $a_{\cA, j}$, the agent selects $j$-th literal in $C_i$, i.e., $C_{i, j}$. When action $a_{\cA}$ is taken, the agents transition to the corresponding literal state, i.e., $s_{C_{i, a_{\cA}}}$, and the reward of agent $\cL$ is $0$.
iii) In state $s_{x_i}$, agent $\cA$ does not influence transitions nor rewards. Agent $\cL$ can choose between its two actions. If action $a_{\cL, 1}$ is taken, the next state is $s_{val(x_i)}$ and the immediate reward is $0$. If action $a_{\cL, 2}$ is taken, the immediate reward is $\gamma$ and the next state is $s_F$. These transitions effectively simulate the choice that agent $\cL$ is making: it can either keep $x_i$ and receive the value in state $s_{val(x_i)}$, or it can collect reward of $\gamma$ and transition to $s_F$. 
iv) In state $\bar{x}_i$, agent $\cA$ does not influence transitions nor rewards. Agent $\cL$ can choose between its two actions. If action $a_{\cL, 1}$, the next state is $s_F$. If action $a_{\cL, 2}$ is taken,  the next state is $s_{val(x_i)}$. In both cases, the immediate reward is $0$. Similar to the previous case, agent $\cL$ can either transition to $s_F$ obtaining reward $0$, or can flip $\bar x_i$ and receive the value in state $s_{val(x_i)}$.
v) 	In state $val_{x_i}$, agent $\cA$ assigns a ``value'' to $x_i$. If it takes action $a_{\cA, 1}$, the reward of agent $\cL$ is equal to $1+ \frac{\epsilon \cdot m}{(1-\gamma) \cdot \gamma^3}$, which simulates the case when $x_i$ is set to {\em true}. Otherwise, if agent $\cA$ takes action $a_{\cA, 2}$ or $a_{\cA, 3}$, the reward of agent $\cL$ is equal to is equal to $- \frac{\epsilon \cdot m}{(1-\gamma) \cdot \gamma^3}$, which simulates the case when $x_i$ is set to {\em false}. vi) {\bf Goal:} Force the actions of agent $\cL$ that lead to orange transitions, by selecting literals in clauses $C_1$ and $C_2$ and setting the values of variables $x_i$ to $1+ \frac{\epsilon \cdot m}{(1-\gamma) \cdot \gamma^3}$ or $- \frac{\epsilon \cdot m}{(1-\gamma) \cdot \gamma^3}$.  
}\label{fig.mdp3sat}
\label{fig.env}
\end{figure*}
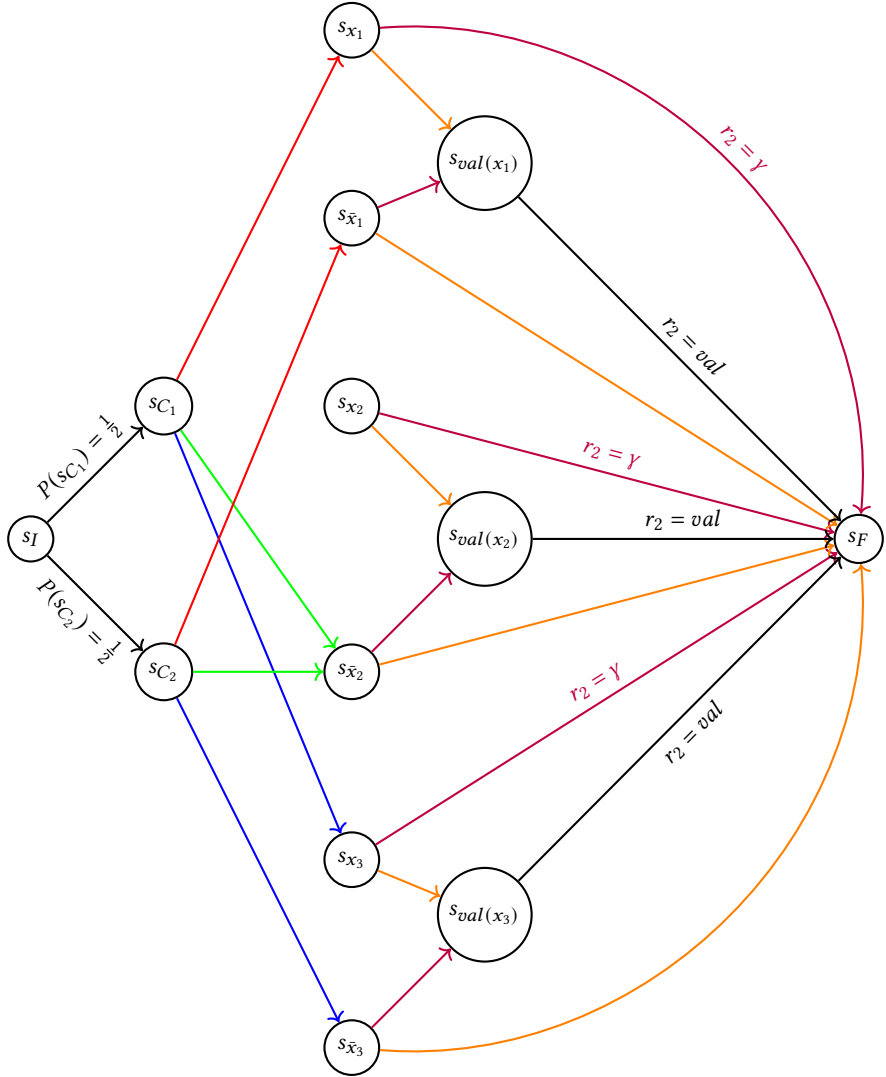

\end{proof}

\section{Proof of Theorem \ref{thm.cost_lower_bound}}\label{sec.app.lower_bound}

To prove the theorem, we first prove a few intermediate results. Our proof technique is similar to the ones from prior work, n particular \cite{ma2019policy,rakhsha2021policy}, and can be considered as an adaptation of these techniques to our problem setting.

We first combine the lower bounds in Lemma \ref{lm.policy_reward_transitions} to obtain a lower bound on the influence of $\pola$ on $Q^{\poll}_{\pola}$-values relative to $\pola'$. Such a lower bound is useful since we can state an optimality condition for a target policy using $Q^{\poll}_{\pola}$. 

\begin{lemma}\label{lm.cost_lower_bound}
For any policy $\pola$, policy $\pola'$, and policy $\poll$, the following holds	
\begin{align*}
\norm{\pola - \pola'}_{1, \infty} \ge \frac{1-\gamma}{ \norm{R_{\cL}}_{\infty} + \gamma \cdot  \norm{V^{\pola', \pi_2}}_{\infty}} \cdot \norm{Q^{\poll}_{\pola} - Q_{\pola'}^{\poll}}_{\infty}.
\end{align*}
\end{lemma}
\begin{proof}
When written in the vector notation, we obtain
\begin{align*}
	\norm{Q^{\poll}_{\pola} - Q_{\pola'}^{\poll}}_{\infty} &= \norm{R_{\cL}^{\pola} + \gamma \cdot (P^{\pola})^T \cdot V^{\pola, \poll} - R_{\cL}^{\pola'} - \gamma \cdot (P^{\pola'})^T \cdot V^{\pola', \poll}}_{\infty}
	\\&\underbrace{\le}_{(i)} \norm{R_{\cL}^{\pola} - R^{\pola'}}_\infty + \gamma \cdot \norm{(P^{\pola})^T \cdot V^{\pola, \poll} - (P^{\pola'}))^T \cdot V^{\pola', \poll}}_{\infty}\\
	\\&= \norm{R_{\cL}^{\pola} - R_{\cL}^{\pola'}}_\infty + \gamma \cdot \norm{(P^{\pola} - P^{\pola'})^T \cdot V^{\pola', \poll} + (P^{\pola})^T\cdot (V^{\pola, \poll} - V^{\pola', \poll}) }_{\infty}\\
	\\&\underbrace{\le}_{(ii)} \norm{R_{\cL}^{\pola} - R_{\cL}^{\pola'}}_\infty + \gamma \cdot \norm{(P^{\pola} - P^{\pola'})^T \cdot V^{\pola', \poll}}_{\infty} + \gamma \cdot \norm{(P^{\pola})^T \cdot (V^{\pola, \poll} - V^{\pola', \poll}) }_{\infty},
\end{align*}
where in $(i)$ and $(ii)$ we applied the triangle inequality. Now, note that 
\begin{align*}
    \norm{(P^{\pola} - P^{\pola'})^T \cdot V^{\pola', \poll}}_{\infty} &= \max_{s, a_\cL} |\sum_{s'} ( P^{\pola}(s, a_\cL, s') - P^{\pola'}(s, a_\cL, s')) \cdot V^{\pola', \poll}(s') | \\
    &\le \max_{s, a_\cL} \sum_{s'} | P^{\pola}(s, a_\cL, s') - P^{\pola'}(s, a_\cL, s')| \cdot \norm{V^{\pola', \poll}}_{\infty} = \norm{P^{\pola} - P^{\pola'}}_{1, \infty} \cdot \norm{V^{\pola', \poll}}_{\infty}. 
\end{align*}
Similarly
\begin{align*}
    \norm{(P^{\pola})^T \cdot (V^{\pola, \poll} - V^{\pola', \poll}) }_{\infty} &= \max_{s, a_\cL} |\sum_{s'} P^{\pola}(s, a_\cL, s')\cdot ( V^{\pola, \poll}(s') - V^{\pola', \poll}(s'))| \\
    &\le \max_{s, a_\cL} \sum_{s'} | P^{\pola}(s, a_\cL, s')| \cdot \norm{V^{\pola, \poll} - V^{\pola', \poll}}_{\infty} \\
    &= \norm{V^{\pola, \poll} - V^{\pola', \poll}}_{\infty} \le \norm{Q_{\pola}^{\poll} - Q_{\pola'}^{\poll}}_{\infty},
\end{align*}
where the last inequality follows from the fact that for all states $s$
\begin{align*}
    |V^{\pola, \poll}(s) - V^{\pola', \poll}(s)| = | \sum_{a_\cL} \poll(s, a_{\cL}) \cdot ( Q_{\pola}^{\poll}(s, a_{\cL}) - Q_{\pola'}^{\poll}(s, a_{\cL})  )| \le \max_{a_\cL} | Q_{\pola}^{\poll}(s, a_{\cL}) - Q_{\pola'}^{\poll}(s, a_{\cL}) |.
\end{align*}
Putting everything together, we have
\begin{align*}
    (1-\gamma) \cdot \norm{Q^{\poll}_{\pola} -   Q_{\pola'}^{\poll}}_{\infty} \le \norm{R_{\cL}^{\pola} - R_{\cL}^{\pola'}}_\infty + \gamma \cdot \norm{P^{\pola} - P^{\pola'}}_{1, \infty} \cdot \norm{V^{\pola', \poll}}_{\infty}.
\end{align*}
Finally, we can apply Lemma \ref{lm.policy_reward_transitions} to obtain
\begin{align*}
    (1-\gamma) \cdot \norm{Q^{\poll}_{\pola} -   Q_{\pola'}^{\poll}}_{\infty} \le \norm{\pola - \pola'}_{1, \infty} \cdot \norm{R_2}_\infty + \norm{\pola - \pola'}_{1, \infty} \cdot \gamma \cdot \norm{V^{\pola', \poll}}_{\infty}
\end{align*}
which implies the statement.  
\end{proof}

Finally, we are ready the prove the theorem. 

\textbf{Statement of Theorem \ref{thm.cost_lower_bound}}: {\em 
The attack cost of any solution to the optimization problem \eqref{prob.instance_1}, if it exists, satisfies 
    \begin{align*}
        \Cost(\pola, \initpi) \ge \frac{1-\gamma}{2} \cdot \frac{\norm{\overline \chi_{0}}_{\infty} }{\norm{R_2}_{\infty} + \gamma \cdot \norm{V^{\initpi, \targetpi}}_{\infty}}.
    \end{align*}
}
\begin{proof}
Assume that a policy that forces $\targetpi$ exists, and denote it by $\pola^{*}$. If $\overline \chi_{0}(s, a) = 0$, then the statement trivially follows, since the cost function is non-negative. 
Consider state-action pair $(s', a_\cL')$ such that $\overline \chi_{0}(s',a') > 0$ and let $a^{\dagger}_{\cL}$ be the action that $\targetpi$ takes in $s'$.
Since $\overline \chi_{0}(s',a') > 0$, we know that $\occstate^{\pola, \targetpi}(s) > 0$ for all $\pola$, hence, $\occstate^{\initpi, \targetpi\{s,a_\cL'\}}(s) > 0$ and $\occstate^{\pola^*, \targetpi\{s,a_\cL'\}}(s) > 0$ hold. We have that 
\begin{align*}
    \norm{Q^{\targetpi}_{\pola^*} -   Q_{\initpi}^{\targetpi}}_{\infty} &\ge \frac{1}{2} \cdot \max_{s, a_\cL} [Q^{\targetpi}_{\pola^*}(s, a_\cL) -   Q_{\initpi}^{\targetpi}(s, a_\cL)]  -  \frac{1}{2} \cdot \min_{s, a_\cL} [Q^{\targetpi}_{\pola^*}(s, a_\cL) -   Q_{\initpi}^{\targetpi}(s, a_\cL)] \\
    &\ge \frac{1}{2} \cdot [Q^{\targetpi}_{\pola^*}(s', a^{\dagger}_{\cL}) -   Q_{\initpi}^{\targetpi}(s', a^{\dagger}_{\cL})] - \frac{1}{2} \cdot  [Q^{\targetpi}_{\pola^*}(s', a_\cL') -   Q_{\initpi}^{\targetpi}(s', a_\cL')] \\
    &= \frac{1}{2} \cdot [Q^{\targetpi}_{\pola^*}(s', a^{\dagger}_{\cL}) - Q^{\targetpi}_{\pola^*}(s', a_\cL')] + \frac{1}{2} \cdot [ Q^{\targetpi}_{\initpi}(s', a') - Q^{\targetpi}_{\initpi}(s', a^{\dagger}_{\cL})]\\ 
    &\underbrace{\ge}_{(i)} \frac{1}{2} \cdot [Q^{\targetpi}_{\initpi}(s', a_\cL') - Q^{\targetpi}_{\initpi}(s', a^{\dagger}_{\cL}) ] \underbrace{=}_{(ii)} \frac{\score_{\cL }^{\initpi, \targetpi\{s', a_\cL'\}} - \score_{\cL }^{\initpi, \targetpi}}{ 2 \cdot \occstate^{\initpi, \targetpi\{s',a_\cL'\}}(s)} = \frac{1}{2} \cdot \chi_{0}(s', a_\cL').
\end{align*}
Note that $(i)$ follows the fact that $\pola^*$ forces $\targetpi$, hence, $ \score_{\cL }^{\pola^*, \targetpi} - \score_{\cL }^{\pola^*, \targetpi\{s, a_\cL'\}} \ge \epsilon$, which together with Lemma \ref{lm.score_advantage} and $\occstate^{\pola^*, \targetpi\{s,a_\cL'\}}(s) > 0$ implies $$Q^{\targetpi}_{\pola^*}(s', a^{\dagger}_{\cL}) - Q^{\targetpi}_{\pola^*}(s', a_\cL') = Q^{\pola^*, \targetpi}(s', \pola^*, a^{\dagger}_{\cL}) - Q^{\pola^*, \targetpi}(s', \pola^*, a_\cL') = \frac{ \score_{\cL }^{\pola^*, \targetpi} - \score_{\cL }^{\pola^*, \targetpi\{s, a_\cL'\}}}{\occstate^{\pola^*, \targetpi\{s,a_\cL'\}}(s)} \ge 0.$$ 
To obtain $(ii)$, we can again apply Lemma \ref{lm.score_advantage} together with $\occstate^{\initpi, \targetpi\{s,a_\cL'\}}(s) > 0$, i.e.,
$$Q^{\targetpi}_{\initpi}(s', a_\cL') - Q^{\targetpi}_{\initpi}(s', a^{\dagger}_{\cL}) = Q^{\initpi, \targetpi}(s', \initpi, a_\cL') - Q^{\initpi, \targetpi}(s', \initpi, a^{\dagger}_{\cL}) = \frac{\score_{\cL }^{\initpi, \targetpi\{s', a_\cL'\}} - \score_{\cL }^{\initpi, \targetpi}}{ \occstate^{\initpi, \targetpi\{s',a_\cL'\}}(s)}.$$
Finally, by using $\ell_p$-norm inequalities and Lemma \ref{lm.cost_lower_bound}, we obtain the following lower bound
\begin{align*}
    \Cost(\pola, \initpi) = \norm{\pola - \initpi}_{p, 1} \ge  \norm{\pola - \initpi}_{\infty, 1} &\ge \frac{1-\gamma}{\norm{R_2}_\infty + \gamma \cdot \norm{V^{\pola', \poll}}_{\infty}} \cdot \norm{Q^{\targetpi}_{\pola^*} -   Q_{\initpi}^{\targetpi}}_{\infty} 
    \\&\ge \frac{(1-\gamma) \cdot \norm{\chi_{0}}_\infty}{2 \cdot \left (\norm{R_2}_\infty + \gamma \cdot \norm{V^{\pola', \poll}}_{\infty}\right )}
\end{align*}
\end{proof}

\section{Proof of Theorem \ref{thm.upper_bound_2}}\label{sec.app.upper_bounds}

In this section, we provide the proof of Theorem \ref{thm.upper_bound_2}. Before providing the proof we consider another specific case, when transitions dynamics is independent of policies $\pola$ and $\poll$, i.e., $P(s, a_{\cA}, a_{\cL}) = P(s, a_{\cA}', a_{\cL}')$.

In this case, the agents do not influence the dynamics, which means they can reason myopically, i.e., state-wise, when selecting their policies. 
Therefore, the problem of forcing a target policy becomes computationally tractable, and its feasibility can be efficiently determined (see also Section \nameref{sec.app.comp_complexity}). Proposition \ref{prop.upper_bound_1} provides a sufficient condition for the feasibility of \eqref{prob.instance_1}, as well as an upper bound on the cost of the attack when this condition holds. 
To formally express the feasibility condition, we define:
\begin{align*}
    \alpha_1^{*}(s) = \sup_{\pi_s \in \cP(A_{\cA})} \min_{a_\cL} \occstate^{\initpi, \targetpi}(s) \cdot \sum_{a_\cA} \pi_s(a_{\cA}) \cdot \left [ R_2(s, a_{\cA}, \targetpi) - R_2(s, a_{\cA}, a_\cL) \right ]
\end{align*} 
for $\occstate^{\initpi, \targetpi}(s) > 0$, and otherwise $\alpha^{*}(s) = \eps$. $\alpha_1^*$ measures for each state how much the adversary can increase the relative reward of the target action in states visited by $\pola$ and $\poll$. This gives us: 
\begin{proposition}\label{prop.upper_bound_1}
Assume that $P(s, a_{\cA}, a_{\cL}) = P(s, a_{\cA}', a_{\cL}')$ for all $a_{\cL}$ and $a_{\cA}'$. Then, the optimization problem \eqref{prob.instance_1} is feasible if $\alpha_1^*(s) \ge \eps$ 
and the cost of an optimal solution  satisfies $ \Cost(\pola, \initpi) \le 2 \cdot \norm{\frac{\overline \chi_{\eps}}{\chi_1^*  + \overline \chi_{\eps}}}_{p, \infty}$,
with the element-wise division (equal to $0$ if $\overline \chi_{\eps}(s, a_\cL) = \chi_1^*(s,a_\cL) = 0$), where $\chi_{1}^*(s, a_\cL) = \frac{\alpha_1^{*}(s) - \eps}{\occstate^{\initpi, \targetpi \{ s, a_\cL \}}(s)} \cdot \ind{\occstate^{\initpi, \targetpi \{ s, a_\cL\}}(s) > 0}$. 
\end{proposition}
\begin{proof}
First, the assumption of the proposition implies that $\occstate^{\initpi, \targetpi} = \occstate^{\initpi, \targetpi\{s, a_{\cL}\}} = \occstate^{\pola, \poll}$ for all $\pola$ and $\poll$. Therefore,
\begin{align*}
    \score_2^{\pola, \ntargetpiproxy\{s,a_\cL\}}  &= \sum_{s} \occstate^{\pola, \ntargetpiproxy\{s,a_\cL\}} (s) \cdot R_2(s, \pola, \ntargetpiproxy\{s,a_\cL\}) \\
    &= \sum_{s} \occstate^{\pola, \targetpi\{s,a_\cL\}} (s) \cdot R_2(s, \pola, \ntargetpiproxy\{s,a_\cL\}) \\
    &= \sum_{s | \occstate^{\pola, \targetpi\{s,a_\cL\}}(s) > 0} \occstate^{\pola, \targetpi\{s,a_\cL\}} (s) \cdot R_2(s, \pola, \targetpi\{s,a_\cL\}) \\&= \score_2^{\pola, \targetpi\{s,a_\cL\}} %
\end{align*}
where $\ntargetpiproxy$ is defined in Lemma \ref{lm.neighbor_policies}.  Therefore the sufficient and necessary condition in Lemma \ref{lm.neighbor_policies} can be equivalently written as
\begin{align}\label{eq.neighbor_constrain_simple}
     \score_2^{\pola, \targetpi} \ge \score_2^{\pola, \targetpi \{s,a_\cL \}} + \eps \quad \forall s \text{ \em  s.t. } \occstate^{\pola, \targetpi}(s) > 0 \text{ \em and } \forall a_\cL \text{ \em s.t. } \targetpi(s, a_\cL) = 0.
\end{align}
Note that the assumption of the proposition also implies
\begin{align}\label{eq.neighbor_reward_simple}
     \score_2^{\pola, \targetpi} - \score_2^{\pola, \targetpi \{s,a_\cL \}} &= \sum_{s'} \occstate^{\pola, \targetpi}(s') \cdot R_2(s', \pola, \targetpi) - \sum_{s'} \occstate^{\pola, \targetpi\{s,a_\cL\}}(s') \cdot R_2(s', \pola, \targetpi\{s,a_\cL\}) \nonumber\\
     &= \sum_{s'} \occstate^{\pola, \targetpi\{s,a_\cL\}}(s') \cdot [R_2(s', \pola, \targetpi) -  R_2(s', \pola, \targetpi\{s,a_\cL\})] \nonumber\\
     &\underbrace{=}_{(i)}  \occstate^{\pola, \targetpi\{s,a_\cL\}}(s) \cdot [R_2(s, \pola, \targetpi) -  R_2(s, \pola, a_\cL)] \nonumber\\
     &=  \occstate^{\initpi, \targetpi\{s,a_\cL\}}(s) \cdot [R_2(s, \pola, \targetpi) -  R_2(s, \pola, a_\cL)],
\end{align}
where the equality $(i)$ holds since $\targetpi\{s, a_{\cL}\}$ is a neighbor  policy of $\targetpi$, differing from it in only state $s$. 
Now, note that the supremum that defines  $\alpha_1^{*}(s)$ exists (since $\pi_s \in \cP(A_{\cA})$), and denote $\pi_s$ that achieves it by $\pi_s^*$. Define $\pola^{\alpha_1^{*}}$ as $\pola^{\alpha_1^{*}}(s, .) = \pi_s^*$, and consider the following policy
 \begin{align*}
    \pola^{\beta}(s, a_\cA) = (1-\beta_s) \cdot \initpi(s, a_\cA) + \beta_s \cdot \pola^{\alpha_1^{*}}(s, a_\cA).
 \end{align*}
Eq. \eqref{eq.neighbor_reward_simple}, the assumption of the proposition, and the definition of $\alpha_1^*$ imply 
\begin{align*}
    \score_2^{\pola^{\beta}, \targetpi} - \score_2^{\pola^{\beta}, \targetpi \{s,a_\cL \}} &= (1-\beta_s) \cdot \occstate^{\initpi, \targetpi\{s,a_\cL\}}(s) \cdot [R_2(s, \initpi, \targetpi) -  R_2(s, \initpi, a_\cL)]\\
    &+\beta_s \cdot \occstate^{\initpi, \targetpi\{s,a_\cL\}}(s) \cdot [R_2(s, \pola^{\alpha_1^{*}}, \targetpi) -  R_2(s, \pola^{\alpha_1^{*}}, a_\cL)] \\
    &=(1-\beta_s) \cdot [\score_2^{\initpi, \targetpi} - \score_2^{\initpi, \targetpi \{s,a_\cL \}}]\\
    &+\beta_s \cdot \occstate^{\initpi, \targetpi}(s) \cdot [R_2(s, \pola^{\alpha_1^{*}}, \targetpi) -  R_2(s, \pola^{\alpha_1^{*}}, a_\cL)]\\
    &\ge(1-\beta_s) \cdot [\score_2^{\initpi, \targetpi} - \score_2^{\initpi, \targetpi \{s,a_\cL \}}]+ \beta_s \cdot \alpha_1^{*}(s).
\end{align*}
Therefore, to satisfy the constraints in Eq. \eqref{eq.neighbor_constrain_simple}, it suffices that
\begin{align*}
    (1-\beta_s) \cdot [\score_2^{\initpi, \targetpi} - \score_2^{\initpi, \targetpi \{s,a_\cL \}} - \eps] + \beta_s \cdot (\alpha_1^{*}(s) - \eps) \ge 0.
\end{align*}
If $\alpha_1^{*}(s) \ge \eps$, this sufficient condition can always be satisfied by setting $\beta_s = 1$, and hence the optimization problem \eqref{prob.instance_1} is feasible. Furthermore, note that $\tilde \beta_s = \max_{a_{\cL}} \ind{\chi_{1}^*(s, a_\cL) \ne 0 \lor \overline \chi_{\eps} (s, a_{\cL}) \ne 0} \cdot \frac{\overline \chi_{\eps} (s, a_{\cL})}{\chi_{1}^*(s, a) + \overline \chi_{\eps} (s, a_{\cL})}$ satisfies this sufficient condition; this can be easily verified by multiplying the terms in the sufficient condition by $\frac{1}{\occstate^{\initpi, \targetpi\{s, a_{\cL}\}}(s)}$.\footnote{Note that $\occstate^{\initpi, \targetpi\{s, a_{\cL}\}} > 0$ since we only aim to satisfy the condition for visited states.}
To obtain an upper bound on the cost of the optimal solution consider a policy $\pola^{\tilde \beta}$. 
We have

\begin{align*}
     \Cost(\pola, \initpi) \le \norm{\pola^{\tilde \beta} - \initpi}_{1, p} = \tilde \beta \cdot \norm{\pola^{\alpha_1^{*}} - \initpi}_{1, p} \le  \norm{B}_{1, p} \le 2 \cdot \norm{\frac{\overline \chi_{\eps}}{\chi_2^*  + \overline \chi_{\eps}}}_{p, \infty},
\end{align*}
where $B(s, a_{\cA}) = \tilde \beta_s \cdot (\pola^{\alpha_1^{*}}(s, a_{\cA}) - \initpi(s, a_{\cA}))$.
\end{proof}

\subsection{Proof of Theorem \ref{thm.upper_bound_2}}

We can define an analogous measure to $\alpha_1^*$, now by taking into account that we cannot only reason state-wise since transitions depend on actions $a_\cL$. 
In particular, we define $\alpha_2^{\pola}(s,a) = \score_2^{\pola, \targetpi} - \score_2^{\pola, \targetpi\{s, a\}}$ and $\alpha_2^* = \sup_{\pola} \min_{s, a} \alpha_2^{\pola}(s,a)$. 
Given the definition of $\Pola$, policy $\pola \in \Pola$ that achieves this supremum exists, and we denote it by $\pola^{*}$. This policy can be efficiently obtained using the following optimization problem
\begin{align*}
    \max_{\pola, \Delta} \Delta \text{ s.t. }  \sum_{s'} \left [ \occstate^{\initpi, \targetpi}(s') \cdot R_2(s', \pola, \targetpi) - \occstate^{\initpi, \targetpi\{s, a_\cL\}}(s') \cdot R_2(s', \pola, \targetpi\{s, a_\cL\}) \right ] \ge \Delta 
\end{align*}
where $\occstate^{\initpi, \targetpi}$ and $\occstate^{\initpi, \targetpi\{s , a\}}$ can be pre-computed. 
This claim follows from the fact that agent $\cA$ does not influence transition dynamics. We obtain the following result:

\textbf{Statement of Theorem \ref{thm.upper_bound_2}}: {\em 
Assume that $P(s, a_{\cA}, a_{\cL}) = P(s, a_{\cA}', a_{\cL})$ for all $a_{\cL}$ and $a_{\cA}'$, and that for $\targetpi$ and every policy $\pola$ the underlying Markov chain is ergodic, i.e., $\occstate^{\pola, \targetpi}(s)>0$ for all $\pola$. 
Then, the optimization problem \eqref{prob.instance_1} is feasible  if $\alpha_2^* \ge \eps$ and the cost of an optimal solution satisfies
\begin{align*}
         \Cost(\pola, \initpi) \le 2 \cdot \norm{\frac{\overline \chi_{\eps}}{\chi_2^*  + \overline \chi_{\eps}}}_{\infty} \cdot |S|^{1/p} 
\end{align*}
     with the element-wise division (equal to $0$ if $\chi_{\eps}(s, a_\cL) = \chi_2^*(s,a_\cL) = 0$), where $ \chi_{2}^*(s, a_\cL) = \frac{\alpha^{*}_2(s,  a_\cL) - \eps}{\occstate^{\initpi, \targetpi \{ s, a_\cL \}}(s)}$.
}
\begin{proof}
Consider a mixed policy:
\begin{align*}
    \pola^{\beta}(s, a_\cA) = (1-\beta) \cdot \initpi(s, a_\cA) + \beta \cdot \pola^*(s, a_\cA).
\end{align*}
By Lemma \ref{lm.neigbhor_policies_ergodic} we have that this policy force $\targetpi$ if 
$\score_2^{\pola^{\beta}, \targetpi} - \score_2^{\pola^{\beta},\targetpi\{s, a\}} \ge \eps$ 
holds for all states $s$ and actions $a$. 
By using the fact that agent $\cA$ does not influence transition dynamics we obtain
\begin{align*}
    \score_2^{\pola^{\beta}, \targetpi} - \score_2^{\pola^{\beta}, \targetpi \{s, a\}} &= \sum_{s'} \left [ \occstate^{\initpi, \targetpi}(s') \cdot R_2(s', \pola^{\beta}, \targetpi) - \occstate^{\initpi, \targetpi\{s, a_\cL\}}(s') \cdot R_2(s', \pola^{\beta}, \targetpi\{s, a_\cL\}) \right ]\\
    &= (1-\beta) \cdot \sum_{s'} \left [ \occstate^{\initpi, \targetpi}(s') \cdot R_2(s', \initpi, \targetpi) - \occstate^{\initpi, \targetpi\{s, a_\cL\}}(s') \cdot R_2(s', \initpi, \targetpi\{s, a_\cL\}) \right ]\\
    &+\beta \cdot \sum_{s'} \left [ \occstate^{\initpi, \targetpi}(s') \cdot R_2(s', \pola^{*}, \targetpi) - \occstate^{\initpi, \targetpi\{s, a_\cL\}}(s') \cdot R_2(s', \pola^{*}, \targetpi\{s, a_\cL\}) \right ]\\
    &=(1-\beta) \cdot [\score_2^{\initpi, \targetpi} - \score_2^{\initpi, \targetpi \{s, a\}}] + \beta \cdot \alpha_2^{\pola^*}(s, a).
\end{align*}
Using Lemma \ref{lm.neigbhor_policies_ergodic}, we obtain a sufficient condition for $\beta$
\begin{align*}
    (1-\beta) \cdot [\score_2^{\pola^{\beta}, \targetpi} - \score_2^{\pola^{\beta}, \targetpi \{s, a\}}] + \beta \cdot \alpha_2^{*}(s, a) \ge \eps \quad \forall s, a_\cL \text{ s.t } \targetpi(s, a_\cL) = 0.
\end{align*}
Note that if $\alpha_2^{*}(s, a) \ge \eps$ this condition can always be satisfied by setting $\beta = 1$. Now, by multiplying everything by\footnote{Note that the ergodicity assumption ensures $\occstate^{\initpi, \targetpi\{s, a_\cL\}}(s) > 0$.} $\frac{1}{\occstate^{\initpi, \targetpi\{s, a_\cL\}}(s)}$ and rearranging we obtain a sufficient condition
\begin{align*}
    \beta \ge \begin{cases}
    0 &\mbox{ if } \chi_2^*(s, a) = \bar \chi_{\eps}(s, a_\cL) = 0 \\
    \frac{\bar \chi_{\eps}(s, a_\cL)}{\chi_2^*(s, a) + \bar \chi_{\eps}(s, a_\cL)} &\mbox{ otw.}
    \end{cases}, \quad \forall s, a_\cL.
\end{align*}
 Since this has to hold for all state and action pairs $s$ and $a_\cL$,it suffices to set
 \begin{align*}
     \tilde \beta = \norm{\frac{\bar \chi_{\eps}}{\chi_2^* + \bar \chi_{\eps}}}_{\infty}
 \end{align*}
 to obtain an upper bound 
 \begin{align*}
     \Cost(\pola, \initpi) \le \norm{\pola^{\tilde \beta} - \initpi}_{1, p} = \tilde \beta \cdot \norm{\pola^* - \initpi}_{1, p} \le 2 \cdot \norm{\frac{\overline \chi_{\eps}}{\chi_2^*  + \overline \chi_{\eps}}}_{\infty} \cdot |S|^{1/p}.
 \end{align*}
\end{proof}


\begin{algorithm}
\caption[labelsep=period]{\conspolsearch}\label{alg.conserv_pol_iter_general}
\begin{algorithmic}
\REQUIRE $\mathcal M = (\{1, 2\}, S, A, P, R, \gamma, \sigma)$, $\eps$, $\delta_\eps$, $\initpi$, $\lambda$, $p$
\ENSURE Policy of the adversary, $\pola$
\STATE Initialize $t = 0$
\FOR{$t = 0$ to $T-1$}
\STATE For all $s$ and $a_{\cL}$ s.t. $\occstate^{\pola^t, \targetpi}(s) = 0$ and $\targetpi(s, a_{\cL}) = 0$, find $\ntargetpiproxy\{s, a_{\cL}\}$ by solving the Bellman's equation \eqref{eq.belman_equations} from Lemma \ref{lm.neighbor_policies}
\STATE Calculate state occupancy measures $\occstate^{\pola^t, \targetpi}$ and $\occstate^{\pola^t, \ntargetpiproxy \{ s, a_{\cL}\}}$
\STATE Evaluate the true gap $\eps_{\pola^t} = \min_{\eps'} \eps'$ s.t. $\score_\cL^{\pola^t, \targetpi} \ge \score_\cL^{\pola^t, \ntargetpiproxy\{s, a_{\cL}\}} + \eps'$ 
\STATE Solve the optimization problem \eqref{prob.instance_1.rel_2} to obtain $\pola^{t+1}$ 
\IF{$\pola^{t+1} = \pola^t$}
    \STATE {\bf break}
\ENDIF
\ENDFOR
\STATE Set the result $\pola$ to solution $\pola^{t}$ that minimizes $\norm{\pola^{t} - \initpi}_{1, p}$ while satisfying $\eps_{\pola^t} \ge \eps$  %
\end{algorithmic}
\end{algorithm}

\section{Algorithms: Additional Details}\label{sec.app.algorihtms}

In this appendix, we provide additional details about the algorithms. In particular, we describe a general version of the conservative policy search mentioned and we provide additional implementation details related of our alternating policy updates approach.

\textbf{Conservative Policy Search. } In our general version of conservative policy search method (applicable to non-ergodic environments), we use the following optimization problem instead of \eqref{prob.instance_1.rel}:
 \begin{align*}
    &\min_{\pola \in \mathcal B(\pola^t, \dpi), \eps'} \quad \Cost(\pola, \initpi) - \lambda \cdot \min \{\eps', \eps \cdot (1 + \delta_{\eps}) \}
    \\
    \label{prob.instance_1.rel_2}\tag{P1''}
    &\quad\quad\quad \mbox{ s.t. } \quad \hat \score_\cL^{\pola, \targetpi} \ge \hat \score_\cL^{\pola, \ntargetpiproxy \{s, a_\cL\}} + \eps'.
\end{align*}
for all $s \text{ s.t. } \occstate^{\pola, \targetpi}(s) > 0$ and $a_\cL \text{ s.t. } \targetpi(s, a_\cL) = 0$. 
Algorithm \ref{alg.conserv_pol_iter_general} summarizes the mains steps of conservative policy search in that case. Apart from using \eqref{prob.instance_1.rel_2} instead of \eqref{prob.instance_1.rel}, the main difference is that in each step the algorithm needs to find all policies $\ntargetpiproxy \{s, a_\cL\}$ by solving the Bellman equations \eqref{eq.belman_equations.ntargetproxy}. Analogous algorithms were used for CPS-based baselines (i.e., Constraints Only PS (COPS) and Unconservative
PS (UPS).

\begin{table}[]
    \centering
    \begin{tabular}{|c|c|c|c|c|c|}
     \hline
            & $\pi_\dagger$, $\pi_0^1$  & epochs    & $\lambda$ & $\phi$-pretrain timesteps    & $\phi$-update timesteps\\
    \hline
    Push2D  &  stochastic                   & 20        &  [0;5]              & 500000            & 10000\\
    Push1D  &  deterministic                & 50        &  [0;1]          & 500000            & 5000\\
    \hline
\end{tabular}
    \caption{Training parameters for Alternating Policy Updates -- Algorithm \ref{alg.alter_updates_main_text}.}
    \label{tab.training_parameters_apu}
\end{table}
 
\textbf{Alternating policy updates}. 
Algorithm \ref{alg.alter_updates_main_text} is implemented using Pytorch \cite{NEURIPS2019_9015}, while the Environments use the Petting-Zoo API \cite{terry2021pettingzoo}.
The trajectories used for updating $\theta$ are collected over 40 episodes which in total consists 1000 timesteps for Push 1D and 1800 for Push 2D.
The algorithm follows the standard PPO policy updates, but modifies the loss function so that it fits the optimization objective \eqref{prob.instance_2}. Table \ref{tab.training_parameters_apu} provides parameters that were used for training APU. Analogous training procedures are used for APU-based baselines (Random Learner (RL), Symmetric APU (SAPU), and Distance-only APU (DAPU)). Note that RL and SAPU do not perform the pretraining of $\pollp$. Moreover, in SAPU the number of $\phi$-update timesteps is equal to $1000$.

}
}
{}
\end{document}